\documentclass{article}
\PassOptionsToPackage{numbers,sort&compress}{natbib}

\usepackage[final,main]{neurips_2026}

\usepackage[utf8]{inputenc}
\usepackage[T1]{fontenc}
\usepackage{hyperref}
\usepackage{url}
\usepackage{booktabs}
\usepackage{amsmath}
\usepackage{amsfonts}
\usepackage{amssymb}
\usepackage{amsthm}
\usepackage{mathtools}
\usepackage{nicefrac}
\usepackage{microtype}
\usepackage{bm}
\usepackage{xcolor}
\usepackage{algorithm}
\usepackage{algpseudocode}
\usepackage{graphicx}
\usepackage{rotating}
\usepackage{pgfplots}
\pgfplotsset{compat=1.18}
\usepackage{multirow}
\usepackage{tikz}
\usepackage{longtable}
\usepackage{enumitem}

\usetikzlibrary{positioning,calc,arrows.meta,shapes.geometric}

\newcommand{\R}{\mathbb{R}}
\newcommand{\E}{\mathbb{E}}

\newcommand{\MI}{I}
\newcommand{\sgn}{\mathrm{sign}}

\newcommand{\method}{\textsc{PACZero}}

\newcommand{\PACAGC}{\textsc{PACZero-AGC}}
\newcommand{\PACAGCz}{\textsc{PACZero-AGC}$_0$}
\newcommand{\PACABC}{\textsc{PACZero-BAC}}
\newcommand{\PACABCz}{\textsc{PACZero-BAC}$_0$}

\newtheorem{theorem}{Theorem}
\newtheorem{lemma}[theorem]{Lemma}
\newtheorem{proposition}[theorem]{Proposition}
\newtheorem{corollary}[theorem]{Corollary}
\theoremstyle{definition}
\newtheorem{definition}[theorem]{Definition}
\newtheorem{remark}[theorem]{Remark}

\newcommand{\repeatnote}{The OPT-6.7B and Llama-3.1-8B entries at $\mathrm{MI}\!\in\!\{0.33,0.68,0\}$ repeat Table~\ref{tab:headline}.}

\newcommand{\codelink}{\footnote{\url{https://github.com/bilgehanertan/paczero-bac}}}

\title{\textsc{PACZero}: PAC-Private Fine-Tuning of Language~Models via Sign~Quantization}

\author{%
  Murat Bilgehan Ertan\thanks{Equal contribution.\hspace{1em}$^{\dagger}$Affiliated with Vrije Universiteit Amsterdam.}\hphantom{$^{*}$}\footnotemark[2] \\
  CWI Amsterdam \\
  Amsterdam, Netherlands \\
  \texttt{bilgehan@cwi.nl} \\
  \And
  Xiaochen Zhu\footnotemark[1] \\
  MIT \\
  Cambridge, MA, USA \\
  \texttt{xczhu@mit.edu} \\
  \And
  Phuong Ha Nguyen \\
  eBay \\
  San Jose, CA, USA \\
  \texttt{phuongha.ntu@gmail.com} \\
  \AND
  Marten van Dijk\footnotemark[2] \\
  CWI Amsterdam \\
  Amsterdam, Netherlands \\
  \texttt{mevd@cwi.nl} \\
  \And
  Srinivas Devadas \\
  MIT \\
  Cambridge, MA, USA \\
  \texttt{devadas@mit.edu} \\
}

\begin{document}

\maketitle

\begin{abstract}
Differentially private (DP) zeroth-order (ZO) fine-tuning of large language models (LLMs) inherits DP's harsh utility-privacy trade-off. PAC Privacy, an instance-based notion, offers an alternative. It bounds the mutual information (MI) between the secret input and the output, and calibrates the noise to the stability of the output under the input distribution. We introduce PACZero, a family of ZO fine-tuning mechanisms under PAC Privacy. For stability, PACZero quantizes each update to a single bit, the sign of the ZO update aggregated over the secret subset of the training data. On many steps all candidate subsets agree on this sign, so the bit can be released at no cost in MI. On the remaining steps the bit passes through a noisy channel calibrated to a per-step MI budget. We study two channels, the standard additive Gaussian channel of PAC Privacy (PACZero-AGC) and a binary asymmetric channel with optimal bit flipping (PACZero-BAC), along with their zero-budget versions at $\mathrm{MI} \to 0$, PACZero-AGC$_0$ and PACZero-BAC$_0$. At zero MI, nothing released during training, including the fine-tuned model, carries any information about which secret subset was sampled, so no adversary inferring that subset can do better than its prior. In particular, the success rate of any membership inference attack (MIA) is bounded by 50\% under a balanced prior. To our knowledge, this is the first such guarantee for LLM fine-tuning, and DP matches the MIA bound only at $\varepsilon = \delta = 0$ and infinite noise. We evaluate on SST-2 and SQuAD with models from 1.3B to 8B parameters, under LoRA and full fine-tuning. Both channels retain high utility under tight MI budgets, and even PACZero-BAC$_0$, which guarantees zero MI, stays within 2 points of non-private ZO fine-tuning with the same number of steps. To our knowledge no DP method achieves usable utility matching this MIA-resistance level.

\end{abstract}

\section{Introduction}
Zeroth-order (ZO) optimization has emerged as a memory-tractable alternative to first-order LLM supervised fine-tuning, replacing backpropagation with two forward passes per step~\citep{malladi2023mezo}. When the training corpus is sensitive, trained weights leak individual records through membership inference attacks (MIA) and verbatim extraction~\citep{DBLP:conf/sp/ShokriSSS17,DBLP:conf/sp/CarliniCN0TT22,DBLP:conf/uss/CarliniTWJHLRBS21,nasr2025extracting,DBLP:conf/emnlp/MireshghallahGU22}. The dominant defense is differential privacy (DP), which bounds the influence of any single training record on the released model through a parameter $\varepsilon$. Its standard instantiation, DP-SGD~\citep{DBLP:conf/ccs/AbadiCGMMT016}, clips per-record gradients at a fixed clipping threshold and adds Gaussian noise tracked across steps by the R\'enyi accountant~\citep{DBLP:conf/csfw/Mironov17}, and it extends to fine-tuning~\citep{DBLP:conf/iclr/YuNBGI0KLMWYZ22,DBLP:journals/tmis/LiuZZGZWQ25}. Applied to ZO, it yields DP zeroth-order methods~\citep{zhang2024dpzero,tang2024dpzo,liu2025dpzopo,bao2025dpaggzo}, which clip and noise only the per-step zeroth-order scalars, one per sampled direction. Every published DP zeroth-order method still inherits DP's harsh utility-privacy trade-off, because its underlying mechanism is unchanged. The noise is set by the worst-case sensitivity to a single record and composed over $T$ steps. Utility loss at tight $\varepsilon$ is a long-observed property of DP-SGD in iterative deep learning~\citep{DBLP:conf/nips/BagdasaryanPS19,DBLP:journals/corr/abs-2601-10237,de2022unlockinghighaccuracydifferentiallyprivate}, so reaching a meaningful high-privacy regime for LLM fine-tuning remains a challenge.

PAC Privacy~\citep{xiao2023pac} bounds the mutual information (MI) between the secret and the release, and replaces DP's worst-case sensitivity-based noise with instance-based noise calibrated to the empirical covariance of the release. Under a known input distribution of the secret, it exploits the \emph{stability} of the output under the specified distribution, adding less noise for more stable releases. Recent work enables adaptive composition and applies it to serve privacy-preserving \emph{predictions} under adaptive and adversarial querying~\citep{zhu2025pacresponses}. \citet{DBLP:conf/sp/SridharZD26} give PAC-private gradient descent for smooth convex objectives, releasing a noisy gradient at every step. What remains open is private \emph{iterative} fine-tuning of large models, such as LLMs, the setting in which DP-SGD operates, where an update is released to the adversary at every optimization step. At the scale of fine-tuning LLMs ($d\!\sim\!10^9$, $T\!\sim\!10^3$), however, calibrating the noise from the covariance of the output is prohibitively expensive, leaving large-scale private training out of reach of PAC Privacy.

The proposed \method{} family addresses both problems, utility at high privacy and PAC Privacy for iterative LLM fine-tuning, by combining zeroth-order optimization~\citep{malladi2023mezo} with sign quantization before the per-step release.
Concretely, \method{} follows the standard instantiation of PAC Privacy~\citep{sridhar2025pac,zhu2025pacresponses,battiston2026simd,DBLP:conf/sp/SridharZD26} and adopts a subsampling-based input distribution used by PAC Privacy~\citep{xiao2023pac,xiao2024formal} and assumed by MIAs~\citep{DBLP:conf/sp/ShokriSSS17,DBLP:conf/sp/CarliniCN0TT22}. It constructs $M=128$ random candidate subsets of the training data (i.e., the \emph{universe}) such that each record in the universe is in exactly half of the candidate subsets and samples one of them uniformly at random as the secret $S^*$. At each step, it aggregates per-sample two-point zeroth-order loss differences over the $M$ subsets, sign-quantizes each subset's mean response to $s_m\in\{-1,+1\}$, and privatizes and releases the sign produced by the secret subset (Figure~\ref{fig:mechanism}).

Two structural consequences follow from quantizing to one bit. First, a sign is stable and therefore cheap to privatize under PAC Privacy. In particular, quantization creates frequent \emph{unanimity}. On many steps all possible candidate subsets agree on the released bit. On these steps the bit reveals nothing about which candidate is the secret, so it can be released as is at zero conditional MI. 
On the remaining steps the bit is released with PAC Privacy through a noisy channel calibrated to a per-step MI budget. Because the channel's input is a single bit, its MI with the output under additive Gaussian noise is a one-dimensional integral that can be inverted numerically, allowing tight noise calibration. We consider an additive Gaussian channel on the bit, the standard choice in PAC Privacy (\PACAGC{}), and an optimal binary asymmetric channel that flips the bit as rarely as the budget allows (\PACABC{}).
Second, the MI budget can be pushed exactly to \emph{zero}.
We develop a zero-budget variant of each channel, a fair coin for the Gaussian channel (\PACAGCz{}) and the majority sign among candidate subsets for the binary channel (\PACABCz{}). 
Under both variants, the release is independent of the secret on all steps, so $\MI(S^*;\tilde Y_{1:T})=0$  for every $T$, where $\tilde Y_{1:T}$ are the observations over $T$ steps, and the fine-tuned model reveals \emph{nothing} about the secret subset and achieves perfect resistance against MIAs inferring the membership of a record in the secret subset. That is, the MIA posterior success rate equals its prior success rate of $1/2$ for every record in the universe. In summary, we make three contributions:
\begin{itemize}
\item \textbf{The \method{} family of PAC-private ZO mechanisms.} We introduce \method{}\codelink{}, which brings PAC Privacy to iterative LLM fine-tuning. One accounting argument covers all four variants that differ in the MI budget and the noise channel (\S\ref{sec:method}). We evaluate all four on SST-2 and SQuAD with OPT-1.3B, OPT-6.7B, Qwen2.5-7B and Llama-3.1-8B, under both LoRA and full fine-tuning (Table~\ref{tab:headline} and Appendix~\ref{app:full-grid}).
\item \textbf{Tighter and distortion-optimal calibration for a one-bit input.} For the additive Gaussian channel, we propose tighter noise calibration for the one-bit channel input, which sets the noise scale by numerically inverting the exact MI instead of a variance-based upper bound used in canonical PAC Privacy mechanisms (\S\ref{sec:binary-mi}). The binary asymmetric channel goes further and achieves the minimum distortion under a fixed MI budget (\S\ref{sec:abc}). Empirically, \PACABC{} matches or exceeds the utility of \PACAGC{}, test accuracy on SST-2 and F1 on SQuAD, in almost every setting and budget.
\item \textbf{Usable fine-tuning at zero MI.} To our knowledge, \PACABCz{} and \PACAGCz{} give the first guarantee of zero MI between the secret subset and everything released during iterative LLM fine-tuning. 
Moreover, \PACABCz{} stays within 2 points of non-private ZO on all sixteen settings (Table~\ref{tab:headline} and Appendix~\ref{app:full-grid}). DP reaches the same MIA bound only at $\varepsilon\!=\!\delta\!=\!0$, which a noise-adding mechanism attains only with infinite noise (\S\ref{sec:zpl}).
\end{itemize}

\begin{figure}[t]
\centering
\begin{tikzpicture}[
  font=\small,
  box/.style={draw=black!55, rounded corners=2pt, align=center, inner sep=4pt, line width=0.4pt},
  stage/.style={box, fill=black!4, text width=33mm},
  split/.style={box, fill=blue!6, draw=blue!45, text width=40mm},
  free/.style={box, fill=green!10, draw=green!45!black!60, text width=52mm},
  chan/.style={box, fill=orange!10, draw=orange!55!black!50, text width=70mm},
  upd/.style={box, fill=black!6, text width=58mm},
  lbl/.style={font=\footnotesize\itshape, inner sep=2pt},
  flow/.style={-{Stealth[length=4pt]}, draw=black!60}]

  \node[stage] (zo) {per-sample ZO\\$\hat g_i,\; i\!=\!1\!:\!N$};
  \node[stage, right=7mm of zo] (subs) {$M$ subset means\\$s_m\!=\!\sgn(\bar g_m)$};
  \node[split, right=7mm of subs] (q) {posterior mass on $+1$\\$q_t^{+}\!=\!\sum_m p_t[m]\mathbf 1[s_m\!=\!{+}1]$};

  \node[free, below=10mm of zo.south west, anchor=north west] (una)
      {\textbf{unanimity} $q_t^{+}\!\in\!\{0,1\}$\\ $Y_t\!=\!\tilde Y_t\!=\!s_{j^*}$, \; $\beta_t^{\mathrm{used}}\!=\!0$};
  \node[chan, right=8mm of una] (ch)
      {\textbf{disagreement} $q_t^{+}\!\in\!(0,1)$\\ release $\tilde Y_t$ through a channel at $\beta_t$ and $Y_t=\sgn(\tilde Y_t)$};
  \node[upd, below=6mm of $(una.south)!0.5!(ch.south)$, anchor=north] (upd)
      {$\theta_{t+1}\!\gets\!(1\!-\!\eta_t\lambda)\theta_t\!-\!\eta_t Y_t z_t$};

  \draw[flow] (zo) -- (subs);
  \draw[flow] (subs) -- (q);
  \draw[flow] (q.south) -- ++(0,-5mm) -| (una.north) node[lbl, pos=0.60, above]{agree};
  \draw[flow] (q.south) -- ++(0,-5mm) -| (ch.north)  node[lbl, pos=0.60, above]{split};
  \coordinate (busL) at ($(una.south)+(0,-4mm)$);
  \coordinate (busR) at ($(ch.south)+(0,-4mm)$);
  \draw[draw=black!60] (una.south) -- (busL);
  \draw[draw=black!60] (ch.south) -- (busR);
  \draw[draw=black!60] (busL) -- (busR);
  \draw[flow] ($(busL)!0.5!(busR)$) -- (upd.north);
\end{tikzpicture}%
\caption{\textbf{The \method{} per-step mechanism.} Per-sample finite-difference scalars are aggregated over $M\!=\!128$ random subsets, sign-quantized to $s_m\!\in\!\{-1,{+}1\}$. On unanimity, the secret subset's sign is released as is at zero conditional MI. Otherwise it passes through a calibrated channel.}
\label{fig:mechanism}
\end{figure}

\section{Background}
\label{sec:background}
\textbf{Two-point zeroth-order gradient estimation.\;\;}
For a model with parameters $\theta\in\R^d$ and a per-sample loss $\ell_i(\theta)$,
the two-point zeroth-order estimator~\citep{nesterov2017random,malladi2023mezo}
samples a single direction $z\sim\mathcal{N}(0,I_d)$ and computes the
per-sample finite-difference scalar
\begin{equation}
  \hat g_i \;=\; \frac{\ell_i(\theta + \mu z) - \ell_i(\theta - \mu z)}{2\mu}
  \;=\; \langle \nabla \ell_i(\theta), z\rangle + O(\mu^2),
\label{eq:zo-scalar}
\end{equation}
where $\mu>0$ is the perturbation scale~\citep{malladi2023mezo}. The parameter
update is $\theta \leftarrow \theta - \eta \cdot \hat g \cdot z$, where
$\hat g$ is obtained by clipping and averaging the $\hat g_i$ over the
batch~\citep{zhang2024dpzero,bao2025dpaggzo}. The structural property that matters here,
exploited by all DP zeroth-order methods~\citep{zhang2024dpzero,tang2024dpzo,bao2025dpaggzo},
is that $z$ is sampled \emph{independently} of the data and is therefore public.
The only data-dependent quantity per step is the scalar $\hat g$, so the
per-step noise scale is dimension-independent.

\begin{definition}[PAC Privacy~\citep{xiao2023pac}]\label{def:pac}
Given a data generating distribution $\mathcal D$ over possible datasets in $\mathcal X$, and a binary-valued attack success criterion $\rho: \mathcal X\times\mathcal X\to\{0,1\}$, we say a mechanism $\mathcal M:\mathcal X\to\mathcal Y$ is $(\delta,\rho,\mathcal{D})$-PAC private if for every adversary $\mathcal A:\mathcal Y\to\mathcal X$, who knows $\mathcal{D}$ and $\mathcal M$, the attack success rate, as measured by $\rho$, is at most $1-\delta$ when the secret dataset $X$ is drawn from $\mathcal{D}$ and the adversary observes $Y = \mathcal M(X)$,
$
1-\delta_\mathcal A:=\Pr_{X\sim \mathcal D, Y\gets\mathcal M(X)}[\rho(\mathcal A(Y), X) = 1] \leq 1-\delta.
$
\end{definition}

PAC Privacy~\citep{xiao2023pac} provides an instance-based framework for measuring and bounding information leakage about a secret dataset from a mechanism's output. Definition~\ref{def:pac} assumes a strong and informed adversary who knows the input distribution $\mathcal D$, the mechanism $\mathcal M$, and has unbounded compute resources.
Privacy guarantees for $X\sim\mathcal D$ established under a strong adversary model naturally extend to realistic adversaries without such knowledge and/or computational power.

Given $\mathcal D$ and $\rho$, an informed adversary has a \emph{prior success rate} even before observing any outputs, denoted as $1-\delta_0 :=\max_{Q}\Pr_{X\sim \mathcal D,\hat X\sim Q}[\rho(\hat X,X)=1]$.
A fundamental result of PAC Privacy~\citep{xiao2023pac} bounds the \emph{posterior advantage}, $\Delta_\mathrm{KL}\delta_\mathcal{A}$, of an adversary, the gain of its success rate over its prior, by mutual information (MI). For every criterion $\rho$ and every adversary $\mathcal A$ of Definition~\ref{def:pac} with $\delta_{\mathcal A}\leq \delta_0$,
\begin{equation}\label{eq:pac_mi_bound}
\Delta_\mathrm{KL}\delta_\mathcal{A}:=\mathrm{KL}(1-\delta_\mathcal{A}\,\|\,1-\delta_0)\leq I(X;\mathcal M(X)),
\end{equation}
where $\mathrm{KL}(p \| q) := p\ln(p/q) + (1-p)\ln((1-p)/(1-q))$ is the binary KL divergence.
We note that the MI bound depends on neither $\rho$ nor $\mathcal A$, so~\eqref{eq:pac_mi_bound} is attack-agnostic. Given $\mathcal D$ and $\mathcal M$, MI bounds the posterior advantage of every attack.
Hence, one can specify and enforce a total MI budget, and derive concrete bounds on success rate given the prior success rate of an informed adversary under $\rho$.

When the mechanism interacts with a persistent secret over $T$ rounds, releasing outputs $Y_1, Y_2, \ldots, Y_T$, the chain rule of mutual information gives
\begin{equation}
I(X;\, Y_{1:T}) \;=\; \sum_{t=1}^{T} I(X;\, Y_t \mid Y_{1:t-1}).
\label{eq:chain-rule}
\end{equation}
To bound each term, \citet{zhu2025pacresponses} maintain a posterior distribution $P_{X \mid Y_{1:t-1}=y_{1:t-1}}$ over the secret space, which represents the strong adversary's optimal belief given the observed transcript so far.
At step $t$, the mechanism calibrates its randomness to ensure $I(X; Y_t \mid Y_{1:t-1} = y_{1:t-1}) \leq \beta_t$ for every realization $y_{1:t-1}$.
This reduces privacy accounting to per-step MI control, which we exploit in \S\ref{sec:method}.

\textbf{Instantiation and MIA guarantees.\;\;}
Following the standard instantiation of PAC Privacy~\citep{sridhar2025pac,zhu2025pacresponses,battiston2026simd,DBLP:conf/sp/SridharZD26}, which builds on the subsampling-based input distributions used by~\citet{xiao2023pac,xiao2024formal}, we start from a \emph{universe} of $N$ training examples $U=\{x_1,\ldots,x_N\}$ and construct $M$ \emph{candidate} subsets $S_1,\ldots,S_M\subset U$ such that every example lies in exactly $M/2$ of them, chosen uniformly at random and independently for each example as in~\citet{zhu2025pacresponses}.
We fix the subsets before training and set $\mathcal{D}=\mathrm{Unif}\{S_1,\ldots,S_M\}$. Drawing the secret dataset from $\mathcal{D}$ then amounts to sampling a secret index $j^*\sim\mathrm{Unif}[M]$ and training on $S^*=S_{j^*}$ throughout the run. The finite support of $\mathcal D$ makes adaptive composition tractable, because the posterior over the secret index $j^*$ is a discrete distribution over $[M]$.

In a membership inference attack (MIA) on a record $x_i\in U$, the adversary infers whether $x_i\in S^*$. Since $x_i$ lies in half of the $M$ equally likely candidates, the prior success rate is $1-\delta_0=1/2$. A bound $I(S^*;\tilde Y_{1:T})\le\mathrm{MI}_{\mathrm{tot}}$ on the total MI then caps the MIA success rate, through~\eqref{eq:pac_mi_bound}, at the largest $1-\delta_{\mathcal A}$ with $\mathrm{KL}(1-\delta_{\mathcal A}\,\|\,1/2)\le\mathrm{MI}_{\mathrm{tot}}$, which we compute numerically.
For example, $\mathrm{MI}_{\mathrm{tot}}=1/4$ nats limits MIA success to $\approx 84\%$ and lowering it to $1/128$ tightens the guarantee to around $56\%$.
$\mathrm{MI}_{\mathrm{tot}}=0$ achieves perfect MIA resistance and bounds MIA success by its prior of $50\%$.

For comparison, in the standard membership game on neighboring datasets, $(\varepsilon,\delta)$-DP bounds MIA success by $1-(1-\delta)/(1+e^{\varepsilon})$~\citep{DBLP:conf/icml/KairouzOV15,sridhar2025pac}, which reaches the prior of $50\%$ only at $\varepsilon\!=\!\delta\!=\!0$.
Throughout this paper we report the privacy of \method{} in its native unit (MI in nats), and where useful we label an MI budget with the DP $\varepsilon$ that yields a matching MIA upper bound at $\delta\!=\!10^{-5}$.
For example, $(2,10^{-5})$-DP caps MIA success at $88.1\%$, matching the MIA bound of $\mathrm{MI}_{\mathrm{tot}}\!=\!0.33$ nats, so we label $0.33$ nats with $\varepsilon\!=\!2$.

\section{The \method{} Mechanisms}
\label{sec:method}
The per-step procedure of \method{} is summarized in Algorithm~\ref{alg:paczero}. At step $t$ we compute per-sample finite-difference scalars $\hat g_i$ as in~\eqref{eq:zo-scalar} for all $N$ records of the universe $U$, aggregate them to subset means $\bar g_m$ over the $M$ subsets of \S\ref{sec:background}, and sign-quantize them to $s_m\!\in\!\{-1,+1\}$.
The mechanism then privatizes the secret subset's sign $s_{j^*}$ and releases a potentially noised value $\tilde Y_t$, whose sign $Y_t$ updates parameters via $\theta_{t+1} = \theta_t - \eta_t Y_t z_t$. A step therefore reads the whole universe, at a cost of $2N$ forward passes. A step of \PACAGC{} or \PACAGCz{} takes about the time and memory of a non-private ZO step over the same universe, so the privacy layer itself adds little cost. The cost of a step grows linearly with $N$ and shows no trend in the number of subsets $M$ (Appendix~\ref{app:cost-profile}). The release is deterministic on the \emph{unanimity} branch (when the posterior $p_t\!\in\!\Delta^{M-1}$ over the secret index $j^*$ places all mass on subsets that agree on $s_m$) and passes through a noisy \emph{channel} on the \emph{disagreement} branch otherwise. We propose four variants under a unified privacy argument for any calibrated channel (\S\ref{sec:adaptive}), which differ in the channel (\S\ref{sec:binary-mi} and \S\ref{sec:abc}) and the MI budget (\S\ref{sec:zpl}).

\begin{algorithm}[t]
\caption{\method{} per-step update at step $t$.}
\label{alg:paczero}
\begin{algorithmic}[1]
\Require parameters $\theta_t$, posterior $p_t\in\Delta^{M-1}$ over the secret index $j^*$, per-step MI budget $\beta_t$
\State sample $z_t \sim \mathcal{N}(0, I_d)$ \Comment{public ZO direction}
\For{$i = 1, \dots, N$}
   \State $\hat g_{i} \gets [\ell_i(\theta_t+\mu z_t) - \ell_i(\theta_t - \mu z_t)] / (2\mu)$ \Comment{per-sample finite difference}
   \State $\hat g_{i} \gets \sgn(\hat g_i)\cdot\min(|\hat g_i|, c)$ \Comment{per-sample clipping, not privacy-relevant (Appendix~\ref{app:method-details})}
\EndFor
\State $\bar g_m \gets \frac{1}{|S_m|}\sum_{i\in S_m} \hat g_i,\quad s_m \gets \sgn(\bar g_m)$ \quad for $m=1,\ldots,M$
\State $q_t^{+} \gets \sum_m p_t[m]\cdot \mathbf{1}[s_m = +1]$ \Comment{posterior probability that $s_{j^*}\!=\!+1$}
\If{$q_t^{+}\in\{0,1\}$} \Comment{unanimity branch}
   \State $\tilde Y_t \gets s_{j^*}$, \quad $Y_t \gets \sgn(\tilde Y_t)$, \quad $\beta_t^{\mathrm{used}} \gets 0$, \quad $p_{t+1}\gets p_t$
\Else \Comment{disagreement branch, release through a calibrated channel}
   \State $\mathcal{C}_t \gets \mathrm{Calibrate}(q_t^+, \beta_t)$ \Comment{AGC (\S\ref{sec:binary-mi}) or BAC (\S\ref{sec:abc})}
   \State $\tilde Y_t \gets \mathcal{C}_t(s_{j^*})$, \quad $Y_t \gets \sgn(\tilde Y_t)$, \quad $\beta_t^{\mathrm{used}} \gets \min(\beta_t,\, h(q_t^+))$
   \State $p_{t+1}[m] \propto p_t[m]\cdot L_t(\tilde Y_t\mid s_m)$ \Comment{Bayes' rule~\eqref{eq:posterior}}
\EndIf
\State $\theta_{t+1} \gets (1-\eta_t\lambda)\theta_t - \eta_t \cdot Y_t \cdot z_t$ \Comment{$\lambda\!=\!0$ in all experiments.}
\end{algorithmic}
\end{algorithm}

\subsection{Per-step accounting and adaptive budget allocation}
\label{sec:adaptive}
On a \emph{unanimity} step, every candidate $m$ with positive posterior mass, i.e., $p_t[m] > 0$,
produces the same sign $s_m$. The bit $s_{j^*}$ derived from the secret subset is then the same as every other candidate's, so releasing it reveals nothing about which candidate is the secret. Under PAC Privacy, the mechanism releases $\tilde Y_t=s_{j^*}$ and the posterior does not change, charging $\beta_t^{\mathrm{used}}\!=\!0$ of the MI budget.

On a \emph{disagreement} step the candidates do not all agree, so releasing $s_{j^*}$ directly would reveal something about which secret subset $S_{j^*}$ is used.
The bit is therefore passed through a noisy channel $\mathcal{C}_t$, set so that, when its randomness is drawn independently, $\mathcal{C}_t(s_{j^*})$ carries at most $\beta_t$ nats of information about $s_{j^*}$.
This is the calibration step of PAC Privacy, which chooses the noise from algorithmic stability of the output, i.e., how much the noiseless output varies with the secret. Here, the noiseless output is the single bit $s_{j^*}$,
whose distribution under the posterior $p_t$ is specified by $q_t^+ = \sum_m p_t[m]\,\mathbf{1}[s_m\!=\!+1]$, the probability that $s_{j^*}$ equals $+1$. 
Therefore, the mechanism chooses a channel $\mathcal C_t$ from $q_t^+$ and $\beta_t$ alone.
Formally, the channel must be calibrated in the following sense.

\begin{definition}[Calibrated channel]\label{def:calibrated}
Let $q\!\in\!(0,1)$ and $\beta\!\geq\!0$. A channel $\mathcal{C}$ with input alphabet $\{-1,+1\}$ is \emph{calibrated at $(q,\beta)$} if $I(\xi;\mathcal{C}(\xi))\!\leq\!\beta$ for an input bit $\xi$ with $P(\xi\!=\!+1)\!=\!q$.
\end{definition}

After calibrating the channel, the mechanism releases $\tilde Y_t = \mathcal C_t(s_{j^*})$ and updates the posterior by Bayes' rule to enable adaptive composition~\citep{zhu2025pacresponses}. Each candidate $m$ is weighted by how likely the channel is to produce the observed output when its input is $s_m$:
\begin{equation}
p_{t+1}[m]\;\propto\;p_t[m]\,L_t(\tilde Y_t\mid s_m),\qquad\text{where}\quad L_t(y\mid s):=\Pr[\mathcal{C}_t(s)\!=\!y],
\label{eq:posterior}
\end{equation}
is the probability (or density for continuous outputs), over the channel's own randomness, that $\mathcal{C}_t$ outputs $y$ on input $s$.
The optimizer uses only the sign $Y_t=\sgn(\tilde Y_t)$ of the noisy release, updating the model by $\pm\eta_t z_t$. 
A calibrated disagreement step costs at most $\beta_t^{\mathrm{used}}\!=\!\min\{\beta_t,h(q_t^+)\}$ nats, since MI cannot exceed the input's binary entropy $h(q_t^+)$.

\textbf{Adaptive budget allocation.\;\;} Given a total MI budget $\mathrm{MI}_{\mathrm{tot}}$ and step budget $T$, the uniform budget allocation $\beta_t = \mathrm{MI}_{\mathrm{tot}}/T$~\citep{zhu2025pacresponses} is loose, because a step that takes the unanimity branch leaves $\beta_t$ unspent and uniform allocation cannot recover it later. \method{} uses
\begin{equation}
\beta_t \;=\; \max\!\left(0,\, \mathrm{MI}_{\mathrm{tot}} - \mathrm{MI}_{\mathrm{used}}^{(t-1)}\right) \,\big/\, (T - t + 1),
\label{eq:adaptive-beta}
\end{equation}
where $\mathrm{MI}_{\mathrm{used}}^{(t-1)} = \sum_{\tau<t} \beta_\tau^{\mathrm{used}}$ is the budget spent before step $t$. Dividing the remaining budget over the remaining number of steps keeps $\sum_t \beta_t^{\mathrm{used}}\!\leq\!\mathrm{MI}_{\mathrm{tot}}$ on every realization.

\textbf{Privacy guarantee.\;\;} Composing these per-step MI bounds over the run by the chain rule~\eqref{eq:chain-rule} and the adaptive composition of~\citet{zhu2025pacresponses} gives the end-to-end guarantee, which holds for any calibrated channel, and so for \PACAGC{} and \PACABC{} of \S\ref{sec:binary-mi} and \S\ref{sec:abc} (proof in Appendix~\ref{app:pacmi_proof}).
\begin{theorem}[\method{} privacy]\label{thm:pacmi}
Iteratively applying Algorithm~\ref{alg:paczero} for $T$ rounds with a total mutual information budget of $\mathrm{MI}_\mathrm{tot}$ under adaptive budget allocation~\eqref{eq:adaptive-beta}, with the disagreement branch releasing through a channel calibrated at $(q_t^+,\beta_t)$, satisfies
\begin{equation}
I(S^*;\, \tilde Y_{1:T}) \;\leq\; \mathrm{MI}_{\mathrm{tot}}.
\label{eq:pacmi-main}
\end{equation}
The same bound holds for the update signs $Y_{1:T}$, which are computed from $\tilde Y_{1:T}$. Furthermore, end-to-end fine-tuning satisfies $I(S^*;\theta_{1:T+1}) \leq \mathrm{MI}_{\mathrm{tot}}$.
\end{theorem}

\begin{corollary}[MIA guarantee]\label{cor:pacmi-mia}
For an MIA on any record $x_i\!\in\!U$, the informed adversary's prior success is $1\!-\!\delta_0\!=\!1/2$ (\S\ref{sec:background}), and~\eqref{eq:pac_mi_bound} yields
$\mathrm{KL}\!\left(1-\delta_{\mathcal{A}} \;\big\|\; 1/2\right) \;\leq\; \mathrm{MI}_{\mathrm{tot}}$, and the bound is tight (Proposition~\ref{prop:kl-tight}).
\end{corollary}

\textbf{What the adversary knows.\;\;} As in Definition~\ref{def:pac}, we establish the guarantee assuming an informed adversary that knows everything except the sampled index $j^*$ and the channels' random draws.
This includes the universe $U$ and the candidate subsets, the initial parameters, the hyperparameters, the directions $z_t$ and the algorithm itself.
The adversary observes every release $\tilde Y_t$, and therefore every update sign $Y_t$ and every intermediate model. It can then simulate the run and compute, at every step, the sign $s_m$ of every candidate, the posterior $p_t$, $q_t^+$, the budget $\beta_t$ and the channel (Lemma~\ref{lem:replay}). None of these quantities is secret to such a strong adversary, and PAC Privacy, via Theorem~\ref{thm:pacmi}, protects $S^*=S_{j^*}$, the identity of the secret subset, from it.

\subsection{The additive Gaussian channel (\PACAGC{})}
\label{sec:binary-mi}
We first follow the standard PAC Privacy approach~\citep{xiao2023pac,sridhar2025pac,xiao2024formal,zhu2025pacresponses,DBLP:conf/sp/SridharZD26} with an additive Gaussian channel (\PACAGC{}).
On a disagreement step, it releases $\tilde Y_t = s_{j^*} +\mathcal N(0,\sigma_t^2)$, where $\sigma_t$ is set so that $\tilde Y_t$ carries at most $\beta_t$ nats of information about $s_{j^*}$.
These mechanisms invert a variance-based upper bound on MI, which is tight only when the noiseless output is itself Gaussian~\citep{zhang2026residualpac}.
Because our channel input is a single bit with a known distribution, we can instead use the exact MI for tighter calibration. For $\xi\in\{-1,+1\}$ with $P(\xi=+1) = q$ and $\tilde Y = \xi + \nu$, $\nu\sim\mathcal{N}(0,\sigma^2)$, the MI is
\begin{equation}
I_q(\sigma):=
I(\xi;\tilde Y) = \sum_{s\in\{-1,+1\}} P(\xi{=}s) \int \phi_\sigma(y - s)\, \log \frac{\phi_\sigma(y-s)}{q\,\phi_\sigma(y-1) + (1-q)\,\phi_\sigma(y+1)} \,dy,
\label{eq:binary-mi}
\end{equation}
where $\phi_\sigma$ is the Gaussian density with standard deviation $\sigma$. This is a one-dimensional integral, and decreases as $\sigma$ grows, so we can invert it numerically to set $\sigma_t=I_{q_t^+}^{-1}(\beta_t)$ (Appendix~\ref{app:method-details}). The resulting channel is calibrated at $(q_t^+,\beta_t)$ (Lemma~\ref{lem:calibrated}), so Theorem~\ref{thm:pacmi} applies to \PACAGC{}.

\subsection{The distortion-optimal binary asymmetric channel (\PACABC{})}
\label{sec:abc}

Even when calibrated at the exact MI, the Gaussian channel is still loose for \method{}: it spends the MI budget on the information that $\tilde Y_t\in\mathbb R$ carries about $s_{j^*}$, but the optimizer uses only its sign $Y_t$, which carries strictly less information.
This motivates a binary channel that directly releases the bit by itself. 
We choose, among such channels, the one that achieves the minimal flip rate.

A general binary channel, which takes in $\xi\in\{-1,+1\}$ and emits $\tilde Y\in{\{-1,+1\}}$, can be specified by its flip probabilities $a=\Pr[\tilde Y\!=\!{+}1 \mid \xi\!=\!-1]$ and
$b = \Pr[\tilde Y\!=\!-1 \mid \xi\!=\!{+}1]$.
When $\xi$ equals $+1$ with probability $p$, its distortion is $D = \Pr[\tilde Y\!\neq\!\xi] = (1-p)a + pb$. Choosing $(a,b)$ to minimize $D$ subject to $I(\xi;\tilde Y)\le\beta$ is the binary rate-distortion problem~\citep[\S10.3.1]{cover2006elements}, whose solution is
\begin{equation}
D^{*} = h^{-1}\!\big(\max\{h(p)\!-\!\beta,0\}\big), \quad
\pi \;=\; \frac{p - D^{*}}{1 - 2D^{*}}, \quad
a^{*} \;=\; \frac{\pi\,D^{*}}{1-p}, \quad
b^{*} \;=\; \frac{(1-\pi)\,D^{*}}{p}.
\label{eq:abc-calibration}
\end{equation}
Here $h^{-1}$ is the inverse of the binary entropy function on $[0,1/2]$, and $\pi$ is the probability that the channel emits $+1$. \PACABC{} sets its channel by evaluating~\eqref{eq:abc-calibration} with $p=q_t^+$ and $\beta=\beta_t$. Its output is already a sign, so $Y_t=\tilde Y_t$, and none of the budget is spent on information that the optimizer discards. The identity $\pi-1/2=(p-1/2)/(1-2D^*)$ shows that the channel always leans toward the sign that holds more posterior mass, at least as strongly as $p$ does.
By~\eqref{eq:abc-calibration}, $I(\xi;\tilde Y)\!\le\!\beta_t$, so the channel is calibrated at $(q_t^+,\beta_t)$ (Lemma~\ref{lem:calibrated}) and Theorem~\ref{thm:pacmi} applies to \PACABC{}.

\subsection{Releases at zero mutual information (\PACAGCz{} and \PACABCz{})}
\label{sec:zpl}

Taking $\beta_t$ to zero in either channel gives a release that ignores its input $s_{j^*}$. We use these limits as two further variants, \PACAGCz{} and \PACABCz{}, which operate at zero mutual information.
\begin{itemize}
\item For the additive Gaussian channel, as $\beta_t\!\to\!0$ the noise grows without bound, and the sign of the release becomes a fair coin independent of $s_{j^*}$. \PACAGCz{} releases that coin directly on every disagreement step. That is, $\tilde Y_t$ is independently drawn from $\mathrm{Unif}\{-1,+1\}$.
\item For the binary asymmetric channel, with $\beta_t\!=\!0$,~\eqref{eq:abc-calibration} gives $D^{*}\!=\!\min(q_t^+,1\!-\!q_t^+)$. The resulting channel ignores its input and always emits the sign that holds more posterior mass, $+1$ for $q_t^+\!>\!1/2$ and $-1$ for $q_t^+\!<\!1/2$, regardless of what $s_{j^*}$ is.
This is equivalent to releasing the majority sign $\tilde Y_t\!=\!\sgn\!\big(\sum_m p_t[m]\,s_m\big)$, which also covers the case of unanimity.
At $q_t^+\!=\!1/2$,~\eqref{eq:abc-calibration} is undefined, and we emit $+1$, following the convention of $\sgn(0)\!=\!+1$ (Appendix~\ref{app:method-details}). We call this variant \PACABCz{}.
\end{itemize}

In both variants, no release depends on which candidate is the secret, on unanimity and disagreement steps alike. Hence, every candidate has the same likelihood in~\eqref{eq:posterior}, and the posterior $p_t$ stays uniform for the whole run.
The two variants differ in what they release on a disagreement step.
A disagreement step of \PACAGCz{} moves the parameters by $+\eta_t z_t$ or $-\eta_t z_t$ with equal probability, so its progress comes only from unanimity steps, whereas a disagreement step of \PACABCz{} still moves them in the direction that most candidates favor.

Theorem~\ref{thm:pacmi} with $\mathrm{MI}_{\mathrm{tot}}\!=\!0$ covers both variants, as stated in Theorem~\ref{thm:zero} (proof in Appendix~\ref{app:zero_proof}).
\begin{theorem}[Zero MI]\label{thm:zero}\label{thm:abc-zero}\label{thm:zpl-zero}
\PACAGCz{} and \PACABCz{} satisfy $I(S^*;\tilde Y_{1:T})\!=\!0$ for every $T$. Hence no adversary infers whether a record of $U$ is in $S^*$ with probability greater than $1/2$.
\end{theorem}

\begin{remark}[Scope of the zero-MI guarantee]
Theorem~\ref{thm:zero} shows that the run reveals nothing about \emph{which} candidate subset was sampled. It does not say that the release reveals nothing about the records of $U$, which the informed adversary already knows. What a release reveals about $U$ is outside the guarantee, and we discuss it in \S\ref{sec:discussion}. 
While our guarantee bounds MIA success by 50\%, which DP matches at $\varepsilon\!=\!\delta\!=\!0$, it is not equivalent to one. DP at $\varepsilon\!=\!\delta\!=\!0$ is much stronger and would require the same output distribution on every pair of neighboring datasets, which a noise-adding mechanism achieves only with infinite noise.
\end{remark}

\section{Experiments}
\label{sec:experiments}
We evaluate \method{} on the two LLM tasks reported by the strongest DP zeroth-order baseline, DP-AggZO~\citep{bao2025dpaggzo}, namely SST-2 (binary sentiment classification)~\citep{socher2013sst} and SQuAD (extractive QA, F1)~\citep{rajpurkar2016squad}. Each task is run on four models spanning three families and a factor of six in scale, OPT-1.3B and OPT-6.7B~\citep{zhang2022opt}, Qwen2.5-7B~\citep{qwen2025qwen25} and Llama-3.1-8B~\citep{grattafiori2024llama3}. Every model runs in two parameter tracks, LoRA $r{=}8$, $\alpha{=}16$~\citep{hu2022lora} and full-parameter fine-tuning (FT). The main text reports OPT-6.7B and Llama-3.1-8B. OPT-1.3B and Qwen2.5-7B carry the same four mechanisms at the same budgets and are in Appendix~\ref{app:full-grid}. Two tasks, four models and two tracks give sixteen settings in total.

The section supports three claims. (i)~Utility barely depends on the MI budget, down to $\MI(S^*;\tilde Y_{1:T})\!=\!0$. On SST-2 FT OPT-1.3B, \PACAGC{} test accuracy is flat across nearly four decades of MI budget (Table~\ref{tab:plateau} in Appendix~\ref{app:plateau-ft}), and at $\mathrm{MI}\!=\!0$ \PACABCz{} stays within $2.0$ points of the non-private ZO reference at the same step budget on all sixteen settings (Table~\ref{tab:headline} and Appendix~\ref{app:full-grid}). (ii)~The binary asymmetric channel gives higher utility than the Gaussian channel. \PACABC{} is at or above \PACAGC{} on every setting at $\mathrm{MI}\!=\!0.33$, and \PACABCz{} is above \PACAGCz{} on fourteen of the sixteen (\S\ref{sec:headline}). (iii)~Published DP zeroth-order accountants do not reach this MIA resistance at finite per-step noise. The DP MIA bound equals that of $\mathrm{MI}\!=\!0$ only at $\varepsilon\!=\!\delta\!=\!0$, and in our reproduction on SST-2 FT OPT-1.3B every DP run at $\varepsilon\!\leq\!1$, whose MIA bound is weaker, stays below both $\mathrm{MI}\!=\!0$ variants (Table~\ref{tab:dp-cliff} in Appendix~\ref{app:dp-cliff}).

\subsection{Setup}
\label{sec:setup}

We follow the DP-AggZO protocol~\citep{bao2025dpaggzo} on the splits its released code inherits from MeZO~\citep{malladi2023mezo}, using a universe of $N\!=\!1000$ records, a $500$-record dev split, $\delta\!=\!10^{-5}$, and weight decay $\lambda\!=\!0$. Each run trains for $T\!=\!1000$ steps ($2000$ on SST-2 LoRA). The dev split comes from the same seeded permutation as $U$ and is disjoint from it, so no dev record belongs to any candidate subset. Dev-best checkpoint selection is a function of the released transcript and the dev split, so by the data-processing inequality it cannot increase the mutual information and the reported MI accounts for the full $T$-step run. The budgets in Table~\ref{tab:headline} are total $\mathrm{MI}\!\in\!\{0.33, 0.68\}$~nats, approximately matching, through~\eqref{eq:pac_mi_bound}, the MIA upper bound of DP $\varepsilon\!\in\!\{2,6\}$ at $\delta\!=\!10^{-5}$ under the conversion of~\citet{zhu2025pacresponses}. Appendix~\ref{app:full-grid} adds $0.01$ and $0.11$ nats. All four variants run at $M\!=\!128$ subsets, held fixed throughout. Any even $M$ gives the per-record prior of $1/2$, and $M$ sets how many subsets must agree for a step to cost nothing. In each setting the four variants share one learning rate, one clipping threshold $c$ and one step budget, so within a track the \method{} rows of a column of Table~\ref{tab:headline} differ only in the mechanism and its budget. Appendix~\ref{app:protocol-details} summarizes the protocol, Table~\ref{tab:pool-recipes} in Appendix~\ref{app:lr-sweep} gives the hyperparameters of every setting, and Appendix~\ref{app:org} lists the sweeps, ablations, cost measurements and compute.

We compare against two DP zeroth-order baselines. Our primary comparator is DPZero~\citep{zhang2024dpzero}, which, like \method{}, draws a single perturbation direction per step ($K\!=\!1$). We also report DP-AggZO~\citep{bao2025dpaggzo} at $K\!=\!16$, whose per-example forward-pass count approximately matches that of \method{} (Table~\ref{tab:headline-budget}), as a reference for methods that aggregate $K\!>\!1$ directions per step. The OPT-6.7B DP rows of Table~\ref{tab:headline} are taken from Table~2 of~\citet{bao2025dpaggzo}, where each entry is the best over that paper's sampling-rate, learning-rate and clipping-threshold grids while every \method{} entry is selected on dev (implementation parity checked at OPT-1.3B in Appendix~\ref{app:dpaggzo-parity}). Since~\citet{bao2025dpaggzo} do not report DP-LoRA on OPT, all comparisons with DP are in the FT rows. On both tasks we run both DP baselines in-house on Llama-3.1-8B at $\varepsilon\!\in\!\{2,6\}$, at a learning rate selected by a sweep at $\varepsilon\!=\!2$ and used at both $\varepsilon$, and mark those entries $\dagger$. The non-private ZO row is our own zeroth-order fine-tuning on all of $U$ with no release mechanism, at the learning rate its own dev sweep selects (Table~\ref{tab:ref-lr-sweeps}). It runs at $c\!=\!1000$, the loosest clipping threshold of the grid. It measures what zeroth-order fine-tuning reaches at the same step budget without privacy.

\subsection{Main results}
\label{sec:headline}

\newcommand{\gridnote}{Appendix~\ref{app:full-grid} carries the four \method{} variants on all four models.}
\begin{table}[t]
\centering
\small
\setlength{\tabcolsep}{3.5pt}
\caption{\textbf{SST-2 and SQuAD main results (\%).} Mean$\,\pm\,$std over $n\!=\!3$ seeds (per-seed values in Appendix~\ref{app:per-seed}). An entry shown without a spread is single-seed. OPT DP rows are as reported by~\citet{bao2025dpaggzo}. $\dagger$ marks baseline and reference entries we ran ourselves, and every \method{} entry is our own. Table~\ref{tab:headline-budget} in Appendix~\ref{app:vs-dp} gives the hyperparameters and forward-pass count of every row. \textbf{Bold} marks the better of the two $\mathrm{MI}\!=\!0$ variants in each column and does not mark the column maximum. A DP $\varepsilon$ and an MI budget in the same block approximately share the MIA upper bound (\S\ref{sec:background}). \gridnote{} Appendix~\ref{app:protocol-details} describes the non-private ZO reference.}
\label{tab:headline}
\resizebox{\textwidth}{!}{
\begin{tabular}{ll l cc cc}
\toprule
& & & \multicolumn{2}{c}{\textbf{SST-2 (acc.)}} & \multicolumn{2}{c}{\textbf{SQuAD (F1)}} \\
\cmidrule(lr){4-5}\cmidrule(lr){6-7}
Track & Privacy & Method & OPT-6.7B & Llama-3.1-8B & OPT-6.7B & Llama-3.1-8B \\
\midrule
\multirow{11}{*}{FT}
& \multirow{4}{*}{\shortstack[l]{DP $\varepsilon\!=\!2$\\$\mathrm{MI}\!=\!0.33$}}
   & DPZero $K\!=\!1$   & $92.7$ & $92.89^{\dagger}$ & $78.5$ & $88.41^{\dagger}$ \\
& & DP-AggZO $K\!=\!16$ & $93.8$ & $93.81^{\dagger}$ & $82.9$ & $87.96^{\dagger}$ \\
& & \PACAGC{}            & $91.09_{\pm 1.03}$ & $91.44_{\pm 1.08}$ & $75.65_{\pm 1.07}$ & $87.74_{\pm 0.26}$ \\
& & \PACABC{}           & $91.44_{\pm 0.24}$ & $93.12_{\pm 0.34}$ & $76.69_{\pm 1.01}$ & $87.78_{\pm 1.33}$ \\
\cmidrule(l){2-7}
& \multirow{4}{*}{\shortstack[l]{DP $\varepsilon\!=\!6$\\$\mathrm{MI}\!=\!0.68$}}
   & DPZero $K\!=\!1$   & $92.9$ & $92.78^{\dagger}$ & $80.1$ & $88.50^{\dagger}$ \\
& & DP-AggZO $K\!=\!16$ & $94.6$ & $92.78^{\dagger}$ & $83.3$ & $87.53^{\dagger}$ \\
& & \PACAGC{}            & $91.28_{\pm 1.09}$ & $91.63_{\pm 0.75}$ & $75.45_{\pm 1.61}$ & $88.07_{\pm 0.96}$ \\
& & \PACABC{}           & $92.16_{\pm 0.48}$ & $93.12_{\pm 0.50}$ & $75.84_{\pm 0.78}$ & $87.74_{\pm 1.19}$ \\
\cmidrule(l){2-7}
& \multirow{2}{*}{$\mathrm{MI}\!=\!0$}
   & \PACAGCz{}          & $89.87_{\pm 0.63}$ & $92.20_{\pm 0.61}$ & $74.83_{\pm 1.32}$ & $87.39_{\pm 1.58}$ \\
& & \PACABCz{}          & $\bm{91.63_{\pm 0.30}}$ & $\bm{92.78_{\pm 0.23}}$ & $\bm{77.10_{\pm 1.39}}$ & $\bm{87.80_{\pm 1.12}}$ \\
\cmidrule(l){2-7}
& none & non-private ZO reference$^{\dagger}$ & $93.58$ & $92.43_{\pm 0.61}$ & $75.20_{\pm 0.91}$ & $89.01$ \\
\midrule
\multirow{7}{*}{LoRA}
& \multirow{2}{*}{$\mathrm{MI}\!=\!0.33$}
   & \PACAGC{}           & $92.62_{\pm 0.33}$ & $88.80_{\pm 2.66}$ & $73.93_{\pm 1.51}$ & $88.19_{\pm 0.36}$ \\
& & \PACABC{}           & $93.00_{\pm 0.91}$ & $90.25_{\pm 0.75}$ & $75.44_{\pm 1.39}$ & $88.89_{\pm 1.69}$ \\
\cmidrule(l){2-7}
& \multirow{2}{*}{$\mathrm{MI}\!=\!0.68$}
   & \PACAGC{}           & $92.39_{\pm 0.83}$ & $90.52_{\pm 2.36}$ & $73.30_{\pm 1.49}$ & $88.61_{\pm 1.41}$ \\
& & \PACABC{}           & $93.43_{\pm 0.57}$ & $90.79_{\pm 0.70}$ & $76.10_{\pm 2.56}$ & $88.93_{\pm 1.59}$ \\
\cmidrule(l){2-7}
& \multirow{2}{*}{$\mathrm{MI}\!=\!0$}
   & \PACAGCz{}          & $\bm{93.62_{\pm 0.13}}$ & $87.69_{\pm 0.63}$ & $72.12_{\pm 1.81}$ & $88.05_{\pm 1.22}$ \\
& & \PACABCz{}          & $93.31_{\pm 0.43}$ & $\bm{89.37_{\pm 2.84}}$ & $\bm{75.44_{\pm 0.77}}$ & $\bm{88.20_{\pm 2.21}}$ \\
\cmidrule(l){2-7}
& none & non-private ZO reference$^{\dagger}$ & $92.78$ & $91.17$ & $75.73$ & $86.98$ \\
\bottomrule
\end{tabular}
}%
\end{table}

Utility depends little on the MI budget. On SST-2 FT OPT-1.3B we hold the learning rate, clipping threshold and step budget fixed and sweep the total MI budget of \PACAGC{} across nearly four decades, from $10^{-4}$ to $0.68$ nats. Test accuracy stays between $88.30$ and $90.48\%$, and the unanimity fraction, the share of steps that spend no budget, stays between $52.4$ and $56.0\%$ (Table~\ref{tab:plateau} and Figure~\ref{fig:plateau} in Appendix~\ref{app:plateau-ft}). At $\mathrm{MI}\!=\!0.33$, the Gaussian release flips its sign almost as often as a fair coin on disagreement steps (Appendix~\ref{app:protocol-details}). Most of the optimization signal at this budget therefore comes from unanimity steps. Across settings, the unanimity fraction ranges from $12.8$ to $73.7\%$ (Table~\ref{tab:unanimity-pools} in Appendix~\ref{app:unanimity-headline}). Appendix~\ref{app:lora-plateau} repeats the sweep on the LoRA track. In every setting, \PACABC{} moves by at most $1.76$ points between $\mathrm{MI}\!=\!0$ and $\mathrm{MI}\!=\!0.68$ (Appendix~\ref{app:full-grid}).

The binary asymmetric channel improves on the Gaussian channel. At $\mathrm{MI}\!=\!0.33$, \PACABC{} is at or above \PACAGC{} on every setting, by up to $3.44$ points on SST-2 FT for Qwen2.5-7B. At $\mathrm{MI}\!=\!0.68$ it is at or above \PACAGC{} on thirteen of the sixteen settings and trails on the other three by at most $0.45$ points (Appendix~\ref{app:full-grid}).

The same holds at zero budget. At $\MI(S^*;\tilde Y_{1:T})\!=\!0$, \PACABCz{} improves on \PACAGCz{} on fourteen of the sixteen settings, by up to $3.86$ F1 on SQuAD LoRA for OPT-1.3B, and trails it by at most $0.31$ points on the other two, both SST-2 LoRA on OPT. The two variants differ only in what the disagreement branch releases (\S\ref{sec:zpl}) and both satisfy Theorem~\ref{thm:zero}, so the gap measures what the majority of the subset signs buys over an unbiased coin at identical privacy. On six of the sixteen settings \PACABCz{} at $\mathrm{MI}\!=\!0$ is above \PACABC{} at $\mathrm{MI}\!=\!0.68$. At $\mathrm{MI}\!=\!0$, \PACABCz{} is at most $1.95$ points below the non-private ZO reference on all sixteen settings and above it on seven (Appendix~\ref{app:full-grid}).

We also verify Theorem~\ref{thm:zero} empirically. Across all settings, nine black-box membership-inference attacks, among them RMIA and LiRA, separate the records of $S^*$ from the rest of $U$ no more often than chance. Their median AUC lies between $0.489$ and $0.507$ (Appendix~\ref{app:mia}).

We now turn to DP. The comparisons in Table~\ref{tab:headline} set \PACABCz{} at $\MI(S^*;\tilde Y_{1:T})\!=\!0$ against DP at $\varepsilon\!=\!2$, whose MIA bound is much weaker (\S\ref{sec:background}). Against DPZero, the primary comparator, \PACABCz{} is at most $1.40$ points below on all four FT settings. On Llama-3.1-8B FT, \PACABCz{} is within $1.03$ points of the strongest DP baseline on both tasks. On OPT-6.7B FT, DP-AggZO $K\!=\!16$ leads by $2.2$ points on SST-2 and $5.8$ F1 on SQuAD (Appendix~\ref{app:vs-dp}).

In the high-privacy regime the order reverses. We reproduce both DP baselines on SST-2 FT OPT-1.3B at $\varepsilon\!\in\!\{0.2,0.3,0.5,1.0,2.0\}$, with DPZero at DP-AggZO's learning rate, clipping threshold and sampling rate (Table~\ref{tab:dp-cliff} in Appendix~\ref{app:dp-cliff}). DPZero $K\!=\!1$ stays at or near chance for every $\varepsilon\!<\!1$. At $\varepsilon\!\leq\!0.5$, DP-AggZO $K\!=\!64$ improves when its run is shortened to about the per-example forward passes of \method{}, since a shorter run composes fewer times and needs less noise per step. Its best run at $\varepsilon\!\leq\!1$ reaches $86.93\%$. In the same setting \PACABCz{} and \PACAGCz{} at $\mathrm{MI}\!=\!0$ reach $90.10\%$ and $89.83\%$, above every DP run at $\varepsilon\!\leq\!1$ and within half a point of the strongest in-house DP run at $\varepsilon\!=\!2$. \method{} thus produces usable accuracy at $\mathrm{MI}\!=\!0$, a level that published DP zeroth-order accountants~\citep{zhang2024dpzero,tang2024dpzo,bao2025dpaggzo} reach only at $\varepsilon\!=\!\delta\!=\!0$ with infinite per-step noise.

\section{Related Work}
\label{sec:rworks}

A mature line of work adapts DP-SGD~\citep{DBLP:conf/ccs/AbadiCGMMT016} to LLM fine-tuning, with refinements to clipping and adapter integration~\citep{DBLP:conf/iclr/YuNBGI0KLMWYZ22,li2022strong,he2023exploring,bu2023scalable,bu2024automatic,zhang2024dpsgdwithoutclippingbias,DBLP:journals/tmis/LiuZZGZWQ25}. DP zeroth-order methods privatize only the per-step finite-difference scalars, one per sampled direction, which removes explicit dimension dependence from the per-step noise. \citet{zhang2024dpzero} (DPZero) prove nearly dimension-independent rates with a Gaussian scalar mechanism. DP-ZO~\citep{tang2024dpzo} studies the same scalar privatization empirically, with Gaussian or Laplace noise, the latter giving pure $\varepsilon$-DP. \citet{bao2025dpaggzo} (DP-AggZO) aggregate $K$ finite-difference scalars into a vector before clipping and report the strongest results among DP zeroth-order methods on OPT-1.3B/6.7B. \citet{liu2025dpzopo} schedule the ZO hyperparameters in stages and prune parameters without data. PAC Privacy~\citep{xiao2023pac} measures and controls what a data-processing algorithm reveals about a secret drawn from a known distribution. Later work gives PAC Privacy guarantees for classical algorithms~\citep{sridhar2025pac,DBLP:conf/sp/SridharZD26}, quantifies the leakage of learnable obfuscation~\citep{xiao2024formal}, and brings PAC Privacy to SQL queries~\citep{sridhar2026pacdb,battiston2026simd} and private prediction~\citep{zhu2025pacresponses}. 
Additionally, \citet{DBLP:conf/ccs/Xiao0SD25} study one-sided noise for PAC Privacy and \citet{zhang2026residualpac} use $f$-divergence for tighter noise calibration.
\method{} applies PAC Privacy to LLM fine-tuning. No per-record sensitivity enters its accounting, and per-sample clipping serves only as a utility stabilizer. The sign-quantized update shares structure with signSGD~\citep{bernstein2018signsgd,liu2019zosignsgd}, and \method{} calibrates its release for privacy while signSGD quantizes for compression.

\section{Discussion and Limitations}
\label{sec:discussion}

\textbf{PAC Privacy and DP.\;\;} DP gives a worst-case guarantee that does not depend on the data. Under DP, the output distributions on every pair of neighboring datasets must be close, which makes it the stronger guarantee. PAC Privacy gives an instance-based guarantee with respect to a distribution $\mathcal{D}$ over the input data, and calibrates its noise to how much the output varies when the input is drawn from $\mathcal{D}$ (\S\ref{sec:adaptive}). After sign quantization and subset aggregation this variation can be far smaller than the largest change a single record can make, and on unanimity steps it is zero. DP relaxations that use instance-specific information still charge for any single record's influence (Appendix~\ref{app:dp-relaxations2}).

\textbf{The guarantee and the MIA game.\;\;} For most training sets no distribution over the secret data can be specified, since the data were drawn once from an unknown source and cannot be drawn again. The MIA game supplies one~\cite{DBLP:conf/sp/ShokriSSS17,yeom2018privacy,DBLP:conf/sp/CarliniCN0TT22}: a challenger samples a training set from a universe of records and trains a model on it, and an adversary, who may know the universe, tries to decide whether a given record was in the training set. Our construction of \S\ref{sec:background} is an instance of this game, with the training set drawn uniformly from the candidate subsets of $U$. \method{} bounds what an adversary can infer about the secret subset $S^*$ even when it knows $U$, and therefore also when it does not. Furthermore, the bound covers any other inference whose target is $S^*$, such as reconstructing it~\cite{xiao2024formal}. At $I(S^*;\tilde Y_{1:T}) = 0$ the release tells the adversary \emph{nothing} about $S^*$ beyond what it knew before training.

\textbf{Distributional leakage.\;\;} Since the guarantee is stated with respect to $\mathcal{D}$, the distribution itself lies outside it. We call what a release reveals about $\mathcal{D}$ \emph{distributional leakage}. PAC Privacy states its guarantee for an adversary that already knows $\mathcal D$, that is, $U$ and the candidate subsets for \method{}, so the mechanism may use $U$ without charge to the budget, and all four variants do. Every step evaluates each record of $U$, and the release depends on how the posterior mass over the $M$ candidate subsets splits between the two signs, which is $q^+_t$. The variants use this split to different extents. The additive Gaussian channel of \PACAGC{} uses it only through how evenly the candidates are divided, which fixes the variance of the secret subset's sign. The binary asymmetric channel of \PACABC{} additionally uses which sign holds the larger share and tilts its output toward that sign. At zero budget, \PACAGCz{} uses only whether a step is unanimous and, when it is, the common sign, while \PACABCz{} releases the majority sign on every step. This use of $U$ is the stability that PAC Privacy exploits, and it is what makes zero-cost steps possible. A realistic adversary that does not know $U$ may therefore learn something about $U$ from the release, and the guarantee does not bound this. 
Whether such an adversary, which observes only the releases $\tilde Y_{1:T}$, can recover $q_t^+$ from them, let alone reconstruct properties of $U$, is an open question for future work.

\section{Conclusion}
\label{sec:conc}
We introduced \method{}, a family of PAC-private zeroth-order mechanisms for fine-tuning large language models, built around one primitive, which is to sign-quantize subset-aggregated ZO updates and release the sign of the secret subset, through a calibrated channel when the candidates disagree. Quantizing to one bit admits exact mutual-information calibration, creates frequent unanimity, on which a step costs zero conditional MI, and lets one accounting argument cover both channels at every budget. At $\MI(S^*;\tilde Y_{1:T})\!=\!0$, \PACABCz{} stays within $2.0$ points of the non-private ZO reference at the same step budget on all sixteen settings of Table~\ref{tab:headline} and Appendix~\ref{app:full-grid}, reaching $\bm{91.63\!\pm\!0.30\%}$ on SST-2 OPT-6.7B FT and $\bm{87.80\!\pm\!1.12}$ F1 on SQuAD Llama-3.1-8B FT. Quantizing before calibrating noise opens an MIA-resistance level, $\MI(S^*;\tilde Y_{1:T})\!=\!0$, that no published DP zeroth-order accountant reaches at finite per-step noise. Natural next steps are a convergence analysis of \method{}, larger corpora, which need minibatched subset estimation with a new privacy analysis, and settings beyond supervised fine-tuning.

\section*{Acknowledgments}
\label{sec:acknow}
The contribution of Marten van Dijk and Murat Bilgehan Ertan to this publication is part of the project CiCS of the research program Gravitation which is (partly) financed by the Dutch Research Council (NWO) under the grant 024.006.037. Xiaochen Zhu is supported by the MIT Schwarzman College of Computing PhD Fellowship. We acknowledge the use of the DAS-6 High-Performance Computing cluster at Vrije Universiteit Amsterdam for GPU-based experiments~\citep{bal2016medium}. This work used the Dutch national e-infrastructure with the support of the SURF Cooperative using grant nos. EINF-20085 and EINF-18721. The authors declare no competing interests.

\bibliographystyle{plainnat}
\bibliography{references}

\appendix

\section{Appendix Organization}
\label{app:org}

Appendix is organized as follows.
\begin{itemize}\setlength{\itemsep}{2pt}
  \item Appendix~\ref{app:proofs} proves Theorem~\ref{thm:pacmi} (Appendix~\ref{app:pacmi_proof}), through Lemma~\ref{lem:replay}, which shows that the state of the mechanism is determined by the earlier channel outputs, and Lemma~\ref{lem:per-step-mi}, which bounds the information released in one step. It then shows that the channels of all four variants are calibrated (Appendix~\ref{app:calibrated}), proves Theorem~\ref{thm:zero} (Appendix~\ref{app:zero_proof}), and proves Proposition~\ref{prop:kl-tight} (Appendix~\ref{app:kl-tight}), that the MIA bound cannot be improved from the budget alone.
  \item Appendix~\ref{app:method-details} collects the implementation details of \method{}, namely the role of per-sample clipping, numerical evaluation of the Gaussian-channel MI integral, the edge cases of the \PACABC{} calibration, the sign convention at zero, the unanimity tolerance, and the entropy-ceiling cap on $\beta_t$.
  \item Appendix~\ref{app:dp-relaxations2} compares \method{} with DP relaxations.
  \item Appendix~\ref{app:protocol-details} carries the protocol and baseline detail behind \S\ref{sec:setup}.
  \item Appendix~\ref{app:full-grid} gives all four variants on all sixteen settings (Tables~\ref{tab:full-grid-sst2} and~\ref{tab:full-grid-squad}).
  \item Appendix~\ref{app:unanimity-headline} reports the unanimity fraction $f$ per setting.
  \item Appendix~\ref{app:mia} tests the $\mathrm{MI}\!=\!0$ guarantee empirically with nine black-box membership-inference attacks.
  \item Appendix~\ref{app:plateau-ft} gives the MI-budget sweep summarized in \S\ref{sec:headline} (Table~\ref{tab:plateau}).
  \item Appendix~\ref{app:lora-plateau} reports the LoRA counterpart of the same sweep (Table~\ref{tab:LoRA-plateau-adaptive}).
  \item Appendix~\ref{app:dpaggzo-parity} reports the parity check at OPT-1.3B between our in-house DP baselines and Table~2 of~\citet{bao2025dpaggzo}.
  \item Appendix~\ref{app:dp-cliff} reports the in-house DP baseline reproduction (Table~\ref{tab:dp-cliff}).
  \item Appendix~\ref{app:vs-dp} sets \method{} against the DP comparators.
  \item Appendix~\ref{app:lr-sweep} reports the hyperparameters of every setting (Table~\ref{tab:pool-recipes}).
  \item Appendix~\ref{app:lr-sweeps-all} gives the learning-rate sweep that selects the learning rate for every setting and for the non-private reference (Tables~\ref{tab:lr-sweeps} and~\ref{tab:ref-lr-sweeps}).
  \item Appendix~\ref{sec:clip-ablation} reports the clipping-threshold ablations, namely the LoRA grid on SST-2 across all four models, the FT grid on both tasks across all four models, and the SQuAD LoRA grid.
  \item Appendix~\ref{app:rank-ablation} reports the 6.7B LoRA rank ablation.
  \item Appendix~\ref{app:mechanism-ablation} decomposes the \method{} per-step release into deterministic surrogates and reports the random-sign negative control.
  \item Appendix~\ref{app:zpl-canonical} reports the \PACAGCz{} step-budget sweep on SST-2 OPT-1.3B FT (Table~\ref{tab:zpl-ft-ladder}).
  \item Appendix~\ref{app:k-ablation} extends the per-step release to $K\!\in\!\{4,16,64\}$ following the DP-AggZO-style aggregation pattern on SST-2.
  \item Appendix~\ref{app:scale-family} reports the universe-size sweep (Table~\ref{tab:n-sweep}).
  \item Appendix~\ref{app:cost-profile} profiles the mechanism's cost against the non-private ZO reference at matched settings and reports how cost scales in the universe size $N$ and the subset count $M$ (Tables~\ref{tab:cost-profile} and~\ref{tab:cost-sweeps}).
  \item Appendix~\ref{app:compute} reports the hardware and the GPU-hours by kind of experiment.
  \item Appendix~\ref{app:per-seed} gives per-seed values for every multi-seed \method{} entry of Table~\ref{tab:headline} and of Appendix~\ref{app:full-grid}.
\end{itemize}

\section{Deferred full proofs}\label{app:proofs}

\subsection{Proof of Theorem~\ref{thm:pacmi}}
\label{app:pacmi_proof}\label{app:replay_proof}\label{app:per_step_mi_proof}
Throughout this subsection we fix the directions $z_1,\ldots,z_T$. The universe, the candidate subsets, the initial parameters and the hyperparameters are fixed as well, so the only randomness is the secret index $j^*$, which is uniform on $[M]$, and the channels' randomness, which each step draws independently of $j^*$ and of the earlier steps. We first show that the state of the mechanism is determined by the earlier channel outputs and that $p_t$ is the exact posterior. We then bound what one step releases and sum over the steps. The last part of the proof removes the assumption that the directions are fixed.

\begin{lemma}\label{lem:replay}
For every step $t$, the parameters $\theta_t$, the subset signs $s_1,\ldots,s_M$, the posterior $p_t$, the probability $q_t^+$, the budget $\beta_t$ and the channel $\mathcal{C}_t$ are functions of the earlier channel outputs $\tilde Y_{1:t-1}$. Moreover, $p_t[m]=\Pr[j^*\!=\!m\mid \tilde Y_{1:t-1}]$ for every $m$.
\end{lemma}
\begin{proof}
We argue by induction on $t$. At the first step the parameters are the initial ones, the posterior is uniform, and $\beta_1\!=\!\mathrm{MI}_{\mathrm{tot}}/T$. Suppose the first claim holds at every step up to $t$. The parameters $\theta_t$ and the direction $z_t$ fix every per-sample scalar $\hat g_i$, hence every subset mean and every sign $s_m$, and with $p_t$ they fix $q_t^+$. The budget $\beta_t$ follows from~\eqref{eq:adaptive-beta} and the charges of the earlier steps, each of which is $0$ or $\min\{\beta_\tau,h(q_\tau^+)\}$ and so is already fixed. The channel is a function of $q_t^+$ and $\beta_t$. After the output $\tilde Y_t$, the update rule gives $\theta_{t+1}$ from $\theta_t$, $z_t$ and $\sgn(\tilde Y_t)$, and the update~\eqref{eq:posterior} gives $p_{t+1}$ from $p_t$, the signs and $\tilde Y_t$. This proves the first claim at step $t\!+\!1$.

For the second claim, $p_1$ is uniform, which is the distribution of $j^*$. Suppose that $p_t[m]=\Pr[j^*\!=\!m\mid \tilde Y_{1:t-1}]$. By the first claim, $\tilde Y_{1:t-1}$ fixes the channel $\mathcal{C}_t$, and its randomness at step $t$ is independent of $j^*$ and of the earlier steps, so given $\tilde Y_{1:t-1}$ and $j^*\!=\!m$ the output $\tilde Y_t$ is distributed as $\mathcal{C}_t(s_m)$, with probability (or density) $L_t(\cdot\mid s_m)$. Bayes' rule then gives
\[
\Pr[j^*\!=\!m\mid \tilde Y_{1:t}]\;\propto\;\Pr[j^*\!=\!m\mid \tilde Y_{1:t-1}]\,L_t(\tilde Y_t\mid s_m)\;=\;p_t[m]\,L_t(\tilde Y_t\mid s_m),
\]
which is the update~\eqref{eq:posterior}. On a unanimity step $\tilde Y_t\!=\!s_{j^*}$ and every candidate $m$ with $p_t[m]\!>\!0$ has $s_m\!=\!\tilde Y_t$, so the likelihood is the same for all of them and $p_{t+1}\!=\!p_t$, as in Algorithm~\ref{alg:paczero}. In each case $p_{t+1}[m]=\Pr[j^*\!=\!m\mid \tilde Y_{1:t}]$.
\end{proof}

\begin{lemma}[One step]\label{lem:per-step-mi}
Fix a step $t$ and a realization $y_{1:t-1}$ of $\tilde Y_{1:t-1}$, and let $q_t^+$ and $\beta_t$ be the values it determines. Then $I(j^*;\tilde Y_t\mid \tilde Y_{1:t-1}\!=\!y_{1:t-1})\leq\beta_t^{\mathrm{used}}$, where $\beta_t^{\mathrm{used}}\!=\!0$ on a unanimity step and $\beta_t^{\mathrm{used}}\!=\!\min\{\beta_t,h(q_t^+)\}$ on a disagreement step.
\end{lemma}
\begin{proof}
By Lemma~\ref{lem:replay}, the realization fixes the signs, $q_t^+$, $\beta_t$ and the channel, and given it $j^*$ has distribution $p_t$. In the unanimity case, $q_t^+\!\in\!\{0,1\}$. If $q_t^+=1$, then $s_m=+1$ for all $m$ with $p_t[m]\!>\!0$, so $s_{j^*}=+1$ with probability~1 and $\tilde Y_t = s_{j^*} = +1$ is a constant. Symmetrically if $q_t^+=0$. A constant carries zero MI.

In the disagreement case, $q_t^+\!\in\!(0,1)$. Define $\xi := s_{j^*}\!\in\!\{-1,+1\}$, a deterministic function of $j^*$ given the fixed signs, with $P(\xi\!=\!+1)\!=\!q_t^+$. The channel's randomness at step $t$ is independent of $j^*$ and of the earlier steps, so $\tilde Y_t=\mathcal{C}_t(\xi)$ depends on $j^*$ only through $\xi$, and $j^*\!\to\!\xi\!\to\!\tilde Y_t$ is a Markov chain. By the chain rule,
\[
I(j^*;\,\tilde Y_t) \;=\; I(\xi;\,\tilde Y_t) \;+\; \underbrace{I(j^*;\,\tilde Y_t \mid \xi)}_{=\,0},
\]
where the second term vanishes because $\tilde Y_t$ is conditionally independent of $j^*$ given $\xi$. The channel is calibrated at $(q_t^+,\beta_t)$, so $I(\xi;\tilde Y_t)\!\leq\!\beta_t$. A binary input also carries at most its entropy, $I(\xi;\tilde Y_t)\!\le\!H(\xi)\!=\!h(q_t^+)$, so $I(\xi;\tilde Y_t)\!\le\!\min\{\beta_t,h(q_t^+)\}\!=\!\beta_t^{\mathrm{used}}$.
\end{proof}

\begin{proof}[Proof of Theorem~\ref{thm:pacmi}]
Since $S^*$ is a function of $j^*$, the data-processing inequality and the chain rule give
\begin{equation}
I(S^*;\tilde Y_{1:T})\;\leq\;I(j^*;\tilde Y_{1:T})\;=\;\sum_{t=1}^T I(j^*;\tilde Y_t\mid \tilde Y_{1:t-1})\;=\;\sum_{t=1}^T \E\big[I(j^*;\tilde Y_t\mid \tilde Y_{1:t-1}\!=\!y_{1:t-1})\big],
\label{eq:proof-chain}
\end{equation}
where the expectation is over the realization $y_{1:t-1}$ of $\tilde Y_{1:t-1}$. By Lemma~\ref{lem:per-step-mi}, the $t$-th term is at most $\E[\beta_t^{\mathrm{used}}]$, so
\begin{equation}
I(S^*;\, \tilde Y_{1:T}) \;\leq\; \E\!\left[\sum_{t=1}^T \beta_t^{\mathrm{used}}\right] \;\leq\; \mathrm{MI}_{\mathrm{tot}},
\label{eq:proof-sum}
\end{equation}
where the second inequality holds because $\sum_t \beta_t^{\mathrm{used}} \leq \mathrm{MI}_{\mathrm{tot}}$ for every realization (by the adaptive budget allocation~\eqref{eq:adaptive-beta}).

The update signs $Y_t=\sgn(\tilde Y_t)$ are a function of $\tilde Y_{1:T}$, and so are the parameters, which the update rule of Algorithm~\ref{alg:paczero} computes from them with the directions fixed. By the data-processing inequality, $I(S^*;Y_{1:T})$ and $I(S^*;\theta_{1:T+1})$ are at most $I(S^*;\tilde Y_{1:T})\leq\mathrm{MI}_{\mathrm{tot}}$.

It remains to let the directions $Z\!=\!z_{1:T}$ be random. They are drawn independently of $j^*$ and of the channels' randomness, so $I(S^*;Z)\!=\!0$, and the bound above holds for every value of $Z$. Hence $I(S^*;\tilde Y_{1:T},Z) = I(S^*;Z)+I(S^*;\tilde Y_{1:T}\mid Z)\leq\mathrm{MI}_{\mathrm{tot}}$. The update signs $Y_{1:T}$ and the parameters $\theta_{1:T+1}$ are functions of $(\tilde Y_{1:T},Z)$, so the bound holds for them even when the adversary also sees the directions.
\end{proof}

\subsection{The channels are calibrated}
\label{app:calibrated}
\begin{lemma}\label{lem:calibrated}
Let $q\!\in\!(0,1)$. For every $\beta\!>\!0$, the Gaussian channel with $\sigma\!=\!I_q^{-1}(\beta)$ and the binary asymmetric channel~\eqref{eq:abc-calibration} are calibrated at $(q,\beta)$. The releases of \PACAGCz{} and \PACABCz{} on a disagreement step are calibrated at $(q,0)$.
\end{lemma}
\begin{proof}
Let $\xi$ be a bit with $P(\xi\!=\!+1)\!=\!q$. We bound $I(\xi;\mathcal{C}(\xi))$ for each channel.

For the Gaussian channel, $I_q(\sigma)$ decreases continuously from $h(q)$ at $\sigma\!=\!0$, where the channel adds no noise, to $0$ as $\sigma\!\to\!\infty$. We take $I_q^{-1}(\beta)$ to be the smallest $\sigma\!\geq\!0$ with $I_q(\sigma)\!\leq\!\beta$, so $I_q(I_q^{-1}(\beta))\!\leq\!\beta$. When $\beta\!<\!h(q)$ this holds with equality, and when $\beta\!\geq\!h(q)$ we have $I_q^{-1}(\beta)\!=\!0$.

For the binary asymmetric channel, suppose first that $\beta\!\geq\!h(q)$. Then $D^*\!=\!0$, $\pi\!=\!q$ and $a^*\!=\!b^*\!=\!0$, so the channel is the identity and $I(\xi;\tilde Y)\!=\!h(q)\!\leq\!\beta$. Suppose next that $\beta\!<\!h(q)$. Then $h(D^*)\!=\!h(q)-\beta\!>\!0$. Since $h$ increases on $[0,1/2]$ and $h(D^*)\!<\!h(q)\!=\!h(\min\{q,1\!-\!q\})$, we have $0\!<\!D^*\!<\!\min\{q,1\!-\!q\}$, and therefore $0\!<\!\pi\!<\!1$. Consider the joint distribution in which $\tilde Y\!=\!+1$ with probability $\pi$ and, given $\tilde Y$, $\xi$ differs from $\tilde Y$ with probability $D^*$. Its marginal is $P(\xi\!=\!+1)=\pi(1\!-\!D^*)+(1\!-\!\pi)D^*=q$ by the definition of $\pi$, and its conditional probabilities are $\Pr[\tilde Y\!=\!+1\mid\xi\!=\!-1]=\pi D^*/(1\!-\!q)=a^*$ and $\Pr[\tilde Y\!=\!-1\mid\xi\!=\!+1]=(1\!-\!\pi)D^*/q=b^*$. So it is the joint distribution of the input and the output of the channel~\eqref{eq:abc-calibration}. Given either output, $\xi$ differs from it with probability $D^*$, so $H(\xi\mid\tilde Y)\!=\!h(D^*)$ and $I(\xi;\tilde Y)=h(q)-h(D^*)=\beta$.

At zero budget, the coin of \PACAGCz{} and the release of \PACABCz{} do not depend on $\xi$, so $I(\xi;\mathcal{C}(\xi))\!=\!0$.
\end{proof}

\subsection{Proof of Theorem~\ref{thm:zero}}
\label{app:zero_proof}
\begin{proof}[Proof of Theorem~\ref{thm:zero}]
With $\mathrm{MI}_{\mathrm{tot}}\!=\!0$, the allocation~\eqref{eq:adaptive-beta} sets $\beta_t\!=\!0$ on every step, so Theorem~\ref{thm:pacmi} applies to any variant whose disagreement steps release through a channel calibrated at $(q_t^+,0)$. By Lemma~\ref{lem:calibrated}, the coin of \PACAGCz{} and the release of \PACABCz{} are such channels. Theorem~\ref{thm:pacmi} therefore gives $I(S^*;\tilde Y_{1:T})\!\leq\!0$, so $I(S^*;\tilde Y_{1:T})\!=\!0$, and the same holds for the parameters $\theta_{1:T+1}$.

For the membership claim, $I(S^*;\tilde Y_{1:T})\!=\!0$ makes the release independent of $S^*$, and hence of the uniform bit $\mathbf{1}[x_i\!\in\!S^*]$, so every guess computed from the release is correct with probability exactly $1/2$.
\end{proof}

\subsection{The MIA bound cannot be improved from the budget alone}
\label{app:kl-tight}
\begin{proposition}\label{prop:kl-tight}
Fix a record $x_i$, and let $A_i$ be the $M/2$ candidates that contain it. For any release $W$ with $I(j^*;W)\!\le\!\mathrm{MI}_{\mathrm{tot}}$, the success probability $s$ of the best guess of whether $x_i\!\in\!S^*$ satisfies $\mathrm{KL}(s\,\|\,1/2)\!\le\!\mathrm{MI}_{\mathrm{tot}}$. For every $s\!\in\![1/2,1]$ some release attains success $s$ with $I(j^*;W)\!=\!\mathrm{KL}(s\,\|\,1/2)$.
\end{proposition}
\begin{proof}
Let $p$ be the posterior of $j^*$ given $W\!=\!w$, and let $r\!=\!p(A_i)$ be its mass on $A_i$. The best guess succeeds with probability $\max\{r,1\!-\!r\}$, and both $A_i$ and its complement contain $M/2$ candidates. Grouping the entropy of $p$ by the two halves gives $H(p)\!\le\!h(r)+\ln(M/2)$, with equality when $p$ is uniform within each half. Hence $\ln M-H(p)\!\ge\!\ln2-h(r)\!=\!\mathrm{KL}(\max\{r,1\!-\!r\}\,\|\,1/2)$. The function $\phi(s)\!=\!\ln2-h(s)$ is convex, and the average of $\ln M-H(p)$ over $w$ is $I(j^*;W)$, so Jensen's inequality gives $\mathrm{KL}(s\,\|\,1/2)\!\le\!I(j^*;W)\!\le\!\mathrm{MI}_{\mathrm{tot}}$ for the average success $s$. For attainment, let $V\!=\!\mathbf 1[x_i\!\in\!S^*]$, which is uniform, and release $W\!=\!V$ with probability $s$ and $W\!=\!1\!-\!V$ otherwise, independently of everything else. The posterior given $W$ is uniform within each half, the best guess succeeds with probability $s$, and $I(j^*;W)\!=\!I(V;W)\!=\!\ln2-h(s)$.
\end{proof}
The proposition reproves~\eqref{eq:pac_mi_bound} for this construction and shows that no bound depending only on $\mathrm{MI}_{\mathrm{tot}}$ can improve on it.

\section{\method{} implementation details}
\label{app:method-details}

\paragraph{The finite-difference error.}
The $O(\mu^2)$ term in~\eqref{eq:zo-scalar} holds when the Hessian of $\ell_i$ is Lipschitz, by a second-order Taylor expansion~\citep{nesterov2017random}.

\paragraph{The role of per-sample clipping.}
Per-sample clipping at the threshold $c$ before subset aggregation is a \emph{utility} stabilizer. It bounds any single record's leverage on the subset mean. A sign change from one record stays possible when the remaining subset sum is close enough to zero, so clipping bounds a record's influence without ruling out a flip. Because the release depends on the secret only through the sign $s_{j^*}$, and the channel accounts for all leakage of $s_{j^*}$, the magnitude of $\hat g_i$ does not enter the privacy proof. Clipping only affects the empirical sign distribution. Appendix~\ref{sec:clip-ablation} shows it is utility-relevant, with dev accuracy spanning $1.0$ to $2.2$ points across the clipping-threshold grid on SST-2 LoRA, but it needs no calibration to a record-level sensitivity.

\paragraph{Numerical evaluation of the Gaussian-channel MI integral.}
This paragraph and the entropy-ceiling cap below concern the Gaussian channel of \S\ref{sec:binary-mi} only. The binary asymmetric channel of \S\ref{sec:abc} is calibrated in closed form by~\eqref{eq:abc-calibration}.

We evaluate~\eqref{eq:binary-mi} by 60-node Gauss-Hermite quadrature and find $\sigma_t$ by bisection in $\log\sigma$, since $I_q(\sigma)$ decreases in $\sigma$. The accountant charges $\beta_t$ whatever the solver returns, so the bisection targets $0.999\,\beta_t$, a margin that exceeds the quadrature's measured error over the operating range, and returns the larger-$\sigma$ endpoint. Hence $\sigma_t\!\geq\!I_q^{-1}(\beta_t)$ and $I_q(\sigma_t)\!\leq\!\beta_t$. Below $\beta_t\!=\!10^{-10}$ the quadrature loses its significant digits, so the solver bisects the analytic bound $I_q(\sigma)\!\leq\!2q(1\!-\!q)/\sigma^2$ there. An independent $60$-digit evaluation of~\eqref{eq:binary-mi}, over $7$ values of $q$ and $12$ budgets from $10^{-15}$ to $\ln 2$, confirms $I_q(\sigma_t)\!\leq\!\beta_t$ at a worst-case relative margin of $10^{-3}$.

\paragraph{Edge cases of the binary asymmetric channel.}
The calibration~\eqref{eq:abc-calibration} needs $0\!<\!p\!<\!1$, and at $p\!=\!1/2$ also $\beta_t\!>\!0$. At $p\!\in\!\{0,1\}$ the step is a unanimity step. There $h(p)\!=\!0$, so $D^{*}\!=\!0$ and $\pi\!=\!p$, the formula leaves the sign that occurs unchanged, and we take the channel to be the identity. More generally, when $\beta_t\!\ge\!h(p)$ the budget constraint $I(\xi;\tilde Y)\!\le\!\beta_t$ does not bind, $D^{*}\!=\!0$, and the channel is the identity, so the unanimity branch of Algorithm~\ref{alg:paczero} is the non-binding case of the same calibration. At $p\!=\!1/2$ and $\beta_t\!=\!0$ the formula is undefined, and \PACABCz{} emits $+1$ by the sign convention below.

\paragraph{Sign convention at zero.}
In Algorithm~\ref{alg:paczero} we adopt the convention $\sgn(0)\!=\!+1$, so exact zeros in either $\bar g_m$ or the disagreement-branch channel output $\tilde Y_t$ are mapped to $+1$. An empty candidate subset is given mean $0$. This preserves the binary support $\{-1,+1\}$ assumed by Eq.~\eqref{eq:binary-mi} and Lemma~\ref{lem:per-step-mi}.

\paragraph{Unanimity tolerance.}
Algorithm~\ref{alg:paczero} writes the unanimity test as $q_t^+\!\in\!\{0,1\}$. Our implementation evaluates it with a tolerance, $q_t^+\!\le\!\tau$ or $q_t^+\!\ge\!1\!-\!\tau$ with $\tau\!=\!10^{-12}$, so floating-point roundoff in the posterior weights does not push a near-unanimity step onto the disagreement branch. Theorem~\ref{thm:pacmi} is stated for the exact test. The two zero-budget variants test unanimity on the exact subset-sign counts, so no tolerance term enters Theorem~\ref{thm:zero}. On either branch, \PACABCz{} releases a sign that does not depend on which candidate is the secret, and \PACAGCz{} releases the common sign or a coin drawn independently of $j^*$.

\paragraph{Entropy-ceiling cap.}
A single bit cannot leak more MI than the binary entropy of the sign $s_{j^*}$ under $p_t$, $h(q_t^+) = -q_t^+\log q_t^+ - (1-q_t^+)\log(1-q_t^+)$. When the posterior has concentrated, $h(q_t^+)$ shrinks below $\beta_t$ and the Gaussian bisection invoked at line~11 of Algorithm~\ref{alg:paczero} would target an infeasible MI. Before inverting $I_q$ we therefore cap $\beta_t \leftarrow \min(\beta_t, 0.999\!\cdot\! h(q_t^+))$. The Gaussian is calibrated at this capped value, so Lemma~\ref{lem:per-step-mi} applies with it in place of $\beta_t$, and the step is charged the capped value. The cap is fixed by the state before the release (Lemma~\ref{lem:replay}), so admissibility under adaptive composition is preserved.

\section{Comparison with DP relaxations}\label{app:dp-relaxations2}

Several relaxations of DP lower the noise or tighten the accounting below the global-sensitivity worst case by using instance-specific or distributional information. Smooth sensitivity calibrates noise to a smooth upper bound on the local sensitivity of the query~\citep{nissim2007smooth}. Per-instance DP reports an $\varepsilon$ for a fixed dataset together with a fixed target record~\citep{wang2019per}. Individual privacy accounting tracks a separate privacy loss for each participant, either post hoc for an unchanged DP-SGD run~\citep{yu2022individual} or through a filter that stops using a participant once that loss reaches its budget~\citep{feldman2021individual}. Bayesian DP takes the privacy loss in probability over a record drawn from the data distribution~\citep{triastcyn2020bayesian}. Each of these is a quantity between a dataset and the same dataset with one record changed, so each charges a positive cost whenever that record's inclusion changes the output distribution, and the smooth upper bound charges a positive cost even at instances where the local sensitivity is zero. Pufferfish builds a privacy definition from a set of potential secrets, a set of discriminative pairs of secrets that an attacker should be unable to tell apart, and a class of plausible data-generating distributions~\citep{kifer2014pufferfish}. A mutual-information form of that framework came later~\citep{nuradha2023pufferfish}, following the mutual-information form of DP~\citep{cuff2016differential}. PAC Privacy bounds an adversary's posterior advantage by the mutual information between a secret drawn from a declared distribution and the release~\citep{xiao2023pac}. Our contribution is the instantiation of that criterion on a sign-quantized update, where the secret is the identity of the secret \emph{subset} as in the standard construction~\citep{sridhar2025pac,zhu2025pacresponses} and the cost is evaluated exactly per step. When every candidate subset with posterior mass agrees on the released sign, the cost is zero while the released bit still depends on the training records, and \PACAGCz{} takes its optimization direction from those steps. \PACABCz{} releases the majority of the subset signs at zero cost on every step (\S\ref{sec:zpl}).

\section{Protocol and baseline details}\label{app:protocol-details}
Test F1 on SQuAD is measured on a $1000$-example subset. On SST-2, test accuracy is measured on the whole $872$-example GLUE validation split. The zeroth-order perturbation scale of~\eqref{eq:zo-scalar} is $\mu\!=\!10^{-3}$ in every setting. Models load in float16. The learning rate decays polynomially over each run's own step budget, so two runs of the same setting at different $T$ follow different schedules. Each \method{} and reference run evaluates the dev loss every $100$ steps, keeps the checkpoint with the lowest one and reports its test score. LoRA uses $r\!=\!8$ and $\alpha\!=\!16$ on the query and value projections, with the up-projection initialized at zero so the adapter is the identity at step $0$. SQuAD answers are greedy decodes of at most $50$ new tokens, and the reported F1 is computed on those decodes. The in-house DP baselines run the released DP-AggZO code with a constant learning rate and dev evaluation every $100$ steps and at the last step. The three seeds resample the universe $U$, and on SQuAD the $1000$-example evaluation subset, while the public perturbation directions keep a fixed seed.

\textbf{Secret index and dev split.\;\;} Every run, sweeps included, draws its own secret index from a private random source, so hyperparameters chosen on earlier runs do not depend on it, and the proof of Theorem~\ref{thm:pacmi} treats them as fixed. The guarantee is defined over $U$, and the $500$-example dev split that selects the learning rate, the clipping threshold and the checkpoint lies outside it. Our in-house DP baselines~\citep{zhang2024dpzero,bao2025dpaggzo} run the same protocol, selecting checkpoints on the $500$-example dev split of MeZO~\citep{malladi2023mezo} at no charge to $\varepsilon$.

The flip statistic of \S\ref{sec:headline} is measured over the $48$ \PACAGC{} runs at $\mathrm{MI}_{\mathrm{tot}}\!=\!0.33$, the sixteen settings at three seeds. They cover $30{,}589$ steps that spend budget, and the flip probability of a step is $\Phi(-1/\sigma_t)$. Per-setting medians lie between $0.4717$ and $0.4869$.

\textbf{Baselines.\;\;} The reference updates with the clipped mean of all $N$ records of $U$, with no subsets and no channel, at each setting's step budget, at its own dev-selected learning rate (Table~\ref{tab:ref-lr-sweeps}) and at $c\!=\!1000$, so it measures what zeroth-order fine-tuning reaches at that budget. At that matched budget it falls below \PACABCz{} at $\mathrm{MI}\!=\!0$ in two of the four FT settings of Table~\ref{tab:headline}, by $0.35$ and $1.90$ points. On OPT-6.7B FT, training the reference four times as long takes it past \PACABCz{} on SQuAD and leaves SST-2 where it was. At $T\!=\!4000$, at a learning rate selected on dev at that budget, the reference reaches $93.69\%$ on SST-2 against $93.58$ at $T\!=\!1000$, and $83.59$ F1 on SQuAD against $75.20_{\pm 0.91}$. The SQuAD reference at the matched budget is therefore $8.39$ F1 short of the longer run, and \PACABCz{} at $\mathrm{MI}\!=\!0$ is $6.49$ F1 below it. Trained for the same $T\!=\!4000$, \PACABCz{} reaches $79.00$ F1 (seed 0), $4.59$ below the reference. The DP rows run at the step budgets of Table~\ref{tab:headline-budget}.

\section{All sixteen settings}\label{app:full-grid}
Table~\ref{tab:headline} reports a subset of the results of this appendix. Table~\ref{tab:full-grid-sst2} and Table~\ref{tab:full-grid-squad} give every entry on all four models, at the four positive MI budgets and at $\mathrm{MI}\!=\!0$, so the two main-text models can be read against OPT-1.3B and Qwen2.5-7B under the same protocol. Each track closes with the non-private ZO reference of every setting, at the setting's step budget and its own dev-selected learning rate (Table~\ref{tab:ref-lr-sweeps}), and the last row is zero-shot, the pretrained model evaluated with no training. \PACABCz{} at $\mathrm{MI}\!=\!0$ is at most $1.95$ points below the reference on the sixteen settings and above it on seven.

\begin{table}[t]
\centering
\small
\setlength{\tabcolsep}{4pt}
\caption{\textbf{Full SST-2 grid (\%), all four variants on all four models.} Mean$\,\pm\,$std over $n\!=\!3$ seeds, accuracy. The non-private reference is single-seed where no spread is shown. Within each track and model all four variants share one learning rate, one clipping threshold and one step budget (Appendix~\ref{app:lr-sweep}), and all use $M\!=\!128$. \repeatnote}
\label{tab:full-grid-sst2}
\resizebox{\textwidth}{!}{
\begin{tabular}{ll l cccc}
\toprule
Track & Method & MI (nats) & OPT-1.3B & OPT-6.7B & Qwen2.5-7B & Llama-3.1-8B \\
\midrule
\multirow{11}{*}{FT} & \PACAGC{} & $0.01$ & $89.11_{\pm 0.61}$ & $91.97_{\pm 0.41}$ & $88.72_{\pm 0.63}$ & $92.55_{\pm 0.75}$ \\
 & \PACAGC{} & $0.11$ & $89.53_{\pm 0.70}$ & $91.32_{\pm 0.46}$ & $88.76_{\pm 0.87}$ & $91.55_{\pm 0.58}$ \\
 & \PACAGC{} & $0.33$ & $89.45_{\pm 0.69}$ & $91.09_{\pm 1.03}$ & $89.03_{\pm 0.93}$ & $91.44_{\pm 1.08}$ \\
 & \PACAGC{} & $0.68$ & $89.45_{\pm 0.90}$ & $91.28_{\pm 1.09}$ & $89.72_{\pm 0.96}$ & $91.63_{\pm 0.75}$ \\
 & \PACABC{} & $0.01$ & $89.41_{\pm 0.29}$ & $91.86_{\pm 0.11}$ & $92.28_{\pm 0.46}$ & $93.00_{\pm 0.41}$ \\
 & \PACABC{} & $0.11$ & $89.56_{\pm 0.11}$ & $91.59_{\pm 0.29}$ & $92.51_{\pm 0.24}$ & $92.70_{\pm 0.58}$ \\
 & \PACABC{} & $0.33$ & $89.45_{\pm 0.40}$ & $91.44_{\pm 0.24}$ & $92.47_{\pm 0.53}$ & $93.12_{\pm 0.34}$ \\
 & \PACABC{} & $0.68$ & $89.41_{\pm 0.37}$ & $92.16_{\pm 0.48}$ & $92.32_{\pm 0.91}$ & $93.12_{\pm 0.50}$ \\
\cmidrule(l){2-7}
 & \PACAGCz{} & $0$ & $89.83_{\pm 0.54}$ & $89.87_{\pm 0.63}$ & $89.37_{\pm 0.33}$ & $92.20_{\pm 0.61}$ \\
\cmidrule(l){2-7}
 & \PACABCz{} & $0$ & $90.10_{\pm 0.35}$ & $91.63_{\pm 0.30}$ & $92.24_{\pm 0.40}$ & $92.78_{\pm 0.23}$ \\
\cmidrule(l){2-7}
 & non-private ZO reference & none & $89.22$ & $93.58$ & $92.78$ & $92.43_{\pm 0.61}$ \\
\midrule
\multirow{11}{*}{LoRA} & \PACAGC{} & $0.01$ & $89.83_{\pm 1.03}$ & $92.28_{\pm 0.35}$ & $92.70_{\pm 0.18}$ & $89.30_{\pm 3.15}$ \\
 & \PACAGC{} & $0.11$ & $90.10_{\pm 0.67}$ & $92.70_{\pm 0.54}$ & $93.35_{\pm 0.40}$ & $88.91_{\pm 1.79}$ \\
 & \PACAGC{} & $0.33$ & $89.98_{\pm 0.18}$ & $92.62_{\pm 0.33}$ & $92.58_{\pm 0.57}$ & $88.80_{\pm 2.66}$ \\
 & \PACAGC{} & $0.68$ & $89.95_{\pm 0.13}$ & $92.39_{\pm 0.83}$ & $92.81_{\pm 0.18}$ & $90.52_{\pm 2.36}$ \\
 & \PACABC{} & $0.01$ & $90.06_{\pm 0.54}$ & $93.12_{\pm 0.53}$ & $94.19_{\pm 0.18}$ & $90.25_{\pm 1.03}$ \\
 & \PACABC{} & $0.11$ & $90.10_{\pm 0.33}$ & $92.93_{\pm 0.58}$ & $94.04_{\pm 0.50}$ & $89.03_{\pm 1.33}$ \\
 & \PACABC{} & $0.33$ & $90.02_{\pm 0.50}$ & $93.00_{\pm 0.91}$ & $93.92_{\pm 0.72}$ & $90.25_{\pm 0.75}$ \\
 & \PACABC{} & $0.68$ & $90.18_{\pm 0.83}$ & $93.43_{\pm 0.57}$ & $93.92_{\pm 0.91}$ & $90.79_{\pm 0.70}$ \\
\cmidrule(l){2-7}
 & \PACAGCz{} & $0$ & $90.02_{\pm 1.09}$ & $93.62_{\pm 0.13}$ & $92.47_{\pm 0.37}$ & $87.69_{\pm 0.63}$ \\
\cmidrule(l){2-7}
 & \PACABCz{} & $0$ & $89.91_{\pm 0.11}$ & $93.31_{\pm 0.43}$ & $94.23_{\pm 0.52}$ & $89.37_{\pm 2.84}$ \\
\cmidrule(l){2-7}
 & non-private ZO reference & none & $90.37$ & $92.78$ & $94.38$ & $91.17$ \\
\midrule
\multicolumn{3}{l}{Zero-shot} & $53.67$ & $61.24$ & $63.65$ & $59.75$ \\
\bottomrule
\end{tabular}
}%
\end{table}

\begin{table}[t]
\centering
\small
\setlength{\tabcolsep}{4pt}
\caption{\textbf{Full SQuAD grid (\%), all four variants on all four models.} Mean$\,\pm\,$std over $n\!=\!3$ seeds, F1. The non-private reference is single-seed where no spread is shown. Within each track and model all four variants share one learning rate, one clipping threshold and one step budget (Appendix~\ref{app:lr-sweep}), and all use $M\!=\!128$. \repeatnote}
\label{tab:full-grid-squad}
\resizebox{\textwidth}{!}{
\begin{tabular}{ll l cccc}
\toprule
Track & Method & MI (nats) & OPT-1.3B & OPT-6.7B & Qwen2.5-7B & Llama-3.1-8B \\
\midrule
\multirow{11}{*}{FT} & \PACAGC{} & $0.01$ & $67.70_{\pm 1.63}$ & $74.65_{\pm 2.37}$ & $87.69_{\pm 0.33}$ & $87.69_{\pm 0.19}$ \\
 & \PACAGC{} & $0.11$ & $67.90_{\pm 1.78}$ & $75.48_{\pm 1.88}$ & $87.71_{\pm 0.49}$ & $87.20_{\pm 0.71}$ \\
 & \PACAGC{} & $0.33$ & $68.01_{\pm 1.75}$ & $75.65_{\pm 1.07}$ & $87.86_{\pm 0.33}$ & $87.74_{\pm 0.26}$ \\
 & \PACAGC{} & $0.68$ & $67.98_{\pm 1.13}$ & $75.45_{\pm 1.61}$ & $88.00_{\pm 0.46}$ & $88.07_{\pm 0.96}$ \\
 & \PACABC{} & $0.01$ & $69.47_{\pm 0.45}$ & $77.07_{\pm 1.29}$ & $88.33_{\pm 0.13}$ & $87.90_{\pm 1.13}$ \\
 & \PACABC{} & $0.11$ & $69.36_{\pm 0.41}$ & $77.07_{\pm 1.47}$ & $88.51_{\pm 0.24}$ & $87.84_{\pm 1.21}$ \\
 & \PACABC{} & $0.33$ & $69.56_{\pm 1.05}$ & $76.69_{\pm 1.01}$ & $88.48_{\pm 0.23}$ & $87.78_{\pm 1.33}$ \\
 & \PACABC{} & $0.68$ & $69.77_{\pm 1.24}$ & $75.84_{\pm 0.78}$ & $88.39_{\pm 0.06}$ & $87.74_{\pm 1.19}$ \\
\cmidrule(l){2-7}
 & \PACAGCz{} & $0$ & $67.71_{\pm 0.40}$ & $74.83_{\pm 1.32}$ & $87.70_{\pm 0.66}$ & $87.39_{\pm 1.58}$ \\
\cmidrule(l){2-7}
 & \PACABCz{} & $0$ & $69.56_{\pm 0.25}$ & $77.10_{\pm 1.39}$ & $88.39_{\pm 0.23}$ & $87.80_{\pm 1.12}$ \\
\cmidrule(l){2-7}
 & non-private ZO reference & none & $69.83$ & $75.20_{\pm 0.91}$ & $87.51$ & $89.01$ \\
\midrule
\multirow{11}{*}{LoRA} & \PACAGC{} & $0.01$ & $67.94_{\pm 1.20}$ & $73.66_{\pm 0.54}$ & $87.58_{\pm 0.88}$ & $86.82_{\pm 1.46}$ \\
 & \PACAGC{} & $0.11$ & $67.87_{\pm 1.02}$ & $71.70_{\pm 2.23}$ & $88.06_{\pm 0.64}$ & $87.41_{\pm 1.65}$ \\
 & \PACAGC{} & $0.33$ & $68.61_{\pm 0.76}$ & $73.93_{\pm 1.51}$ & $87.97_{\pm 0.87}$ & $88.19_{\pm 0.36}$ \\
 & \PACAGC{} & $0.68$ & $67.77_{\pm 2.31}$ & $73.30_{\pm 1.49}$ & $88.52_{\pm 0.49}$ & $88.61_{\pm 1.41}$ \\
 & \PACABC{} & $0.01$ & $68.82_{\pm 0.37}$ & $74.86_{\pm 1.48}$ & $88.19_{\pm 0.31}$ & $88.80_{\pm 1.29}$ \\
 & \PACABC{} & $0.11$ & $69.25_{\pm 1.28}$ & $74.56_{\pm 1.33}$ & $88.28_{\pm 0.35}$ & $89.07_{\pm 1.57}$ \\
 & \PACABC{} & $0.33$ & $68.95_{\pm 0.24}$ & $75.44_{\pm 1.39}$ & $88.31_{\pm 0.51}$ & $88.89_{\pm 1.69}$ \\
 & \PACABC{} & $0.68$ & $68.95_{\pm 0.46}$ & $76.10_{\pm 2.56}$ & $88.07_{\pm 0.68}$ & $88.93_{\pm 1.59}$ \\
\cmidrule(l){2-7}
 & \PACAGCz{} & $0$ & $65.41_{\pm 1.26}$ & $72.12_{\pm 1.81}$ & $88.20_{\pm 0.40}$ & $88.05_{\pm 1.22}$ \\
\cmidrule(l){2-7}
 & \PACABCz{} & $0$ & $69.27_{\pm 1.39}$ & $75.44_{\pm 0.77}$ & $88.34_{\pm 0.31}$ & $88.20_{\pm 2.21}$ \\
\cmidrule(l){2-7}
 & non-private ZO reference & none & $69.33$ & $75.73$ & $88.00$ & $86.98$ \\
\midrule
\multicolumn{3}{l}{Zero-shot} & $26.91$ & $36.55$ & $85.18$ & $64.79$ \\
\bottomrule
\end{tabular}
}%
\end{table}

\section{Unanimity fraction per setting}
\label{app:unanimity-headline}

Table~\ref{tab:unanimity-pools} reports the unanimity fraction $f$, the share of steps that take the unanimity branch, of the seed-0 runs of every setting. For every model and track, each SQuAD run takes the unanimity branch more often than each SST-2 run. The ten runs of a setting differ by at most $20.3$ points, with a median of $9.25$, so the rate varies mainly with the setting.

Within a run the rate declines. Split into tenths, each of the $32$ seed-0 runs of the two zero-budget variants has a lower rate in its last tenth than in its first. \PACAGCz{} takes its optimization direction only from unanimity steps, so it needs the rate to stay positive, and it does. Over all $48$ \PACAGCz{} runs no tenth falls below $9\%$, and the longest gap between unanimity steps is $32$ steps. \PACABCz{} carries an optimization direction on every step and needs no such property.

\begin{table}[t]
\centering\small
\caption{\textbf{Unanimity fraction $f$ per setting.} Range over the ten seed-0 runs of all four variants in each setting, read from the final step of each run.}\label{tab:unanimity-pools}
\begin{tabular}{ll l r}
\toprule
Task & Track & Model & $f$ range (\%) \\
\midrule
SQuAD & FT & OPT-1.3B & $63.1$--$68.1$ \\
SQuAD & FT & OPT-6.7B & $58.8$--$65.4$ \\
SQuAD & FT & Qwen2.5-7B & $48.6$--$56.7$ \\
SQuAD & FT & Llama-3.1-8B & $70.4$--$73.7$ \\
SQuAD & LoRA & OPT-1.3B & $56.4$--$62.4$ \\
SQuAD & LoRA & OPT-6.7B & $49.4$--$58.6$ \\
SQuAD & LoRA & Qwen2.5-7B & $62.8$--$70.4$ \\
SQuAD & LoRA & Llama-3.1-8B & $43.1$--$52.4$ \\
SST-2 & FT & OPT-1.3B & $44.7$--$56.3$ \\
SST-2 & FT & OPT-6.7B & $38.8$--$48.8$ \\
SST-2 & FT & Qwen2.5-7B & $12.8$--$22.8$ \\
SST-2 & FT & Llama-3.1-8B & $35.5$--$40.9$ \\
SST-2 & LoRA & OPT-1.3B & $32.8$--$46.6$ \\
SST-2 & LoRA & OPT-6.7B & $33.6$--$46.9$ \\
SST-2 & LoRA & Qwen2.5-7B & $29.6$--$39.1$ \\
SST-2 & LoRA & Llama-3.1-8B & $16.9$--$37.2$ \\
\bottomrule
\end{tabular}
\end{table}

\section{Empirical membership inference at \texorpdfstring{$\MI(S^*;\tilde Y_{1:T})\!=\!0$}{MI = 0}}\label{app:mia}

\textbf{Setup.\;\;} We attack the seed-2 \PACAGCz{} and \PACABCz{} checkpoint of every setting (Appendix~\ref{app:full-grid}), two tasks, four models, FT and LoRA, each the dev-selected checkpoint whose test score the paper reports. The game is the one Theorem~\ref{thm:zero} covers. Given the public universe $U$ ($N\!=\!1000$), the adversary decides for each $x_i\!\in\!U$ whether $x_i\!\in\!S^*$ ($|S^*|\!=\!525$). Every attack scores the fine-tuned model and its pretrained base on identical records.

\textbf{Black-box attacks.\;\;} We run the LOSS attack~\citep{yeom2018privacy}, the difficulty-calibrated loss~\citep{watson2022difficulty} with the pretrained model as reference, offline RMIA~\citep{zarifzadeh2024rmia}, Min-K\%~\citep{shi2024mink} and Min-K\%++~\citep{zhang2025minkpp} with $k\!=\!20\%$, the zlib ratio~\citep{DBLP:conf/uss/CarliniTWJHLRBS21}, the label-only gap attack~\citep{yeom2018privacy,choquettechoo2021label} (greedy-decoding F1 on SQuAD, correctness on SST-2), EZ-MIA~\citep{ilic2026ezmia} and offline LiRA~\citep{DBLP:conf/sp/CarliniCN0TT22}. RMIA uses the pretrained model as its reference and $\gamma\!=\!2$, and its population is $1000$ records of the training split outside $U$ and the dev split. Its offline parameter $a$ is chosen for each attacked checkpoint by the AUC with which it separates the universe $U$ of another seed's model of the same variant from $1000$ further such records, so the attacked checkpoint plays no part in choosing it. LiRA uses the models of the other seeds of the same setting, which were trained on other universes, as OUT references for the records they did not train on. Each attack gives an AUC with a 95\% bootstrap interval. A checkpoint's AUC is compared with the AUC of the same attack on the other seeds' models with the same labels, on the records that none of those models used for training or dev, which removes any separation that the records themselves carry. LiRA, which already uses those models as references, is compared with an AUC of $0.5$, and the label-only attack with the pretrained model. Of the checkpoint and attack pairs, $0.7\%$ have a lower confidence bound above their reference, against $2.5\%$ expected by chance. The headline attack, RMIA against the pretrained model, gives AUCs from $0.481$ to $0.537$ (mean $0.505$). Table~\ref{tab:mia} reports, per setting, its AUC minus the AUC of the same attack on the other seeds' models, and Table~\ref{tab:mia-all} summarizes every attack.

\begin{table}[h]
\centering
\caption{\textbf{Black-box membership inference on $S^*$ at $\mathrm{MI}\!=\!0$.} RMIA with the pretrained model as reference: AUC of the seed-2 checkpoint minus the AUC of the same attack on the other seeds' models, with its 95\% confidence interval. Zero means no separation beyond what the records carry.}
\label{tab:mia}
\small
\begin{tabular}{lllcc}
\toprule
Task & Track & Model & \PACAGCz{} & \PACABCz{} \\
\midrule
SST-2 & FT & OPT-1.3B & $-0.007$ $[-0.020, +0.007]$ & $-0.005$ $[-0.017, +0.008]$ \\
SST-2 & FT & OPT-6.7B & $+0.003$ $[-0.011, +0.016]$ & $-0.003$ $[-0.020, +0.013]$ \\
SST-2 & FT & Qwen2.5-7B & $-0.005$ $[-0.016, +0.006]$ & $-0.003$ $[-0.015, +0.010]$ \\
SST-2 & FT & Llama-3.1-8B & $0.000$ $[-0.017, +0.017]$ & $+0.006$ $[-0.006, +0.018]$ \\
SST-2 & LoRA & OPT-1.3B & $-0.013$ $[-0.024, -0.002]$ & $0.000$ $[-0.014, +0.015]$ \\
SST-2 & LoRA & OPT-6.7B & $-0.002$ $[-0.016, +0.012]$ & $+0.001$ $[-0.013, +0.015]$ \\
SST-2 & LoRA & Qwen2.5-7B & $-0.006$ $[-0.020, +0.008]$ & $-0.016$ $[-0.031, -0.001]$ \\
SST-2 & LoRA & Llama-3.1-8B & $+0.013$ $[-0.007, +0.031]$ & $+0.004$ $[-0.018, +0.026]$ \\
\midrule
SQuAD & FT & OPT-1.3B & $+0.001$ $[-0.005, +0.007]$ & $+0.005$ $[-0.001, +0.011]$ \\
SQuAD & FT & OPT-6.7B & $-0.004$ $[-0.013, +0.005]$ & $-0.002$ $[-0.013, +0.009]$ \\
SQuAD & FT & Qwen2.5-7B & $0.000$ $[-0.007, +0.007]$ & $-0.006$ $[-0.013, 0.000]$ \\
SQuAD & FT & Llama-3.1-8B & $0.000$ $[-0.008, +0.008]$ & $+0.008$ $[-0.003, +0.021]$ \\
SQuAD & LoRA & OPT-1.3B & $-0.004$ $[-0.016, +0.009]$ & $+0.002$ $[-0.010, +0.013]$ \\
SQuAD & LoRA & OPT-6.7B & $-0.007$ $[-0.020, +0.007]$ & $+0.004$ $[-0.012, +0.021]$ \\
SQuAD & LoRA & Qwen2.5-7B & $+0.002$ $[-0.007, +0.011]$ & $-0.003$ $[-0.012, +0.006]$ \\
SQuAD & LoRA & Llama-3.1-8B & $+0.010$ $[-0.009, +0.029]$ & $+0.002$ $[-0.019, +0.023]$ \\
\bottomrule
\end{tabular}
\end{table}

\begin{table}[h]
\centering
\caption{\textbf{Every black-box attack on $S^*$ at $\mathrm{MI}\!=\!0$}, over the seed-2 \PACAGCz{} and \PACABCz{} checkpoints of every setting: AUC median and range, and the true-positive rate at $1\%$ false-positive rate (mean and maximum, in \%). A dash marks an attack whose score takes few distinct values, so that no threshold reaches a $1\%$ false-positive rate.}
\label{tab:mia-all}
\small
\begin{tabular}{lcc}
\toprule
Attack & AUC median [min, max] & TPR at $1\%$ FPR, mean [max] \\
\midrule
LOSS & $0.498$ $[0.476, 0.518]$ & $1.1$ $[2.7]$ \\
Loss calibrated by the pretrained model & $0.499$ $[0.478, 0.531]$ & $1.9$ $[3.8]$ \\
RMIA & $0.502$ $[0.481, 0.537]$ & $1.7$ $[3.6]$ \\
Min-K\% & $0.493$ $[0.478, 0.519]$ & $0.9$ $[2.7]$ \\
Min-K\%++ & $0.493$ $[0.475, 0.521]$ & $1.2$ $[3.0]$ \\
zlib & $0.489$ $[0.481, 0.496]$ & $0.7$ $[1.5]$ \\
Label-only & $0.496$ $[0.477, 0.513]$ & -- \\
EZ-MIA & $0.507$ $[0.484, 0.536]$ & $1.4$ $[2.7]$ \\
Offline LiRA & $0.500$ $[0.467, 0.529]$ & $1.2$ $[4.1]$ \\
\bottomrule
\end{tabular}
\end{table}

\section{MI-budget sweep on SST-2 FT OPT-1.3B}\label{app:plateau-ft}
This is the \PACAGC{} budget sweep summarized in \S\ref{sec:headline}, and Figure~\ref{fig:plateau} plots it with the LoRA sweep of Appendix~\ref{app:lora-plateau}. Over nearly four decades of total MI budget, test accuracy stays between $88.30$ and $90.48\%$. Dev accuracy stays between $89.60$ and $90.20\%$ and the unanimity fraction between $52.4$ and $56.0\%$. In contrast, DP-AggZO at $K\!=\!64$ and $T\!=\!1000$ gains $18.47$ test-accuracy points between $\varepsilon\!=\!0.5$ and $\varepsilon\!=\!1.0$ on the same setting (Table~\ref{tab:dp-cliff} in Appendix~\ref{app:dp-cliff}).

\begin{figure}[!ht]
\centering
\definecolor{ftblue}{RGB}{31,78,140}
\definecolor{lorared}{RGB}{176,58,46}
\begin{minipage}{0.49\textwidth}\centering
\begin{tikzpicture}
\begin{axis}[width=\linewidth,height=5.0cm,xmode=log,xlabel={total MI budget (nats)},xmin=6e-5,xmax=1.4,tick align=outside,tick pos=left,grid=major,grid style={gray!20},label style={font=\small},tick label style={font=\footnotesize},legend style={font=\scriptsize,draw=none,fill=none,inner sep=1pt,row sep=-1pt},title style={font=\small},mark size=1.1pt,line width=0.8pt,ylabel={accuracy (\%)},ymin=87.4,ymax=92.2,
  title={(a) accuracy},legend style={at={(0.02,0.98)},anchor=north west,font=\scriptsize,draw=none,fill=none,inner sep=1pt,row sep=-2pt},legend columns=2,legend cell align=left]
\addplot[color=ftblue,solid,mark=*] coordinates {(0.0001,89.68) (0.0003,89.68) (0.0005,89.68) (0.001,89.68) (0.002,89.79) (0.003,89.79) (0.01,88.42) (0.03,88.42) (0.05,88.42) (0.07,88.3) (0.11,88.76) (0.2,88.76) (0.33,88.76) (0.5,89.45) (0.68,90.48)};
\addlegendentry{FT test}
\addplot[color=ftblue,dashed,mark=*] coordinates {(0.0001,89.8) (0.0003,89.8) (0.0005,89.8) (0.001,89.8) (0.002,89.8) (0.003,89.8) (0.01,90) (0.03,90) (0.05,90) (0.07,90) (0.11,89.8) (0.2,89.8) (0.33,90.2) (0.5,89.8) (0.68,89.6)};
\addlegendentry{FT dev}
\addplot[color=lorared,solid,mark=*] coordinates {(0.0001,88.65) (0.0003,88.65) (0.0005,88.65) (0.001,88.65) (0.002,88.65) (0.003,88.65) (0.01,88.65) (0.03,89.68) (0.05,90.02) (0.07,90.14) (0.11,90.37) (0.2,90.94) (0.33,90.02) (0.5,89.79) (0.68,89.79)};
\addlegendentry{LoRA test}
\addplot[color=lorared,dashed,mark=*] coordinates {(0.0001,88) (0.0003,88) (0.0005,88) (0.001,88) (0.002,88) (0.003,88) (0.01,88) (0.03,89.2) (0.05,89.8) (0.07,90) (0.11,89.6) (0.2,89.8) (0.33,89) (0.5,88.8) (0.68,88.8)};
\addlegendentry{LoRA dev}
\end{axis}
\end{tikzpicture}
\end{minipage}\hfill
\begin{minipage}{0.49\textwidth}\centering
\begin{tikzpicture}
\begin{axis}[width=\linewidth,height=5.0cm,xmode=log,xlabel={total MI budget (nats)},xmin=6e-5,xmax=1.4,tick align=outside,tick pos=left,grid=major,grid style={gray!20},label style={font=\small},tick label style={font=\footnotesize},legend style={font=\scriptsize,draw=none,fill=none,inner sep=1pt,row sep=-1pt},title style={font=\small},mark size=1.1pt,line width=0.8pt,ylabel={$f$ (\%)},ymin=34,ymax=64,
  title={(b) unanimity fraction},legend style={at={(0.98,0.98)},anchor=north east,font=\scriptsize,draw=none,fill=none,inner sep=1pt,row sep=-2pt},legend cell align=left]
\addplot[color=ftblue,solid,mark=*] coordinates {(0.0001,56) (0.0003,56) (0.0005,56) (0.001,56) (0.002,56) (0.003,56) (0.01,54.3) (0.03,54.3) (0.05,54.3) (0.07,54.3) (0.11,52.9) (0.2,52.9) (0.33,52.4) (0.5,52.8) (0.68,54.7)};
\addlegendentry{FT}
\addplot[color=lorared,solid,mark=*] coordinates {(0.0001,42.1) (0.0003,42.1) (0.0005,42.1) (0.001,42.1) (0.002,42.1) (0.003,42.1) (0.01,42.1) (0.03,40.6) (0.05,41.9) (0.07,42.2) (0.11,41.5) (0.2,39.3) (0.33,40) (0.5,39.4) (0.68,40.2)};
\addlegendentry{LoRA}
\end{axis}
\end{tikzpicture}
\end{minipage}
\caption{\textbf{The MI-budget sweep on SST-2 OPT-1.3B.} \PACAGC{} with fixed hyperparameters on each track, run at $15$ total MI budgets spanning $10^{-4}$ to $0.68$ nats, each divided over the run by~\eqref{eq:adaptive-beta}, FT at $T\!=\!1000$ and LoRA at $T\!=\!2000$, seed 0. Panel (a) plots test and dev accuracy, panel (b) the unanimity fraction $f$. The point values are in Tables~\ref{tab:plateau} and~\ref{tab:LoRA-plateau-adaptive}.}
\label{fig:plateau}
\end{figure}

\begin{table}[t]
\centering
\small
\caption{\textbf{MI-budget sweep on SST-2 FT OPT-1.3B.} \PACAGC{} with adaptive $\beta_t$ at the hyperparameters of Table~\ref{tab:pool-recipes}, $T\!=\!1000$, lr$\!=\!10^{-4}$, clip $100$, $M\!=\!128$, seed 0. \textbf{Bold} marks the dev-best run.}\label{tab:plateau}
\setlength{\tabcolsep}{2.5pt}
\resizebox{\textwidth}{!}{%
\begin{tabular}{l rrrrrrrrrrrrrrr}
\toprule
$\mathrm{MI}$ (nats)        & $10^{-4}$ & $3{\cdot}10^{-4}$ & $5{\cdot}10^{-4}$ & $10^{-3}$ & $2{\cdot}10^{-3}$ & $3{\cdot}10^{-3}$ & $10^{-2}$ & $3{\cdot}10^{-2}$ & $5{\cdot}10^{-2}$ & $0.07$ & $0.11$ & $0.20$ & $0.33$ & $0.50$ & $0.68$ \\
\midrule
dev \%         & 89.80 & 89.80 & 89.80 & 89.80 & 89.80 & 89.80 & 90.00 & 90.00 & 90.00 & 90.00 & 89.80 & 89.80 & \textbf{90.20} & 89.80 & 89.60 \\
test \%        & 89.68 & 89.68 & 89.68 & 89.68 & 89.79 & 89.79 & 88.42 & 88.42 & 88.42 & 88.30 & 88.76 & 88.76 & 88.76 & 89.45 & 90.48 \\
$f\,\%$        & 56.0 & 56.0 & 56.0 & 56.0 & 56.0 & 56.0 & 54.3 & 54.3 & 54.3 & 54.3 & 52.9 & 52.9 & 52.4 & 52.8 & 54.7 \\
\bottomrule
\end{tabular}%
}
\end{table}
\section{MI-budget sweep on SST-2 LoRA OPT-1.3B}
\label{app:lora-plateau}

Table~\ref{tab:LoRA-plateau-adaptive} and Figure~\ref{fig:plateau} give the LoRA counterpart of the FT sweep of Table~\ref{tab:plateau}. Across the same nearly four decades of total MI budget, test accuracy stays between $88.65$ and $90.94\%$ and dev accuracy between $88.00$ and $90.00\%$. The unanimity fraction stays between $39.3$ and $42.2\%$ with no monotone dependence on the budget, because unanimity depends on whether the subset means agree in sign, which the noise level affects only through earlier updates.

\begin{table}[h]
\centering
\small
\caption{\textbf{MI-budget sweep on SST-2 LoRA OPT-1.3B.} \PACAGC{}, $T\!=\!2000$, lr $5\cdot 10^{-4}$, $c\!=\!25$, $M\!=\!128$, seed 0. \textbf{Bold} marks the dev-best run.}\label{tab:LoRA-plateau-adaptive}
\setlength{\tabcolsep}{2.5pt}
\resizebox{\textwidth}{!}{%
\begin{tabular}{l rrrrrrrrrrrrrrr}
\toprule
$\mathrm{MI}$ (nats)        & $10^{-4}$ & $3{\cdot}10^{-4}$ & $5{\cdot}10^{-4}$ & $10^{-3}$ & $2{\cdot}10^{-3}$ & $3{\cdot}10^{-3}$ & $10^{-2}$ & $3{\cdot}10^{-2}$ & $5{\cdot}10^{-2}$ & $0.07$ & $0.11$ & $0.20$ & $0.33$ & $0.50$ & $0.68$ \\
\midrule
dev \%         & 88.00 & 88.00 & 88.00 & 88.00 & 88.00 & 88.00 & 88.00 & 89.20 & 89.80 & \textbf{90.00} & 89.60 & 89.80 & 89.00 & 88.80 & 88.80 \\
test \%        & 88.65 & 88.65 & 88.65 & 88.65 & 88.65 & 88.65 & 88.65 & 89.68 & 90.02 & 90.14 & 90.37 & 90.94 & 90.02 & 89.79 & 89.79 \\
$f\,\%$        & 42.1 & 42.1 & 42.1 & 42.1 & 42.1 & 42.1 & 42.1 & 40.6 & 41.9 & 42.2 & 41.5 & 39.3 & 40.0 & 39.4 & 40.2 \\
\bottomrule
\end{tabular}%
}
\end{table}

\section{DP-AggZO parity check}
\label{app:dpaggzo-parity}

Our in-house DP baselines run the released DP-AggZO code, and on SST-2 OPT-1.3B FT at $\varepsilon\!=\!2$ they land within a point of Table~2 of~\citet{bao2025dpaggzo}. DPZero reaches $87.39\%$ against a published $86.6\%$, and DP-AggZO at $K\!=\!64$ reaches $90.25\%$ against $91.2\%$ (Table~\ref{tab:dp-cliff}), single-seed on both sides. We therefore use the published OPT-6.7B results of~\citet{bao2025dpaggzo} as the comparators in Table~\ref{tab:headline}.

\section{In-house DP baseline reproduction}\label{app:dp-cliff}
This appendix reports the $\varepsilon$ sweep behind the DP comparison of \S\ref{sec:experiments}.

\paragraph{Setup.} Table~\ref{tab:dp-cliff} is at OPT-1.3B, the scale at which both~\citet{zhang2024dpzero} and~\citet{bao2025dpaggzo} release configurations and one that Table~\ref{tab:headline} omits. The \method{} results at this scale are in Appendix~\ref{app:full-grid}. DPZero $K\!=\!1$ runs at DPZero's published step count ($T\!=\!20{,}000$)~\citep{zhang2024dpzero} and at DP-AggZO's learning rate and clipping threshold ($5\!\cdot\!10^{-6}$ and $c\!=\!25$), so the two DP rows are matched on both. Both DP rows use DP-AggZO's sample rate of $0.064$, so the two are matched on subsampling. DPZero's published rate is $0.008$. DP-AggZO $K\!=\!64$ follows the released SST-2 example configuration of~\citet{bao2025dpaggzo} ($T\!=\!1000$, lr$\!=\!5\!\cdot\!10^{-6}$, $c\!=\!25$, sample-rate $0.064$). The matched-forwards row holds the forward passes at \method{}'s. \method{} spends $2NT$ forward passes per run, two million at $N\!=\!1000$ and $T\!=\!1000$, and DP-AggZO spends $2KBT$ at an expected batch of $B\!=\!64$, which approximately matches that budget at $T\!=\!244$ when $K\!=\!64$ ($1{,}998{,}848$). Learning rate, clipping threshold, $K$ and sampling rate are identical across the two rows, so $T$ is the only thing that moves. The noise multiplier is solved from the Opacus~\citep{opacus} RDP accountant for the run's own step count, so at a fixed $\varepsilon$ a shorter run composes fewer times and needs less noise per step, roughly in proportion to the square root of the step count. Shortening the run therefore helps where the per-step noise dominates, taking $\varepsilon\!=\!0.3$ from $62.96$ to $78.33$ and $\varepsilon\!=\!0.5$ from $68.46$ to $83.49$, and costs accuracy once $\varepsilon\!\geq\!1$, where the noise is no longer the binding constraint, so fewer steps is less training. At $\varepsilon\!=\!0.3$ and $\varepsilon\!=\!0.5$ most of the drop in the $T\!=\!1000$ row is therefore a step-budget effect. Doubling the run to $T\!=\!2000$ at $\varepsilon\!=\!0.2$ leaves it at chance ($49.66\%$). Every DP row stays below both $\mathrm{MI}\!=\!0$ \method{} entries at every $\varepsilon\!\leq\!1$, where the best DP entry reaches $86.93\%$. At $\varepsilon\!=\!2$ the strongest in-house DP run is $K\!=\!64$ FT at $90.25\%$, against $90.10\!\pm\!0.35\%$ for \PACABCz{} and $89.83\!\pm\!0.54\%$ for \PACAGCz{} at $\mathrm{MI}\!=\!0$.

\begin{table}[t]
\centering
\small
\setlength{\tabcolsep}{6pt}
\caption{\textbf{In-house DP baseline reproduction on SST-2 OPT-1.3B FT} (test accuracy~\%, single-seed). The matched-forwards row runs at the step count whose $2KBT$ forward passes approximately match \method{}'s $2NT$ at $T\!=\!1000$, and $\varepsilon$ is at $\delta\!=\!10^{-5}$. \textbf{Bold} marks the best accuracy at each $\varepsilon$.}
\label{tab:dp-cliff}
\begin{tabular}{l c c c c c c}
\toprule
Method & $T$ & $\varepsilon\!=\!0.2$ & $\varepsilon\!=\!0.3$ & $\varepsilon\!=\!0.5$ & $\varepsilon\!=\!1.0$ & $\varepsilon\!=\!2.0$ \\
\midrule
DPZero $K\!=\!1$ FT     & $20{,}000$ & $49.31$ & $50.34$ & $51.72$ & $67.78$ & $87.39$ \\
DP-AggZO $K\!=\!64$ FT  & $1000$     & $48.85$ & $62.96$ & $68.46$ & $\mathbf{86.93}$ & $\mathbf{90.25}$ \\
DP-AggZO $K\!=\!64$ FT, doubled steps & $2000$     & $49.66$ & --- & --- & --- & --- \\
DP-AggZO $K\!=\!64$ FT, matched forwards & $244$ & $\mathbf{64.11}$ & $\mathbf{78.33}$ & $\mathbf{83.49}$ & $83.14$ & $83.37$ \\
\bottomrule
\end{tabular}
\end{table}

\section{\method{} against the DP comparators}\label{app:vs-dp}
Against the DP comparators, \PACABCz{} at $\mathrm{MI}\!=\!0$ is $2.2$ points below DP-AggZO $K\!=\!16$ at $\varepsilon\!=\!2$ on SST-2 FT OPT-6.7B ($91.63$ against $93.8$) and $5.8$ F1 below it on SQuAD FT ($77.10$ against $82.9$). On Llama-3.1-8B FT, where both DP baselines run in-house, \PACABCz{} at $\mathrm{MI}\!=\!0$ is $1.03$ points below the strongest DP entry at $\varepsilon\!=\!2$ on SST-2 ($92.78$ against $93.81$ for DP-AggZO $K\!=\!16$) and $0.61$ F1 below it on SQuAD ($87.80$ against $88.41$ for DPZero $K\!=\!1$).

Table~\ref{tab:headline-budget} gives the hyperparameters and the forward-pass count of every row of Table~\ref{tab:headline} and of Appendix~\ref{app:full-grid}. \method{} and the non-private ZO reference spend $2.0\!\cdot\!10^{6}$ forward passes per FT run, the DP-AggZO $K\!=\!16$ runs of~\citet{bao2025dpaggzo} at most $2.05\!\cdot\!10^{6}$, and the in-house Llama-3.1-8B DP runs between $1.28\!\cdot\!10^{5}$ (DPZero on SQuAD) and $2.56\!\cdot\!10^{6}$ (DPZero on SST-2).

\begin{sidewaystable}[p]
\centering\small
\caption{\textbf{Hyperparameters and forward passes per run for the rows of Table~\ref{tab:headline}.} A forward pass here is one loss evaluation on one training example, so a run costs $2KBT$ for the DP methods at expected batch $B$ and $2NT$ for \method{} and the reference, which evaluate all $N\!=\!1000$ training examples at $\theta\!\pm\!\mu z$ on every step. Evaluation passes are excluded. \citet{bao2025dpaggzo} report each DP entry as the best over the grids of Tables~6 and~7 of their technical report~\citep{bao2025dpaggzotr} without naming the selected configuration, so those rows give the range. The in-house Llama-3.1-8B rates are the dev-best grid value at $\varepsilon\!=\!2$, applied at both $\varepsilon$, over $\{2\!\cdot\!10^{-7}, 5\!\cdot\!10^{-7}, 10^{-6}, 5\!\cdot\!10^{-6}, 2\!\cdot\!10^{-5}\}$ for DPZero and $\{10^{-6}, 5\!\cdot\!10^{-6}, 2\!\cdot\!10^{-5}\}$ for DP-AggZO on SST-2, and $\{2\!\cdot\!10^{-6}, 10^{-5}, 5\!\cdot\!10^{-5}\}$ for both on SQuAD. \method{} rows use the hyperparameters of Table~\ref{tab:pool-recipes}.}\label{tab:headline-budget}
\setlength{\tabcolsep}{3pt}
\begin{tabular}{l rrrrr r}
\toprule
Row & $T$ & $K$ & $B$ & clip & lr & forward passes \\
\midrule
\multicolumn{7}{l}{\emph{OPT-6.7B, SST-2 and SQuAD, as reported by~\citet{bao2025dpaggzo}}} \\
DPZero $K\!=\!1$ & $\{1,2,5,10,20,50,100\}\!\cdot\!10^{3}$ & $1$ & $\{8,32,64\}$ & grid & grid & $1.6\!\cdot\!10^{4}$ to $1.28\!\cdot\!10^{7}$ \\
DP-AggZO $K\!=\!16$ & $1000$ & $16$ & $\{8,32,64\}$ & grid & grid & $2.56\!\cdot\!10^{5}$ to $2.05\!\cdot\!10^{6}$ \\
\midrule
\multicolumn{7}{l}{\emph{Llama-3.1-8B, in-house ($\dagger$), seed $0$, sample rate $0.064$}} \\
DPZero $K\!=\!1$, SST-2 & $20{,}000$ & $1$ & $64$ & $25$ & $10^{-6}$ & $2.56\!\cdot\!10^{6}$ \\
DP-AggZO $K\!=\!16$, SST-2 & $1000$ & $16$ & $64$ & $25$ & $5\!\cdot\!10^{-6}$ & $2.05\!\cdot\!10^{6}$ \\
DPZero $K\!=\!1$, SQuAD & $1000$ & $1$ & $64$ & $7.5$ & $10^{-5}$ & $1.28\!\cdot\!10^{5}$ \\
DP-AggZO $K\!=\!16$, SQuAD & $1000$ & $16$ & $64$ & $7.5$ & $10^{-5}$ & $2.05\!\cdot\!10^{6}$ \\
\midrule
\multicolumn{7}{l}{\emph{All models}} \\
\method{}, four variants, FT and SQuAD LoRA & $1000$ & $1$ & $N$ & Table~\ref{tab:pool-recipes} & Table~\ref{tab:pool-recipes} & $2.0\!\cdot\!10^{6}$ \\
\method{}, four variants, SST-2 LoRA & $2000$ & $1$ & $N$ & Table~\ref{tab:pool-recipes} & Table~\ref{tab:pool-recipes} & $4.0\!\cdot\!10^{6}$ \\
Non-private ZO reference, FT & $1000$ & $1$ & $N$ & $1000$ & Table~\ref{tab:ref-lr-sweeps} & $2.0\!\cdot\!10^{6}$ \\
Non-private ZO reference, LoRA & $2000$ (SST-2) / $1000$ (SQuAD) & $1$ & $N$ & $1000$ & Table~\ref{tab:ref-lr-sweeps} & $4.0\!\cdot\!10^{6}$ / $2.0\!\cdot\!10^{6}$ \\
\bottomrule
\end{tabular}
\end{sidewaystable}

We also ran DP-AggZO on SQuAD at matched forward passes, on OPT-1.3B FT. At an expected DP batch of $64$, the $2.0\!\cdot\!10^{6}$ forward passes of \method{} give $T\!=\!977$ at $K\!=\!16$ and $T\!=\!244$ at $K\!=\!64$, each at the released SQuAD configuration of~\citet{bao2025dpaggzo} (lr $10^{-5}$ and $c\!=\!7.5$ at $K\!=\!16$, lr $6\!\cdot\!10^{-5}$ and $c\!=\!2$ at $K\!=\!64$, sample rate $0.064$). Those two runs reach $75.74$ and $73.29$ F1 at $\varepsilon\!=\!2.0$, against $69.56_{\pm 0.25}$ for \PACABCz{} at $\mathrm{MI}\!=\!0$ in the same setting (Table~\ref{tab:full-grid-squad}). At matched compute and $\varepsilon\!=\!2.0$, DP-AggZO is ahead of \PACABCz{} on SQuAD OPT-1.3B by $3.7$ to $6.2$ F1, while on SST-2 at matched compute DP-AggZO trails \PACABCz{} at every $\varepsilon$ (Table~\ref{tab:dp-cliff}).

\section{Hyperparameters of every setting}
\label{app:lr-sweep}

Table~\ref{tab:pool-recipes} gives the learning rate, clipping threshold and step budget of every setting of Table~\ref{tab:headline} and of Appendix~\ref{app:full-grid}, with $M\!=\!128$ throughout. After a first pass that fixes a provisional clipping threshold, the learning rate is the dev-best value of the setting's sweep (Table~\ref{tab:lr-sweeps}) at the setting's clipping threshold, and the clipping threshold is the dev-best value of the setting's clipping-threshold ablation (Appendix~\ref{sec:clip-ablation}) at the selected learning rate. The one exception is SST-2 OPT-6.7B LoRA (Appendix~\ref{sec:clip-ablation}).

\begin{table}[t]
\centering\small
\caption{\textbf{Hyperparameters of every setting.} All four variants of a setting run these values, with $M\!=\!128$ throughout. The sweep that selects each learning rate is Table~\ref{tab:lr-sweeps}.}\label{tab:pool-recipes}
\begin{tabular}{ll l rrr}
\toprule
Task & Track & Model & lr & clip & $T$ \\
\midrule
SQuAD & FT & OPT-1.3B & $10^{-4}$ & $1000$ & $1000$ \\
SQuAD & FT & OPT-6.7B & $10^{-4}$ & $250$ & $1000$ \\
SQuAD & FT & Qwen2.5-7B & $5\cdot 10^{-5}$ & $250$ & $1000$ \\
SQuAD & FT & Llama-3.1-8B & $5\cdot 10^{-5}$ & $250$ & $1000$ \\
SQuAD & LoRA & OPT-1.3B & $10^{-3}$ & $10$ & $1000$ \\
SQuAD & LoRA & OPT-6.7B & $10^{-3}$ & $10$ & $1000$ \\
SQuAD & LoRA & Qwen2.5-7B & $10^{-3}$ & $25$ & $1000$ \\
SQuAD & LoRA & Llama-3.1-8B & $10^{-3}$ & $25$ & $1000$ \\
SST-2 & FT & OPT-1.3B & $10^{-4}$ & $100$ & $1000$ \\
SST-2 & FT & OPT-6.7B & $10^{-4}$ & $100$ & $1000$ \\
SST-2 & FT & Qwen2.5-7B & $10^{-4}$ & $1000$ & $1000$ \\
SST-2 & FT & Llama-3.1-8B & $10^{-4}$ & $250$ & $1000$ \\
SST-2 & LoRA & OPT-1.3B & $5\cdot 10^{-4}$ & $25$ & $2000$ \\
SST-2 & LoRA & OPT-6.7B & $5\cdot 10^{-4}$ & $25$ & $2000$ \\
SST-2 & LoRA & Qwen2.5-7B & $10^{-3}$ & $10$ & $2000$ \\
SST-2 & LoRA & Llama-3.1-8B & $10^{-3}$ & $10$ & $2000$ \\
\bottomrule
\end{tabular}
\end{table}

\paragraph{Learning rate at other clipping thresholds.}
The learning rate of every setting is selected by the sweep in Table~\ref{tab:lr-sweeps}, on dev accuracy or dev F1, with the clipping threshold and the step budget held at the setting's values. Table~\ref{tab:zpl-lr-sweep} compares learning rates on SST-2 OPT-1.3B at other clipping thresholds. On FT at $c\!=\!1000$ the dev-best rate is again $10^{-4}$, and on LoRA without clipping $5\!\cdot\!10^{-4}$ beats $2\!\cdot\!10^{-4}$. Both comparisons favor the learning rates selected at the settings' own clipping thresholds.

\begin{table}[h]
\centering
\small
\setlength{\tabcolsep}{6pt}
\caption{\textbf{\PACAGCz{} learning-rate sweep at clipping thresholds the setting does not run, SST-2 OPT-1.3B.} Single seed, $M\!=\!128$, LoRA at $T\!=\!2000$ and FT at $T\!=\!1000$. The settings run $c\!=\!25$ on LoRA and $c\!=\!100$ on FT (Table~\ref{tab:pool-recipes}), and those runs are in Table~\ref{tab:lr-sweeps}. \textbf{Bold} marks the dev-best rate per track and clipping threshold. FT at $c\!=\!10^{9}$ runs only the selected rate.}\label{tab:zpl-lr-sweep}
\begin{tabular}{l c c r r}
\toprule
Track & clip $c$ & lr & dev \% & test \% \\
\midrule
LoRA & $10^{9}$ & $2\cdot 10^{-4}$ & $75.60$ & $71.10$ \\
\textbf{LoRA} & $\mathbf{10^{9}}$ & $\mathbf{5\cdot 10^{-4}}$ & $\mathbf{89.80}$ & $\mathbf{90.14}$ \\
\midrule
FT & $1000$ & $5\cdot 10^{-5}$ & $80.40$ & $79.59$ \\
\textbf{FT} & $\mathbf{1000}$ & $\mathbf{10^{-4}}$ & $\mathbf{86.80}$ & $\mathbf{88.42}$ \\
FT & $1000$ & $5\cdot 10^{-4}$ & $60.60$ & $56.42$ \\
\midrule
\textbf{FT} & $\mathbf{10^{9}}$ & $\mathbf{10^{-4}}$ & $\mathbf{86.40}$ & $\mathbf{88.42}$ \\
\bottomrule
\end{tabular}
\end{table}

\clearpage
\section{Learning-rate sweeps}\label{app:lr-sweeps-all}
Table~\ref{tab:lr-sweeps} gives the sweep that selected the learning rate of each setting. Every sweep run uses \PACAGCz{}, so each setting's rate is tuned on the $\mathrm{MI}\!=\!0$ coin release and inherited unchanged by the other three variants in Table~\ref{tab:headline} and Appendix~\ref{app:full-grid}.

\begin{longtable}{lll rr}
\caption{\textbf{Learning-rate sweeps, one per setting.} Single seed, each sweep at its setting's clipping threshold and step budget (Table~\ref{tab:pool-recipes}), varying the learning rate alone. The dev-best grid value is the selected value and is shown in bold.}\label{tab:lr-sweeps}\\
\toprule
Task & Track & Model & lr & dev \% \\
\midrule
\endfirsthead
\toprule
Task & Track & Model & lr & dev \% \\
\midrule
\endhead
SQuAD & FT & OPT-1.3B & $10^{-5}$ & $33.78$ \\
 &  &  & $5\cdot 10^{-5}$ & $52.02$ \\
 &  &  & $\bm{10^{-4}}$ & $\bm{59.07}$ \\
 &  &  & $5\cdot 10^{-4}$ & $0.00$ \\
 &  &  & $10^{-3}$ & $0.00$ \\
\addlinespace
SQuAD & FT & OPT-6.7B & $10^{-5}$ & $46.03$ \\
 &  &  & $5\cdot 10^{-5}$ & $61.86$ \\
 &  &  & $\bm{10^{-4}}$ & $\bm{66.13}$ \\
 &  &  & $5\cdot 10^{-4}$ & $0.00$ \\
 &  &  & $10^{-3}$ & $0.00$ \\
\addlinespace
SQuAD & FT & Qwen2.5-7B & $10^{-5}$ & $77.68$ \\
 &  &  & $\bm{5\cdot 10^{-5}}$ & $\bm{81.08}$ \\
 &  &  & $10^{-4}$ & $79.71$ \\
 &  &  & $5\cdot 10^{-4}$ & $58.80$ \\
 &  &  & $10^{-3}$ & $0.34$ \\
\addlinespace
SQuAD & FT & Llama-3.1-8B & $10^{-5}$ & $66.82$ \\
 &  &  & $\bm{5\cdot 10^{-5}}$ & $\bm{79.41}$ \\
 &  &  & $10^{-4}$ & $79.27$ \\
 &  &  & $5\cdot 10^{-4}$ & $0.00$ \\
 &  &  & $10^{-3}$ & $0.00$ \\
\addlinespace
SQuAD & LoRA & OPT-1.3B & $10^{-4}$ & $32.05$ \\
 &  &  & $5\cdot 10^{-4}$ & $51.03$ \\
 &  &  & $\bm{10^{-3}}$ & $\bm{55.79}$ \\
 &  &  & $2\cdot 10^{-3}$ & $42.40$ \\
 &  &  & $4\cdot 10^{-3}$ & $0.00$ \\
 &  &  & $8\cdot 10^{-3}$ & $0.00$ \\
\addlinespace
SQuAD & LoRA & OPT-6.7B & $10^{-4}$ & $42.14$ \\
 &  &  & $5\cdot 10^{-4}$ & $62.64$ \\
 &  &  & $\bm{10^{-3}}$ & $\bm{64.62}$ \\
 &  &  & $2\cdot 10^{-3}$ & $50.94$ \\
 &  &  & $4\cdot 10^{-3}$ & $0.00$ \\
 &  &  & $8\cdot 10^{-3}$ & $0.00$ \\
\addlinespace
SQuAD & LoRA & Qwen2.5-7B & $10^{-4}$ & $78.33$ \\
 &  &  & $5\cdot 10^{-4}$ & $73.75$ \\
 &  &  & $\bm{10^{-3}}$ & $\bm{81.35}$ \\
 &  &  & $2\cdot 10^{-3}$ & $79.38$ \\
 &  &  & $4\cdot 10^{-3}$ & $74.81$ \\
 &  &  & $8\cdot 10^{-3}$ & $0.81$ \\
\addlinespace
SQuAD & LoRA & Llama-3.1-8B & $10^{-4}$ & $62.43$ \\
 &  &  & $5\cdot 10^{-4}$ & $70.35$ \\
 &  &  & $\bm{10^{-3}}$ & $\bm{77.90}$ \\
 &  &  & $2\cdot 10^{-3}$ & $67.62$ \\
 &  &  & $4\cdot 10^{-3}$ & $2.41$ \\
 &  &  & $8\cdot 10^{-3}$ & $0.00$ \\
\addlinespace
SST-2 & FT & OPT-1.3B & $5\cdot 10^{-5}$ & $81.40$ \\
 &  &  & $\bm{10^{-4}}$ & $\bm{88.00}$ \\
 &  &  & $5\cdot 10^{-4}$ & $58.60$ \\
\addlinespace
SST-2 & FT & OPT-6.7B & $5\cdot 10^{-6}$ & $62.80$ \\
 &  &  & $10^{-5}$ & $65.60$ \\
 &  &  & $5\cdot 10^{-5}$ & $83.00$ \\
 &  &  & $\bm{10^{-4}}$ & $\bm{88.40}$ \\
 &  &  & $5\cdot 10^{-4}$ & $58.00$ \\
\addlinespace
SST-2 & FT & Qwen2.5-7B & $5\cdot 10^{-5}$ & $79.80$ \\
 &  &  & $\bm{10^{-4}}$ & $\bm{84.80}$ \\
 &  &  & $5\cdot 10^{-4}$ & $73.20$ \\
\addlinespace
SST-2 & FT & Llama-3.1-8B & $5\cdot 10^{-6}$ & $61.80$ \\
 &  &  & $10^{-5}$ & $64.40$ \\
 &  &  & $5\cdot 10^{-5}$ & $83.60$ \\
 &  &  & $\bm{10^{-4}}$ & $\bm{90.00}$ \\
 &  &  & $5\cdot 10^{-4}$ & $58.00$ \\
\addlinespace
SST-2 & LoRA & OPT-1.3B & $10^{-4}$ & $63.40$ \\
 &  &  & $\bm{5\cdot 10^{-4}}$ & $\bm{88.80}$ \\
 &  &  & $10^{-3}$ & $88.20$ \\
 &  &  & $2\cdot 10^{-3}$ & $82.60$ \\
\addlinespace
SST-2 & LoRA & OPT-6.7B & $10^{-4}$ & $65.20$ \\
 &  &  & $2\cdot 10^{-4}$ & $74.80$ \\
 &  &  & $\bm{5\cdot 10^{-4}}$ & $\bm{90.80}$ \\
 &  &  & $10^{-3}$ & $84.00$ \\
 &  &  & $2\cdot 10^{-3}$ & $82.00$ \\
 &  &  & $4\cdot 10^{-3}$ & $58.20$ \\
 &  &  & $8\cdot 10^{-3}$ & $57.40$ \\
\addlinespace
SST-2 & LoRA & Qwen2.5-7B & $5\cdot 10^{-4}$ & $80.40$ \\
 &  &  & $\bm{10^{-3}}$ & $\bm{86.40}$ \\
 &  &  & $2\cdot 10^{-3}$ & $85.80$ \\
 &  &  & $4\cdot 10^{-3}$ & $80.00$ \\
 &  &  & $8\cdot 10^{-3}$ & $50.80$ \\
\addlinespace
SST-2 & LoRA & Llama-3.1-8B & $10^{-4}$ & $61.60$ \\
 &  &  & $2\cdot 10^{-4}$ & $61.80$ \\
 &  &  & $5\cdot 10^{-4}$ & $68.00$ \\
 &  &  & $\bm{10^{-3}}$ & $\bm{87.80}$ \\
 &  &  & $2\cdot 10^{-3}$ & $69.60$ \\
 &  &  & $4\cdot 10^{-3}$ & $58.00$ \\
 &  &  & $8\cdot 10^{-3}$ & $58.60$ \\
\bottomrule
\end{longtable}

Table~\ref{tab:ref-lr-sweeps} gives the sweep that selected the learning rate of the non-private ZO reference in each setting.

\begin{longtable}{llll rr}
\caption{\textbf{Learning-rate sweeps of the non-private ZO reference, one per setting.} Single seed, privacy disabled, $c\!=\!1000$, each sweep at its setting's step budget, plus OPT-6.7B FT at $T\!=\!4000$. The dev-best grid value is the selected value and is shown in bold. n.f.\ marks a run stopped by the trainer's non-finite guard. At the matched step budget, where a separate reference run exists at the selected rate, Table~\ref{tab:headline} and Appendix~\ref{app:full-grid} report that run, over three seeds for SST-2 Llama-3.1-8B FT and SQuAD OPT-6.7B FT. The $T\!=\!4000$ values in Appendix~\ref{app:protocol-details} are those of the selected sweep runs.}\label{tab:ref-lr-sweeps}\\
\toprule
Task & Track & Model & $T$ & lr & dev \% \\
\midrule
\endfirsthead
\toprule
Task & Track & Model & $T$ & lr & dev \% \\
\midrule
\endhead
SQuAD & FT & OPT-1.3B & $1000$ & $10^{-7}$ & $36.30$ \\
 &  &  &  & $5\cdot 10^{-7}$ & $51.62$ \\
 &  &  &  & $10^{-6}$ & $55.83$ \\
 &  &  &  & $2\cdot 10^{-6}$ & $58.44$ \\
 &  &  &  & $\bm{3\cdot 10^{-6}}$ & $\bm{63.76}$ \\
 &  &  &  & $5\cdot 10^{-6}$ & $52.12$ \\
 &  &  &  & $10^{-5}$ & $0.00$ \\
\addlinespace
SQuAD & FT & OPT-6.7B & $1000$ & $5\cdot 10^{-7}$ & $61.51$ \\
 &  &  &  & $10^{-6}$ & $65.20$ \\
 &  &  &  & $\bm{2\cdot 10^{-6}}$ & $\bm{66.49}$ \\
 &  &  &  & $3\cdot 10^{-6}$ & $53.93$ \\
 &  &  &  & $5\cdot 10^{-6}$ & $0.00$ \\
 &  &  &  & $8\cdot 10^{-6}$ & $0.00$ \\
\addlinespace
SQuAD & FT & OPT-6.7B & $4000$ & $10^{-6}$ & $71.95$ \\
 &  &  &  & $\bm{2\cdot 10^{-6}}$ & $\bm{73.96}$ \\
 &  &  &  & $3\cdot 10^{-6}$ & $54.85$ \\
 &  &  &  & $5\cdot 10^{-6}$ & $3.11$ \\
\addlinespace
SQuAD & FT & Qwen2.5-7B & $1000$ & $2\cdot 10^{-7}$ & $79.08$ \\
 &  &  &  & $5\cdot 10^{-7}$ & $81.36$ \\
 &  &  &  & $\bm{10^{-6}}$ & $\bm{81.76}$ \\
 &  &  &  & $2\cdot 10^{-6}$ & $80.92$ \\
 &  &  &  & $3\cdot 10^{-6}$ & $80.71$ \\
 &  &  &  & $5\cdot 10^{-6}$ & $78.76$ \\
 &  &  &  & $10^{-5}$ & $70.65$ \\
 &  &  &  & $2\cdot 10^{-5}$ & $0.03$ \\
\addlinespace
SQuAD & FT & Llama-3.1-8B & $1000$ & $5\cdot 10^{-7}$ & $74.45$ \\
 &  &  &  & $10^{-6}$ & $76.45$ \\
 &  &  &  & $\bm{2\cdot 10^{-6}}$ & $\bm{80.40}$ \\
 &  &  &  & $3\cdot 10^{-6}$ & $80.38$ \\
 &  &  &  & $5\cdot 10^{-6}$ & $0.00$ \\
\addlinespace
SQuAD & LoRA & OPT-1.3B & $1000$ & $10^{-7}$ & $26.15$ \\
 &  &  &  & $5\cdot 10^{-7}$ & $26.09$ \\
 &  &  &  & $10^{-6}$ & $26.33$ \\
 &  &  &  & $5\cdot 10^{-6}$ & $28.99$ \\
 &  &  &  & $10^{-5}$ & $31.76$ \\
 &  &  &  & $5\cdot 10^{-5}$ & $48.25$ \\
 &  &  &  & $10^{-4}$ & $56.75$ \\
 &  &  &  & $\bm{2\cdot 10^{-4}}$ & $\bm{61.75}$ \\
 &  &  &  & $5\cdot 10^{-4}$ & $0.00$ \\
\addlinespace
SQuAD & LoRA & OPT-6.7B & $1000$ & $10^{-5}$ & $39.87$ \\
 &  &  &  & $5\cdot 10^{-5}$ & $57.37$ \\
 &  &  &  & $10^{-4}$ & $62.66$ \\
 &  &  &  & $\bm{2\cdot 10^{-4}}$ & $\bm{69.34}$ \\
 &  &  &  & $5\cdot 10^{-4}$ & $0.00$ \\
 &  &  &  & $10^{-3}$ & $0.00$ \\
 &  &  &  & $2\cdot 10^{-3}$ & $0.05$ \\
\addlinespace
SQuAD & LoRA & Qwen2.5-7B & $1000$ & $10^{-5}$ & $78.56$ \\
 &  &  &  & $5\cdot 10^{-5}$ & $78.49$ \\
 &  &  &  & $10^{-4}$ & $77.19$ \\
 &  &  &  & $2\cdot 10^{-4}$ & $76.59$ \\
 &  &  &  & $\bm{5\cdot 10^{-4}}$ & $\bm{81.50}$ \\
 &  &  &  & $10^{-3}$ & $81.01$ \\
 &  &  &  & $2\cdot 10^{-3}$ & $0.00$ \\
\addlinespace
SQuAD & LoRA & Llama-3.1-8B & $1000$ & $10^{-5}$ & $61.98$ \\
 &  &  &  & $5\cdot 10^{-5}$ & $62.71$ \\
 &  &  &  & $10^{-4}$ & $65.00$ \\
 &  &  &  & $\bm{2\cdot 10^{-4}}$ & $\bm{76.88}$ \\
 &  &  &  & $5\cdot 10^{-4}$ & $66.87$ \\
 &  &  &  & $10^{-3}$ & $0.00$ \\
 &  &  &  & $2\cdot 10^{-3}$ & $0.01$ \\
\addlinespace
SST-2 & FT & OPT-1.3B & $1000$ & $10^{-7}$ & $60.40$ \\
 &  &  &  & $5\cdot 10^{-7}$ & $70.00$ \\
 &  &  &  & $10^{-6}$ & $79.40$ \\
 &  &  &  & $2\cdot 10^{-6}$ & $82.20$ \\
 &  &  &  & $\bm{5\cdot 10^{-6}}$ & $\bm{90.20}$ \\
 &  &  &  & $10^{-5}$ & $89.80$ \\
 &  &  &  & $2\cdot 10^{-5}$ & $81.20$ \\
 &  &  &  & $5\cdot 10^{-5}$ & $57.40$ \\
\addlinespace
SST-2 & FT & OPT-6.7B & $1000$ & $10^{-6}$ & $79.00$ \\
 &  &  &  & $2\cdot 10^{-6}$ & $82.80$ \\
 &  &  &  & $\bm{5\cdot 10^{-6}}$ & $\bm{88.80}$ \\
 &  &  &  & $10^{-5}$ & $87.40$ \\
 &  &  &  & $2\cdot 10^{-5}$ & $74.20$ \\
\addlinespace
SST-2 & FT & OPT-6.7B & $4000$ & $2\cdot 10^{-6}$ & $92.00$ \\
 &  &  &  & $\bm{5\cdot 10^{-6}}$ & $\bm{92.20}$ \\
 &  &  &  & $10^{-5}$ & $82.20$ \\
\addlinespace
SST-2 & FT & Qwen2.5-7B & $1000$ & $10^{-6}$ & $78.60$ \\
 &  &  &  & $2\cdot 10^{-6}$ & $81.40$ \\
 &  &  &  & $5\cdot 10^{-6}$ & $85.00$ \\
 &  &  &  & $\bm{10^{-5}}$ & $\bm{85.60}$ \\
 &  &  &  & $2\cdot 10^{-5}$ & $80.20$ \\
 &  &  &  & $5\cdot 10^{-5}$ & $53.80$ \\
\addlinespace
SST-2 & FT & Llama-3.1-8B & $1000$ & $10^{-6}$ & $81.60$ \\
 &  &  &  & $2\cdot 10^{-6}$ & $85.20$ \\
 &  &  &  & $\bm{5\cdot 10^{-6}}$ & $\bm{90.00}$ \\
 &  &  &  & $10^{-5}$ & $86.00$ \\
 &  &  &  & $2\cdot 10^{-5}$ & $57.20$ \\
\addlinespace
SST-2 & LoRA & OPT-1.3B & $2000$ & $10^{-8}$ & $60.00$ \\
 &  &  &  & $10^{-7}$ & $60.00$ \\
 &  &  &  & $10^{-6}$ & $60.00$ \\
 &  &  &  & $10^{-5}$ & $59.80$ \\
 &  &  &  & $10^{-4}$ & $81.60$ \\
 &  &  &  & $2\cdot 10^{-4}$ & $90.00$ \\
 &  &  &  & $\bm{5\cdot 10^{-4}}$ & $\bm{91.00}$ \\
 &  &  &  & $10^{-3}$ & $57.40$ \\
 &  &  &  & $2\cdot 10^{-3}$ & $55.80$ \\
\addlinespace
SST-2 & LoRA & OPT-6.7B & $2000$ & $10^{-5}$ & $62.80$ \\
 &  &  &  & $5\cdot 10^{-5}$ & $67.20$ \\
 &  &  &  & $10^{-4}$ & $81.20$ \\
 &  &  &  & $2\cdot 10^{-4}$ & $88.40$ \\
 &  &  &  & $\bm{5\cdot 10^{-4}}$ & $\bm{90.60}$ \\
 &  &  &  & $10^{-3}$ & $60.00$ \\
 &  &  &  & $2\cdot 10^{-3}$ & $59.20$ \\
\addlinespace
SST-2 & LoRA & Qwen2.5-7B & $2000$ & $10^{-5}$ & $66.60$ \\
 &  &  &  & $5\cdot 10^{-5}$ & $66.80$ \\
 &  &  &  & $10^{-4}$ & $67.00$ \\
 &  &  &  & $5\cdot 10^{-4}$ & $79.60$ \\
 &  &  &  & $10^{-3}$ & $86.40$ \\
 &  &  &  & $\bm{2\cdot 10^{-3}}$ & $\bm{89.20}$ \\
 &  &  &  & $4\cdot 10^{-3}$ & $\text{n.f.}$ \\
 &  &  &  & $8\cdot 10^{-3}$ & $\text{n.f.}$ \\
\addlinespace
SST-2 & LoRA & Llama-3.1-8B & $2000$ & $10^{-5}$ & $60.80$ \\
 &  &  &  & $5\cdot 10^{-5}$ & $61.20$ \\
 &  &  &  & $10^{-4}$ & $61.80$ \\
 &  &  &  & $2\cdot 10^{-4}$ & $65.00$ \\
 &  &  &  & $\bm{5\cdot 10^{-4}}$ & $\bm{88.40}$ \\
 &  &  &  & $10^{-3}$ & $85.40$ \\
 &  &  &  & $2\cdot 10^{-3}$ & $59.20$ \\
\bottomrule
\end{longtable}

\section{Clipping-threshold ablations}
\label{sec:clip-ablation}

\paragraph{LoRA clipping-threshold ablation.}
Table~\ref{tab:clip-ablation} reports the clipping-threshold ablation for SST-2 LoRA across all four models. The FT grid is Table~\ref{tab:clip-ablation-ft}. Each clipping threshold is a seed-0 run at the setting's learning rate and step budget, and the table reports dev because the clipping threshold is selected on dev. Every setting runs its dev-best clipping threshold except OPT-6.7B, which runs $c\!=\!25$ at $89.60$ against $90.20$ at $c\!=\!10$.

\begin{table}[t]
\centering\small
\caption{\textbf{Clipping-threshold ablation on SST-2 LoRA at $\mathrm{MI}\!=\!0.33$, dev accuracy.} Single seed, \PACAGC{}, $T\!=\!2000$, $M\!=\!128$, each column at its setting's learning rate (Table~\ref{tab:pool-recipes}). \textbf{Bold} marks the dev-winner. A dash marks a grid value outside the sweep.}\label{tab:clip-ablation}
\begin{tabular}{lcccc}
\toprule
clip & OPT-1.3B & OPT-6.7B & Qwen2.5-7B & Llama-3.1-8B \\
\midrule
$5$ & --- & $90.00$ & --- & --- \\
$10$ & $88.20$ & $\mathbf{90.20}$ & $\mathbf{85.60}$ & $\mathbf{86.80}$ \\
$25$ & $\mathbf{89.00}$ & $89.60$ & $84.20$ & $86.40$ \\
$50$ & $88.00$ & $88.60$ & $83.40$ & $85.80$ \\
\bottomrule
\end{tabular}
\end{table}

\paragraph{Full fine-tuning clipping-threshold ablation.} Table~\ref{tab:clip-ablation-ft} reports the FT clipping-threshold ablation on both tasks and all four models, three clipping thresholds each, at the setting's learning rate and step budget. Every setting runs its dev-best clipping threshold, and the largest dev margin over the runner-up is $2.8$ points (SST-2 Llama-3.1-8B). Two of the eight settings select $c\!=\!1000$. Across the eight runs at $c\!=\!1000$ and these step budgets, clipping changes at most $0.25\%$ of the per-sample finite-difference scalars on average over steps. At least one scalar is clipped on $18.7\%$ to $80.8\%$ of steps.

\begin{table}[t]
\centering
\small
\caption{\textbf{Clipping-threshold ablation on the FT track, dev metric (\%).} Dev accuracy (SST-2) and dev F1 (SQuAD), \PACAGC{} at $\mathrm{MI}\!=\!0.33$, $M\!=\!128$, seed 0, at each setting's step budget. \textbf{Bold} marks the dev-best clipping threshold, which is the clipping threshold the setting runs.}
\label{tab:clip-ablation-ft}
\begin{tabular}{ll ccc c}
\toprule
Task & Model & $c\!=\!100$ & $c\!=\!250$ & $c\!=\!1000$ & selected \\
\midrule
SQuAD & OPT-1.3B & $60.80$ & $59.82$ & $\mathbf{61.64}$ & $1000$ \\
SQuAD & OPT-6.7B & $68.56$ & $\mathbf{69.23}$ & $67.39$ & $250$ \\
SQuAD & Qwen2.5-7B & $81.33$ & $\mathbf{82.22}$ & $81.52$ & $250$ \\
SQuAD & Llama-3.1-8B & $79.42$ & $\mathbf{79.52}$ & $78.59$ & $250$ \\
\midrule
SST-2 & OPT-1.3B & $\mathbf{90.20}$ & $88.00$ & $89.20$ & $100$ \\
SST-2 & OPT-6.7B & $\mathbf{90.80}$ & $90.20$ & $88.00$ & $100$ \\
SST-2 & Qwen2.5-7B & $81.40$ & $81.60$ & $\mathbf{84.20}$ & $1000$ \\
SST-2 & Llama-3.1-8B & $86.20$ & $\mathbf{89.60}$ & $86.80$ & $250$ \\
\bottomrule
\end{tabular}
\end{table}

\begin{table}[t]
\centering
\small
\caption{\textbf{Clipping-threshold ablation on the SQuAD LoRA track, dev F1 (\%).} Dev F1, \PACAGC{} at $\mathrm{MI}\!=\!0.33$, $M\!=\!128$, $T\!=\!1000$, lr $10^{-3}$, seed 0. Runs differ only in the clipping threshold. \textbf{Bold} marks the dev-best clipping threshold, which is the clipping threshold the setting runs.}
\label{tab:clip-ablation-squad-lora}
\begin{tabular}{l ccc c}
\toprule
Model & $c\!=\!10$ & $c\!=\!25$ & $c\!=\!50$ & selected \\
\midrule
OPT-1.3B & $\mathbf{60.46}$ & $57.98$ & $59.02$ & $10$ \\
OPT-6.7B & $\mathbf{68.78}$ & $66.43$ & $65.30$ & $10$ \\
Qwen2.5-7B & $80.33$ & $\mathbf{80.81}$ & $80.65$ & $25$ \\
Llama-3.1-8B & $76.78$ & $\mathbf{80.08}$ & $78.78$ & $25$ \\
\bottomrule
\end{tabular}
\end{table}

Table~\ref{tab:clip-ablation-squad-lora} reports the same ablation on the SQuAD LoRA track, where every setting runs its dev-best clipping threshold.

\section{LoRA rank ablation}
\label{app:rank-ablation}

Table~\ref{tab:rank-ablation} reports the LoRA rank ablation. The dev-best rank is $r\!=\!4$, and every setting runs the standard $r\!=\!8$.

\begin{table}[h]
\centering
\small
\setlength{\tabcolsep}{8pt}
\caption{\textbf{LoRA rank ablation on SST-2 OPT-6.7B.} \PACAGC{}, $T\!=\!2000$, lr $5\cdot 10^{-4}$, $c\!=\!25$, $M\!=\!128$, $\mathrm{MI}\!=\!0.33$, seed 0. \textbf{Bold} marks the dev-winner.}\label{tab:rank-ablation}
\begin{tabular}{c r r}
\toprule
rank $r$ & dev \% & test \% \\
\midrule
$\mathbf{4}$ & $\mathbf{90.40}$ & $\mathbf{92.20}$ \\
$8$ & $89.60$ & $92.43$ \\
$16$ & $85.80$ & $91.51$ \\
\bottomrule
\end{tabular}
\end{table}

\section{Mechanism decomposition}
\label{app:mechanism-ablation}

Table~\ref{tab:ablation} removes the channel from the \method{} per-step release and replaces the released bit with one quantity computed from the step's clipped per-sample finite-difference scalars. The independent coin is the negative control. It stays at $54.24\%$, so the released bit carries the optimization signal. Sign quantization is $T$-sensitive. At $T\!=\!1000$ the secret-subset sign lags the secret-subset mean by $6.4$ points, and at $T\!=\!2000$ that gap closes to $0.7$. The universe mean with magnitude was run only at $T\!=\!2000$, where its sign lags it by $1.15$ points. Sign quantization therefore costs steps, and doubling $T$ closes most of the gap. Restricting the aggregate to the secret subset costs $0.8$ points against the universe mean at $T\!=\!2000$. With the channel in place \PACAGC{} reaches $88.65\%$ at $\mathrm{MI}\!=\!0.001$ and $89.98\!\pm\!0.18\%$ at $\mathrm{MI}\!=\!0.33$, both at $T\!=\!2000$.

\begin{table}[h]
\centering
\small
\caption{\textbf{Mechanism decomposition on SST-2 LoRA OPT-1.3B.} The rows above the rule use seed 0 at the setting's hyperparameters (lr $5\cdot 10^{-4}$, $c\!=\!25$, $M\!=\!128$), with the channel removed and no mutual-information accounting. Each row above the rule replaces the released bit with one quantity computed from the step's clipped per-sample finite-difference scalars, and the last of them replaces it with an independent coin. Each $T$ is a separate run whose learning-rate decay spans its own step budget. The two rows below the rule are \PACAGC{} with the channel in place, at $\mathrm{MI}\!=\!0.001$ under the uniform budget allocation $\beta_t\!=\!\mathrm{MI}_{\mathrm{tot}}/T$ and at $\mathrm{MI}\!=\!0.33$ under~\eqref{eq:adaptive-beta}.}\label{tab:ablation}
\setlength{\tabcolsep}{4pt}
\begin{tabular}{l c c}
\toprule
Released quantity & $T$ & Test \% \\
\midrule
Universe mean, with magnitude & $2000$ & $90.60$ \\
Sign of the universe mean & $2000$ & $89.45$ \\
Sign of the universe mean & $1000$ & $81.77$ \\
Secret-subset mean, with magnitude & $2000$ & $89.79$ \\
Secret-subset mean, with magnitude & $1000$ & $88.65$ \\
Secret-subset sign, no channel & $2000$ & $89.11$ \\
Secret-subset sign, no channel & $1000$ & $82.22$ \\
Independent $\pm1$ coin, every step & $2000$ & $54.24$ \\
\midrule
\textbf{\PACAGC{}}, $\mathrm{MI}\!=\!0.001$, uniform $\beta_t$, seed 0 & $2000$ & $88.65$ \\
\textbf{\PACAGC{}}, $\mathrm{MI}\!=\!0.33$, adaptive $\beta_t$, $n\!=\!3$ & $2000$ & $89.98_{\pm 0.18}$ \\
\bottomrule
\end{tabular}
\end{table}

\section{\PACAGCz{} step-budget sweep}
\label{app:zpl-canonical}

\PACAGCz{} releases an unbiased coin on every disagreement step. Table~\ref{tab:zpl-ft-ladder} varies the step budget on SST-2 OPT-1.3B FT with no clipping (lr $10^{-4}$, $M\!=\!128$, seed 0) to test whether accuracy degrades as coin-flip releases accumulate over a long run. Dev accuracy moves from $86.40$ at $T\!=\!1000$ to $90.40$ at $T\!=\!5000$ and test from $88.42$ to $89.68$, with no downward drift over the sweep.

\begin{table}[h]
\centering\small
\caption{\textbf{\PACAGCz{} step-budget sweep on SST-2 OPT-1.3B FT.} No clipping, lr $10^{-4}$, $M\!=\!128$, seed 0.}
\label{tab:zpl-ft-ladder}
\begin{tabular}{l rrrrr}
\toprule
$T$ & $1000$ & $2000$ & $3000$ & $4000$ & $5000$ \\
\midrule
dev \% & 86.40 & 89.00 & 90.00 & 88.80 & 90.40 \\
test \% & 88.42 & 89.79 & 89.11 & 89.68 & 89.68 \\
\bottomrule
\end{tabular}
\end{table}

\section{K-aggregation ablation (DP-AggZO-style)}
\label{app:k-ablation}

The DP-AggZO mechanism of~\citet{bao2025dpaggzo} aggregates $K$ independent zeroth-order direction estimates per training step before the privacy release. \method{} as defined in~\S\ref{sec:method} makes one release per step ($K\!=\!1$). We extend the per-step release to $K\!\in\!\{4, 16, 64\}$ in the DP-AggZO aggregation pattern. Each step samples $K$ independent ZO directions, each direction is released through its own channel at budget $\beta_t/K$, the $K$ bits $Y_{t,1},\ldots,Y_{t,K}\in\{-1,+1\}$ are averaged, and the parameter update is $\theta_{t+1}\!=\!\theta_t - \eta_t (1/K)\!\sum_{k} Y_{t,k}\!\cdot\!z_k$. The posterior is updated after each of the $K$ sub-releases, so every sub-channel is calibrated at the posterior conditioned on the sub-releases before it. Lemma~\ref{lem:per-step-mi} applies to each of the $K$ releases given the ones before it, so by the chain rule the step is charged at most $\beta_t$ and Theorem~\ref{thm:pacmi} applies unchanged. At $K\!=\!1$ the extension reproduces Algorithm~\ref{alg:paczero}, and the two produce identical transcripts.

Table~\ref{tab:k-ablation-all} reports the $K$-ablation on SST-2 at $\mathrm{MI}\!=\!0.33$. The $K\!=\!1$ column is the three-seed entry of Table~\ref{tab:full-grid-sst2}, and the $K\!>\!1$ entries are seed-0 runs at the same hyperparameters. The upper block runs every $K$ at the setting's step budget, so a step costs $K$ times the forward passes of $K\!=\!1$. The largest gains at that step budget are about one point on OPT-1.3B, $1.5$ on OPT-6.7B LoRA and three on OPT-6.7B FT. The lower block holds the forward passes at the $K\!=\!1$ budget, so a $K\!>\!1$ run takes $T/K$ steps at the learning rate selected for $K\!=\!1$. There every $K\!>\!1$ entry falls below $K\!=\!1$, to at most $76.95\%$. The gain of the upper block is bought with extra forward passes. Appendix~\ref{app:vs-dp} gives the DP-AggZO comparison at matched forward passes.

\begin{table}[t]
\centering\small
\caption{\textbf{$K$-aggregation ablation on SST-2, test accuracy \%.} \PACAGC{} at $\mathrm{MI}\!=\!0.33$, at the hyperparameters of Table~\ref{tab:pool-recipes}, with $K\!>\!1$ runs at seed 0. The upper block runs every $K$ at the setting's $T$. The lower block runs $T/K$ steps ($250$, $63$, $16$ on FT and $500$, $125$, $32$ on LoRA), so every run spends the forward passes of $K\!=\!1$ to within $2.4\%$. A dash marks a grid value outside the sweep.}\label{tab:k-ablation-all}
\begin{tabular}{llcccc}
\toprule
Model & Track & $K\!=\!1$ ($n\!=\!3$) & $K\!=\!4$ & $K\!=\!16$ & $K\!=\!64$ \\
\midrule
\multicolumn{6}{l}{\emph{Same step budget $T$}} \\
OPT-1.3B & FT & $89.45_{\pm 0.69}$ & $89.68$ & $90.25$ & $89.45$ \\
OPT-1.3B & LoRA & $89.98_{\pm 0.18}$ & $90.94$ & $90.71$ & --- \\
OPT-6.7B & FT & $91.09_{\pm 1.03}$ & $94.38$ & $94.15$ & --- \\
OPT-6.7B & LoRA & $92.62_{\pm 0.33}$ & $94.15$ & $93.46$ & --- \\
\midrule
\multicolumn{6}{l}{\emph{Same forward passes, $T/K$ steps}} \\
OPT-1.3B & FT & $89.45_{\pm 0.69}$ & $67.43$ & $56.42$ & $54.24$ \\
OPT-1.3B & LoRA & $89.98_{\pm 0.18}$ & $61.01$ & $54.59$ & $53.78$ \\
OPT-6.7B & FT & $91.09_{\pm 1.03}$ & $76.95$ & $63.53$ & $61.24$ \\
OPT-6.7B & LoRA & $92.62_{\pm 0.33}$ & $66.74$ & $62.50$ & $61.47$ \\
\bottomrule
\end{tabular}
\end{table}

\section{Unanimity against the universe size}
\label{app:scale-family}

This section tests whether the unanimity the mechanism relies on survives as the universe grows.

\paragraph{Unanimity rises with the universe size.}
Table~\ref{tab:n-sweep} sweeps the universe size $N$ over an $8\times$ range on SST-2 OPT-1.3B LoRA
under \PACAGCz{}, three seeds per point, and includes the $N\!=\!1000$ used by every setting. Unanimity rises monotonically across all four points, from $37.8\%$ at
$N\!=\!500$ to $55.0\%$ at $N\!=\!4000$. Mean test accuracy does not follow it and spans $1.1$ points, while unanimity moves $17.2$ points.

Each candidate subset holds $N/2$ records in expectation, so a larger universe gives larger subsets. On this sweep the $M$ subset means agree in sign more often as $N$ grows, so \PACAGCz{} releases the common sign more often. Appendix~\ref{app:cost-profile} reports what a larger universe costs per step.

{\renewcommand{\arraystretch}{1.05}
\begin{table}[h]
\centering
\footnotesize
\begin{minipage}[t]{\textwidth}
\centering
\caption{\textbf{Unanimity against universe size $N$}, SST-2 OPT-1.3B LoRA, \PACAGCz{}. This sweep is a separate study from the main settings and runs one set of hyperparameters, lr $10^{-3}$, $c\!=\!25$,
$T\!=\!1000$, $M\!=\!128$, so only $N$ varies across its rows. Each point averages three seeds, which differ in the training and dev sample, in the public perturbation seed and in the private draw.}
\label{tab:n-sweep}
\begin{tabular}{r c c}
\toprule
$N$ & unanimity & mean test \% \\
\midrule
$500$  & $0.378$ & $88.3$ \\
$1000$ & $0.429$ & $88.0$ \\
$2000$ & $0.477$ & $88.6$ \\
$4000$ & $0.550$ & $89.1$ \\
\bottomrule
\end{tabular}
\end{minipage}

\end{table}}

\normalcolor

\section{Matched per-step timing and the \texorpdfstring{$N$}{N} and \texorpdfstring{$M$}{M} sweeps}
\label{app:cost-profile}

This section reports the cost of the mechanism directly. Appendix~\ref{app:compute} gives the GPU-hours by kind of experiment. The matched comparison of Table~\ref{tab:cost-profile} runs
$60$ steps per method, and the $N$ and $M$ sweeps of Table~\ref{tab:cost-sweeps} run $10$ steps per
configuration. Every run is on a single exclusive H100 with no other job on the node, so the
methods are compared under identical conditions.

{\renewcommand{\arraystretch}{1.05}
\begin{table}[h]
\centering
\footnotesize
\setlength{\tabcolsep}{6pt}
\caption{\textbf{Matched per-step cost on SST-2 OPT-1.3B}, $N\!=\!1000$, $M\!=\!128$, $60$ steps, one
exclusive H100 per run. ``Job'' is total job wall clock divided by $60$ and therefore includes model
load and final evaluation, and ``loop'' is the training-loop rate excluding them. Peak memory is the
maximum resident GPU allocation sampled at $1$~Hz over the whole job. Appendix~\ref{app:vs-dp} gives the per-step and per-run forward counts of the DP baselines.}
\label{tab:cost-profile}
\begin{tabular}{l c c c c c c}
\toprule
& \multicolumn{2}{c}{FT (s/step)} & \multicolumn{2}{c}{LoRA (s/step)} & \multicolumn{2}{c}{peak mem (MiB)} \\
\cmidrule(lr){2-3}\cmidrule(lr){4-5}\cmidrule(lr){6-7}
Method & job & loop & job & loop & FT & LoRA \\
\midrule
\PACAGC{}                        & $4.65$ & $3.63$ & $4.90$ & $3.75$ & $5679$ & $5625$ \\
\PACAGCz{}                       & $4.37$ & $3.46$ & $4.98$ & $3.86$ & $5669$ & $5625$ \\
Non-private ZO reference    & $4.45$ & $3.51$ & $4.88$ & $3.80$ & $5661$ & $5625$ \\
\bottomrule
\end{tabular}
\end{table}}

\paragraph{Cost.}
The three methods span $6.4\%$ on FT and $2.0\%$ on LoRA on the job metric, and $4.9\%$ and
$2.9\%$ on the loop metric, and the ordering is not stable between the two tracks. Peak memory is
within $54$~MiB, $1.0\%$, across all six runs. The per-step cost of \PACAGC{} and \PACAGCz{} is that of the non-private ZO reference, which runs the MeZO update~\citep{malladi2023mezo} over the full universe in the same trainer.

\paragraph{Cost grows with the universe size \texorpdfstring{$N$}{N}.}
Table~\ref{tab:cost-sweeps} sweeps $N$ over an $8\times$ range at fixed $M\!=\!128$. Per-step cost
grows $5.26\times$ across that range. Fitting $t(N)=\kappa_0+\kappa_1 N$ by least squares gives a marginal
$\kappa_1\!=\!4.02$~ms per record and a fixed $\kappa_0\!=\!1.05$~s, and reproduces all four points to within
$10.3\%$.

\paragraph{No monotone cost trend in the subset count \texorpdfstring{$M$}{M}.}
The same table sweeps $M$ over a $32\times$ range at fixed $N\!=\!1000$. Per-step cost varies by
$20.4\%$ with no monotone trend in $M$.

{\renewcommand{\arraystretch}{1.05}
\begin{table}[h]
\centering
\footnotesize
\setlength{\tabcolsep}{8pt}
\caption{\textbf{Scaling in the universe size $N$ and the subset count $M$}, SST-2 OPT-1.3B LoRA,
$10$ steps per configuration on one exclusive H100. The left block sweeps $N$ at $M\!=\!128$ and the right block sweeps $M$
at $N\!=\!1000$. ``Fit'' is $t(N)=1.05+0.00402N$ seconds, fitted to the four $N$ points.}
\label{tab:cost-sweeps}
\begin{tabular}{r r r @{\hskip 2.2em} r r}
\toprule
\multicolumn{3}{c}{$N$ sweep ($M\!=\!128$)} & \multicolumn{2}{c}{$M$ sweep ($N\!=\!1000$)} \\
\cmidrule(lr){1-3}\cmidrule(lr){4-5}
$N$ & s/step & fit & $M$ & s/step \\
\midrule
$500$  & $3.24$  & $3.06$  & $32$   & $3.92$ \\
$1000$ & $4.60$  & $5.07$  & $128$  & $4.60$ \\
$2000$ & $9.49$  & $9.09$  & $512$  & $3.82$ \\
$4000$ & $17.03$ & $17.13$ & $1024$ & $4.43$ \\
\bottomrule
\end{tabular}
\end{table}}

\paragraph{The full-universe batch bounds the corpus size.}
Because every step evaluates the loss on all $N$ records, the mechanism as instantiated here does not
scale to corpora far larger than $N\!=\!1000$ without change. At $N\!=\!4000$ a $1000$-step run
already costs about $4.7$ GPU-hours on an H100, at the $17.03$~s per step of
Table~\ref{tab:cost-sweeps}. The natural next step is minibatched subset
estimation, in which each step draws a minibatch and forms the $M$ candidate subsets inside it. That
changes the accounting of the guarantee, since the secret is then observed through a sampled view of
the universe, and we leave the analysis to future work.

\normalcolor

\section{Compute}
\label{app:compute}

Every reported run uses one NVIDIA H100 GPU. Reproducing every number in this paper takes about $3{,}600$ GPU-hours: about $1{,}930$ for the \method{} runs of the sixteen settings, $920$ for the learning-rate sweeps, $480$ for the clipping-threshold, $K$, rank and mechanism ablations, $100$ for the MI-budget sweeps and step-budget runs, $90$ for the in-house DP baselines, $50$ for the non-private references and $20$ for the universe-size sweep. Scoring the checkpoints for the membership-inference attacks of Appendix~\ref{app:mia} takes about $10$ more. Including failed, superseded and exploratory runs, the project used about $8{,}000$ GPU-hours.

\clearpage
\section{Per-seed values for multi-seed entries}
\label{app:per-seed}

Table~\ref{tab:per-seed-all} gives the three per-seed test values behind each of the $160$ \method{} means of Table~\ref{tab:headline} and Appendix~\ref{app:full-grid}, $80$ per task, accuracy on SST-2 and F1 on SQuAD. The reported spread is the sample standard deviation ($n\!-\!1$ denominator).

\begin{longtable}{ll l l ll}
\caption{\textbf{Per-seed test values for every multi-seed \method{} entry of Table~\ref{tab:headline} and Appendix~\ref{app:full-grid}.} Accuracy on SST-2, F1 on SQuAD. Seeds 0, 1 and 2 in order.}\label{tab:per-seed-all}\\
\toprule
Task & Track & Method & MI & Model & per-seed \\
\midrule
\endfirsthead
\toprule
Task & Track & Method & MI & Model & per-seed \\
\midrule
\endhead
SQuAD & FT & \PACABCz{} & $0$ & OPT-1.3B & $69.67,\,69.74,\,69.27$ \\
SQuAD & FT & \PACABCz{} & $0$ & OPT-6.7B & $78.23,\,75.55,\,77.51$ \\
SQuAD & FT & \PACABCz{} & $0$ & Qwen2.5-7B & $88.20,\,88.33,\,88.64$ \\
SQuAD & FT & \PACABCz{} & $0$ & Llama-3.1-8B & $87.50,\,89.05,\,86.86$ \\
SQuAD & FT & \PACABC{} & $0.01$ & OPT-1.3B & $69.71,\,69.74,\,68.95$ \\
SQuAD & FT & \PACABC{} & $0.01$ & OPT-6.7B & $78.23,\,75.69,\,77.29$ \\
SQuAD & FT & \PACABC{} & $0.01$ & Qwen2.5-7B & $88.20,\,88.33,\,88.47$ \\
SQuAD & FT & \PACABC{} & $0.01$ & Llama-3.1-8B & $87.71,\,89.11,\,86.87$ \\
SQuAD & FT & \PACABC{} & $0.11$ & OPT-1.3B & $69.81,\,69.24,\,69.03$ \\
SQuAD & FT & \PACABC{} & $0.11$ & OPT-6.7B & $78.27,\,75.43,\,77.50$ \\
SQuAD & FT & \PACABC{} & $0.11$ & Qwen2.5-7B & $88.31,\,88.78,\,88.45$ \\
SQuAD & FT & \PACABC{} & $0.11$ & Llama-3.1-8B & $87.53,\,89.17,\,86.82$ \\
SQuAD & FT & \PACABC{} & $0.33$ & OPT-1.3B & $70.77,\,68.99,\,68.93$ \\
SQuAD & FT & \PACABC{} & $0.33$ & OPT-6.7B & $77.01,\,75.57,\,77.50$ \\
SQuAD & FT & \PACABC{} & $0.33$ & Qwen2.5-7B & $88.26,\,88.72,\,88.45$ \\
SQuAD & FT & \PACABC{} & $0.33$ & Llama-3.1-8B & $87.64,\,89.17,\,86.53$ \\
SQuAD & FT & \PACABC{} & $0.68$ & OPT-1.3B & $71.20,\,69.08,\,69.03$ \\
SQuAD & FT & \PACABC{} & $0.68$ & OPT-6.7B & $75.87,\,75.04,\,76.60$ \\
SQuAD & FT & \PACABC{} & $0.68$ & Qwen2.5-7B & $88.41,\,88.45,\,88.33$ \\
SQuAD & FT & \PACABC{} & $0.68$ & Llama-3.1-8B & $87.57,\,89.01,\,86.64$ \\
SQuAD & FT & \PACAGC{} & $0.01$ & OPT-1.3B & $68.91,\,68.34,\,65.84$ \\
SQuAD & FT & \PACAGC{} & $0.01$ & OPT-6.7B & $76.70,\,75.21,\,72.05$ \\
SQuAD & FT & \PACAGC{} & $0.01$ & Qwen2.5-7B & $87.82,\,87.31,\,87.93$ \\
SQuAD & FT & \PACAGC{} & $0.01$ & Llama-3.1-8B & $87.48,\,87.70,\,87.87$ \\
SQuAD & FT & \PACAGC{} & $0.11$ & OPT-1.3B & $69.05,\,68.80,\,65.85$ \\
SQuAD & FT & \PACAGC{} & $0.11$ & OPT-6.7B & $77.43,\,75.35,\,73.67$ \\
SQuAD & FT & \PACAGC{} & $0.11$ & Qwen2.5-7B & $88.23,\,87.24,\,87.65$ \\
SQuAD & FT & \PACAGC{} & $0.11$ & Llama-3.1-8B & $86.94,\,88.00,\,86.65$ \\
SQuAD & FT & \PACAGC{} & $0.33$ & OPT-1.3B & $69.05,\,69.00,\,65.99$ \\
SQuAD & FT & \PACAGC{} & $0.33$ & OPT-6.7B & $76.86,\,74.84,\,75.25$ \\
SQuAD & FT & \PACAGC{} & $0.33$ & Qwen2.5-7B & $88.10,\,87.49,\,88.01$ \\
SQuAD & FT & \PACAGC{} & $0.33$ & Llama-3.1-8B & $87.46,\,87.97,\,87.79$ \\
SQuAD & FT & \PACAGC{} & $0.68$ & OPT-1.3B & $68.66,\,68.60,\,66.67$ \\
SQuAD & FT & \PACAGC{} & $0.68$ & OPT-6.7B & $76.02,\,76.69,\,73.62$ \\
SQuAD & FT & \PACAGC{} & $0.68$ & Qwen2.5-7B & $88.27,\,87.47,\,88.26$ \\
SQuAD & FT & \PACAGC{} & $0.68$ & Llama-3.1-8B & $87.78,\,89.14,\,87.30$ \\
SQuAD & FT & \PACAGCz{} & $0$ & OPT-1.3B & $67.58,\,68.16,\,67.39$ \\
SQuAD & FT & \PACAGCz{} & $0$ & OPT-6.7B & $75.22,\,73.36,\,75.91$ \\
SQuAD & FT & \PACAGCz{} & $0$ & Qwen2.5-7B & $88.43,\,87.15,\,87.52$ \\
SQuAD & FT & \PACAGCz{} & $0$ & Llama-3.1-8B & $87.10,\,89.10,\,85.97$ \\
SQuAD & LoRA & \PACABCz{} & $0$ & OPT-1.3B & $70.14,\,67.67,\,69.99$ \\
SQuAD & LoRA & \PACABCz{} & $0$ & OPT-6.7B & $76.14,\,75.57,\,74.60$ \\
SQuAD & LoRA & \PACABCz{} & $0$ & Qwen2.5-7B & $87.99,\,88.43,\,88.59$ \\
SQuAD & LoRA & \PACABCz{} & $0$ & Llama-3.1-8B & $86.47,\,90.69,\,87.44$ \\
SQuAD & LoRA & \PACABC{} & $0.01$ & OPT-1.3B & $69.10,\,68.40,\,68.96$ \\
SQuAD & LoRA & \PACABC{} & $0.01$ & OPT-6.7B & $75.53,\,75.87,\,73.17$ \\
SQuAD & LoRA & \PACABC{} & $0.01$ & Qwen2.5-7B & $88.27,\,88.45,\,87.85$ \\
SQuAD & LoRA & \PACABC{} & $0.01$ & Llama-3.1-8B & $88.94,\,90.00,\,87.44$ \\
SQuAD & LoRA & \PACABC{} & $0.11$ & OPT-1.3B & $68.02,\,69.15,\,70.57$ \\
SQuAD & LoRA & \PACABC{} & $0.11$ & OPT-6.7B & $74.70,\,75.81,\,73.17$ \\
SQuAD & LoRA & \PACABC{} & $0.11$ & Qwen2.5-7B & $88.27,\,88.63,\,87.93$ \\
SQuAD & LoRA & \PACABC{} & $0.11$ & Llama-3.1-8B & $88.94,\,90.70,\,87.56$ \\
SQuAD & LoRA & \PACABC{} & $0.33$ & OPT-1.3B & $68.79,\,68.83,\,69.22$ \\
SQuAD & LoRA & \PACABC{} & $0.33$ & OPT-6.7B & $76.61,\,75.81,\,73.91$ \\
SQuAD & LoRA & \PACABC{} & $0.33$ & Qwen2.5-7B & $88.27,\,88.84,\,87.83$ \\
SQuAD & LoRA & \PACABC{} & $0.33$ & Llama-3.1-8B & $89.16,\,90.43,\,87.09$ \\
SQuAD & LoRA & \PACABC{} & $0.68$ & OPT-1.3B & $69.19,\,68.42,\,69.24$ \\
SQuAD & LoRA & \PACABC{} & $0.68$ & OPT-6.7B & $76.30,\,78.56,\,73.45$ \\
SQuAD & LoRA & \PACABC{} & $0.68$ & Qwen2.5-7B & $87.55,\,88.84,\,87.83$ \\
SQuAD & LoRA & \PACABC{} & $0.68$ & Llama-3.1-8B & $89.16,\,90.40,\,87.24$ \\
SQuAD & LoRA & \PACAGC{} & $0.01$ & OPT-1.3B & $68.90,\,68.33,\,66.60$ \\
SQuAD & LoRA & \PACAGC{} & $0.01$ & OPT-6.7B & $74.28,\,73.36,\,73.33$ \\
SQuAD & LoRA & \PACAGC{} & $0.01$ & Qwen2.5-7B & $87.76,\,88.37,\,86.63$ \\
SQuAD & LoRA & \PACAGC{} & $0.01$ & Llama-3.1-8B & $87.71,\,87.62,\,85.13$ \\
SQuAD & LoRA & \PACAGC{} & $0.11$ & OPT-1.3B & $68.47,\,68.44,\,66.69$ \\
SQuAD & LoRA & \PACAGC{} & $0.11$ & OPT-6.7B & $73.98,\,71.57,\,69.54$ \\
SQuAD & LoRA & \PACAGC{} & $0.11$ & Qwen2.5-7B & $88.47,\,88.39,\,87.31$ \\
SQuAD & LoRA & \PACAGC{} & $0.11$ & Llama-3.1-8B & $87.72,\,88.89,\,85.62$ \\
SQuAD & LoRA & \PACAGC{} & $0.33$ & OPT-1.3B & $68.41,\,67.97,\,69.45$ \\
SQuAD & LoRA & \PACAGC{} & $0.33$ & OPT-6.7B & $75.67,\,73.16,\,72.97$ \\
SQuAD & LoRA & \PACAGC{} & $0.33$ & Qwen2.5-7B & $88.56,\,88.38,\,86.98$ \\
SQuAD & LoRA & \PACAGC{} & $0.33$ & Llama-3.1-8B & $88.12,\,88.59,\,87.87$ \\
SQuAD & LoRA & \PACAGC{} & $0.68$ & OPT-1.3B & $70.31,\,65.80,\,67.20$ \\
SQuAD & LoRA & \PACAGC{} & $0.68$ & OPT-6.7B & $73.02,\,74.91,\,71.96$ \\
SQuAD & LoRA & \PACAGC{} & $0.68$ & Qwen2.5-7B & $88.77,\,88.84,\,87.95$ \\
SQuAD & LoRA & \PACAGC{} & $0.68$ & Llama-3.1-8B & $88.33,\,90.14,\,87.35$ \\
SQuAD & LoRA & \PACAGCz{} & $0$ & OPT-1.3B & $64.55,\,64.82,\,66.86$ \\
SQuAD & LoRA & \PACAGCz{} & $0$ & OPT-6.7B & $74.18,\,71.41,\,70.77$ \\
SQuAD & LoRA & \PACAGCz{} & $0$ & Qwen2.5-7B & $88.45,\,88.41,\,87.74$ \\
SQuAD & LoRA & \PACAGCz{} & $0$ & Llama-3.1-8B & $86.65,\,88.58,\,88.92$ \\
SST-2 & FT & \PACABCz{} & $0$ & OPT-1.3B & $89.79,\,90.02,\,90.48$ \\
SST-2 & FT & \PACABCz{} & $0$ & OPT-6.7B & $91.28,\,91.86,\,91.74$ \\
SST-2 & FT & \PACABCz{} & $0$ & Qwen2.5-7B & $91.86,\,92.20,\,92.66$ \\
SST-2 & FT & \PACABCz{} & $0$ & Llama-3.1-8B & $92.78,\,92.55,\,93.00$ \\
SST-2 & FT & \PACABC{} & $0.01$ & OPT-1.3B & $89.11,\,89.68,\,89.45$ \\
SST-2 & FT & \PACABC{} & $0.01$ & OPT-6.7B & $91.97,\,91.86,\,91.74$ \\
SST-2 & FT & \PACABC{} & $0.01$ & Qwen2.5-7B & $91.74,\,92.55,\,92.55$ \\
SST-2 & FT & \PACABC{} & $0.01$ & Llama-3.1-8B & $93.35,\,92.55,\,93.12$ \\
SST-2 & FT & \PACABC{} & $0.11$ & OPT-1.3B & $89.56,\,89.68,\,89.45$ \\
SST-2 & FT & \PACABC{} & $0.11$ & OPT-6.7B & $91.28,\,91.86,\,91.63$ \\
SST-2 & FT & \PACABC{} & $0.11$ & Qwen2.5-7B & $92.32,\,92.43,\,92.78$ \\
SST-2 & FT & \PACABC{} & $0.11$ & Llama-3.1-8B & $93.23,\,92.09,\,92.78$ \\
SST-2 & FT & \PACABC{} & $0.33$ & OPT-1.3B & $88.99,\,89.68,\,89.68$ \\
SST-2 & FT & \PACABC{} & $0.33$ & OPT-6.7B & $91.51,\,91.17,\,91.63$ \\
SST-2 & FT & \PACABC{} & $0.33$ & Qwen2.5-7B & $91.86,\,92.78,\,92.78$ \\
SST-2 & FT & \PACABC{} & $0.33$ & Llama-3.1-8B & $93.46,\,92.78,\,93.12$ \\
SST-2 & FT & \PACABC{} & $0.68$ & OPT-1.3B & $88.99,\,89.56,\,89.68$ \\
SST-2 & FT & \PACABC{} & $0.68$ & OPT-6.7B & $92.55,\,92.32,\,91.63$ \\
SST-2 & FT & \PACABC{} & $0.68$ & Qwen2.5-7B & $91.28,\,93.00,\,92.66$ \\
SST-2 & FT & \PACABC{} & $0.68$ & Llama-3.1-8B & $92.55,\,93.46,\,93.35$ \\
SST-2 & FT & \PACAGC{} & $0.01$ & OPT-1.3B & $88.42,\,89.56,\,89.33$ \\
SST-2 & FT & \PACAGC{} & $0.01$ & OPT-6.7B & $92.09,\,91.51,\,92.32$ \\
SST-2 & FT & \PACAGC{} & $0.01$ & Qwen2.5-7B & $89.45,\,88.30,\,88.42$ \\
SST-2 & FT & \PACAGC{} & $0.01$ & Llama-3.1-8B & $92.43,\,91.86,\,93.35$ \\
SST-2 & FT & \PACAGC{} & $0.11$ & OPT-1.3B & $88.76,\,89.68,\,90.14$ \\
SST-2 & FT & \PACAGC{} & $0.11$ & OPT-6.7B & $91.86,\,91.06,\,91.06$ \\
SST-2 & FT & \PACAGC{} & $0.11$ & Qwen2.5-7B & $89.56,\,88.88,\,87.84$ \\
SST-2 & FT & \PACAGC{} & $0.11$ & Llama-3.1-8B & $90.94,\,91.63,\,92.09$ \\
SST-2 & FT & \PACAGC{} & $0.33$ & OPT-1.3B & $88.76,\,89.45,\,90.14$ \\
SST-2 & FT & \PACAGC{} & $0.33$ & OPT-6.7B & $92.09,\,90.02,\,91.17$ \\
SST-2 & FT & \PACAGC{} & $0.33$ & Qwen2.5-7B & $89.68,\,89.45,\,87.96$ \\
SST-2 & FT & \PACAGC{} & $0.33$ & Llama-3.1-8B & $91.06,\,90.60,\,92.66$ \\
SST-2 & FT & \PACAGC{} & $0.68$ & OPT-1.3B & $90.48,\,88.88,\,88.99$ \\
SST-2 & FT & \PACAGC{} & $0.68$ & OPT-6.7B & $91.86,\,90.02,\,91.97$ \\
SST-2 & FT & \PACAGC{} & $0.68$ & Qwen2.5-7B & $90.83,\,89.11,\,89.22$ \\
SST-2 & FT & \PACAGC{} & $0.68$ & Llama-3.1-8B & $92.32,\,90.83,\,91.74$ \\
SST-2 & FT & \PACAGCz{} & $0$ & OPT-1.3B & $90.25,\,89.22,\,90.02$ \\
SST-2 & FT & \PACAGCz{} & $0$ & OPT-6.7B & $89.91,\,90.48,\,89.22$ \\
SST-2 & FT & \PACAGCz{} & $0$ & Qwen2.5-7B & $89.56,\,89.56,\,88.99$ \\
SST-2 & FT & \PACAGCz{} & $0$ & Llama-3.1-8B & $91.51,\,92.43,\,92.66$ \\
SST-2 & LoRA & \PACABCz{} & $0$ & OPT-1.3B & $89.91,\,90.02,\,89.79$ \\
SST-2 & LoRA & \PACABCz{} & $0$ & OPT-6.7B & $93.12,\,93.81,\,93.00$ \\
SST-2 & LoRA & \PACABCz{} & $0$ & Qwen2.5-7B & $94.27,\,93.69,\,94.72$ \\
SST-2 & LoRA & \PACABCz{} & $0$ & Llama-3.1-8B & $91.40,\,86.12,\,90.60$ \\
SST-2 & LoRA & \PACABC{} & $0.01$ & OPT-1.3B & $89.45,\,90.25,\,90.48$ \\
SST-2 & LoRA & \PACABC{} & $0.01$ & OPT-6.7B & $92.66,\,93.69,\,93.00$ \\
SST-2 & LoRA & \PACABC{} & $0.01$ & Qwen2.5-7B & $94.15,\,94.04,\,94.38$ \\
SST-2 & LoRA & \PACABC{} & $0.01$ & Llama-3.1-8B & $89.22,\,90.25,\,91.28$ \\
SST-2 & LoRA & \PACABC{} & $0.11$ & OPT-1.3B & $89.91,\,89.91,\,90.48$ \\
SST-2 & LoRA & \PACABC{} & $0.11$ & OPT-6.7B & $92.32,\,93.46,\,93.00$ \\
SST-2 & LoRA & \PACABC{} & $0.11$ & Qwen2.5-7B & $93.46,\,94.38,\,94.27$ \\
SST-2 & LoRA & \PACABC{} & $0.11$ & Llama-3.1-8B & $89.22,\,90.25,\,87.61$ \\
SST-2 & LoRA & \PACABC{} & $0.33$ & OPT-1.3B & $89.68,\,90.60,\,89.79$ \\
SST-2 & LoRA & \PACABC{} & $0.33$ & OPT-6.7B & $92.32,\,94.04,\,92.66$ \\
SST-2 & LoRA & \PACABC{} & $0.33$ & Qwen2.5-7B & $94.15,\,93.12,\,94.50$ \\
SST-2 & LoRA & \PACABC{} & $0.33$ & Llama-3.1-8B & $89.45,\,90.94,\,90.37$ \\
SST-2 & LoRA & \PACABC{} & $0.68$ & OPT-1.3B & $89.22,\,90.60,\,90.71$ \\
SST-2 & LoRA & \PACABC{} & $0.68$ & OPT-6.7B & $92.78,\,93.81,\,93.69$ \\
SST-2 & LoRA & \PACABC{} & $0.68$ & Qwen2.5-7B & $93.58,\,93.23,\,94.95$ \\
SST-2 & LoRA & \PACABC{} & $0.68$ & Llama-3.1-8B & $91.40,\,90.02,\,90.94$ \\
SST-2 & LoRA & \PACAGC{} & $0.01$ & OPT-1.3B & $88.65,\,90.37,\,90.48$ \\
SST-2 & LoRA & \PACAGC{} & $0.01$ & OPT-6.7B & $92.20,\,92.66,\,91.97$ \\
SST-2 & LoRA & \PACAGC{} & $0.01$ & Qwen2.5-7B & $92.66,\,92.89,\,92.55$ \\
SST-2 & LoRA & \PACAGC{} & $0.01$ & Llama-3.1-8B & $91.17,\,85.67,\,91.06$ \\
SST-2 & LoRA & \PACAGC{} & $0.11$ & OPT-1.3B & $90.37,\,89.33,\,90.60$ \\
SST-2 & LoRA & \PACAGC{} & $0.11$ & OPT-6.7B & $93.12,\,92.89,\,92.09$ \\
SST-2 & LoRA & \PACAGC{} & $0.11$ & Qwen2.5-7B & $92.89,\,93.58,\,93.58$ \\
SST-2 & LoRA & \PACAGC{} & $0.11$ & Llama-3.1-8B & $90.60,\,87.04,\,89.11$ \\
SST-2 & LoRA & \PACAGC{} & $0.33$ & OPT-1.3B & $90.02,\,89.79,\,90.14$ \\
SST-2 & LoRA & \PACAGC{} & $0.33$ & OPT-6.7B & $92.43,\,92.43,\,93.00$ \\
SST-2 & LoRA & \PACAGC{} & $0.33$ & Qwen2.5-7B & $92.32,\,92.20,\,93.23$ \\
SST-2 & LoRA & \PACAGC{} & $0.33$ & Llama-3.1-8B & $91.63,\,88.42,\,86.35$ \\
SST-2 & LoRA & \PACAGC{} & $0.68$ & OPT-1.3B & $89.79,\,90.02,\,90.02$ \\
SST-2 & LoRA & \PACAGC{} & $0.68$ & OPT-6.7B & $91.86,\,91.97,\,93.35$ \\
SST-2 & LoRA & \PACAGC{} & $0.68$ & Qwen2.5-7B & $92.78,\,92.66,\,93.00$ \\
SST-2 & LoRA & \PACAGC{} & $0.68$ & Llama-3.1-8B & $92.78,\,88.07,\,90.71$ \\
SST-2 & LoRA & \PACAGCz{} & $0$ & OPT-1.3B & $91.06,\,88.88,\,90.14$ \\
SST-2 & LoRA & \PACAGCz{} & $0$ & OPT-6.7B & $93.46,\,93.69,\,93.69$ \\
SST-2 & LoRA & \PACAGCz{} & $0$ & Qwen2.5-7B & $92.20,\,92.32,\,92.89$ \\
SST-2 & LoRA & \PACAGCz{} & $0$ & Llama-3.1-8B & $87.73,\,87.04,\,88.30$ \\
\bottomrule
\end{longtable}

\newpage
\section*{NeurIPS Paper Checklist}

\begin{enumerate}

\item {\bf Claims}
    \item[] Question: Do the main claims made in the abstract and introduction accurately reflect the paper's contributions and scope?
    \item[] Answer: \answerYes{}
    \item[] Justification: The abstract and the Introduction state three contributions that map directly to formal results and experiments: the \method{} family with four variants \PACAGC{}, \PACAGCz{}, \PACABC{} and \PACABCz{} (defined in \S\ref{sec:method} and evaluated in Tables~\ref{tab:headline} to~\ref{tab:dp-cliff}), exact MI calibration of the additive Gaussian channel in place of the variance-based upper bound of canonical PAC Privacy (\S\ref{sec:binary-mi}, with the MI-budget sweep in Table~\ref{tab:plateau}) and the distortion-optimal binary asymmetric channel (\S\ref{sec:abc}), compared with the Gaussian channel in Table~\ref{tab:headline}, and the $\MI(S^*;\tilde Y_{1:T})\!=\!0$ guarantee of Theorem~\ref{thm:zero} (\S\ref{sec:zpl}, proof in Appendix~\ref{app:zero_proof}) with the $\mathrm{MI}\!=\!0$ entries of Table~\ref{tab:headline}. Privacy semantics are stated in their native units throughout, and the paper states that PAC Privacy and DP are not equivalent (\S\ref{sec:zpl}, \S\ref{sec:discussion}).

\item {\bf Limitations}
    \item[] Question: Does the paper discuss the limitations of the work performed by the authors?
    \item[] Answer: \answerYes{}
    \item[] Justification: \S\ref{sec:discussion} (Discussion and Limitations) discusses (i) the semantic scope of the PAC Privacy guarantee, which is instance-based and not equivalent to worst-case DP, and (ii) distributional leakage, since the guarantee covers which candidate subset was sampled and does not bound what a release reveals about $U$. \S\ref{sec:conc} names as open directions (iii) a convergence analysis of \method{}, (iv) larger corpora, since every step reads all $N$ records and minibatched subset estimation needs a new privacy analysis (Appendix~\ref{app:cost-profile}), and (v) settings beyond supervised fine-tuning. \S\ref{sec:zpl} states the scope of the zero-MI guarantee.

\item {\bf Theory assumptions and proofs}
    \item[] Question: For each theoretical result, does the paper provide the full set of assumptions and a complete (and correct) proof?
    \item[] Answer: \answerYes{}
    \item[] Justification: All theoretical objects are numbered, cross-referenced, and proved. Definition~\ref{def:pac} (PAC Privacy) and the chain rule \eqref{eq:chain-rule} are stated in \S\ref{sec:background} with the data-generating distribution $\mathcal{D}\!=\!\mathrm{Unif}\{S_1,\ldots,S_M\}$ explicitly constructed. Theorem~\ref{thm:pacmi} (\method{} privacy, for any calibrated channel and therefore for \PACAGC{} and \PACABC{}) and Corollary~\ref{cor:pacmi-mia} (the MIA bound via the binary-KL inequality) are stated in \S\ref{sec:adaptive}, and Theorem~\ref{thm:zero} (zero MI of \PACAGCz{} and \PACABCz{}) in \S\ref{sec:zpl}. Appendix~\ref{app:proofs} gives the full proofs together with the lemmas they rely on, and Appendix~\ref{app:method-details} gives the entropy-cap argument used in adaptive composition.

    \item {\bf Experimental result reproducibility}
    \item[] Question: Does the paper fully disclose all the information needed to reproduce the main experimental results of the paper to the extent that it affects the main claims and/or conclusions of the paper (regardless of whether the code and data are provided or not)?
    \item[] Answer: \answerYes{}
    \item[] Justification: Algorithm~\ref{alg:paczero} fully specifies the per-step procedure of \method{}, including the unanimity branch, the disagreement branch, and the posterior update, and \S\ref{sec:binary-mi} to \S\ref{sec:zpl} specify the channels and the zero-budget variants. The numerical calibration is in Appendix~\ref{app:method-details}. The data splits and protocol follow~\citet{bao2025dpaggzo} and are stated in \S\ref{sec:setup}, with the test sets and the remaining protocol details in Appendix~\ref{app:protocol-details}. Per-setting hyperparameters (learning rate, clipping threshold $c$, $T$, $M$) are tabulated in Appendix~\ref{app:lr-sweep} (Table~\ref{tab:pool-recipes}), with the learning-rate sweeps in Table~\ref{tab:lr-sweeps} and the clipping-threshold ablations in Appendix~\ref{sec:clip-ablation}. The DP-AggZO parity check (Appendix~\ref{app:dpaggzo-parity}) documents the in-house DP baseline implementation against published entries. The code release at \url{https://github.com/bilgehanertan/paczero-bac} gives a configuration and instructions for every reported experiment.

\item {\bf Open access to data and code}
    \item[] Question: Does the paper provide open access to the data and code, with sufficient instructions to faithfully reproduce the main experimental results, as described in supplemental material?
    \item[] Answer: \answerYes{}
    \item[] Justification: The code and the configuration of every reported experiment are released at \url{https://github.com/bilgehanertan/paczero-bac}, with pinned environments and instructions for setup and data loading. SST-2 and SQuAD load from Hugging Face Datasets. No model weights are released. All runs start from the public OPT-1.3B, OPT-6.7B, Qwen2.5-7B and Llama-3.1-8B base models.

\item {\bf Experimental setting/details}
    \item[] Question: Does the paper specify all the training and test details (e.g., data splits, hyperparameters, how they were chosen, type of optimizer) necessary to understand the results?
    \item[] Answer: \answerYes{}
    \item[] Justification: \S\ref{sec:setup} and Appendix~\ref{app:protocol-details} give the training and evaluation protocol, and \S\ref{sec:method} gives the update rule. Table~\ref{tab:pool-recipes} (Appendix~\ref{app:lr-sweep}) gives the hyperparameters of every setting and how they were selected. The learning-rate sweeps are in Table~\ref{tab:lr-sweeps}, and the clipping-threshold, LoRA-rank, $K$-aggregation and \PACAGCz{} step-budget ablations are in Appendices~\ref{sec:clip-ablation},~\ref{app:rank-ablation},~\ref{app:k-ablation} and~\ref{app:zpl-canonical}.

\item {\bf Experiment statistical significance}
    \item[] Question: Does the paper report error bars suitably and correctly defined or other appropriate information about the statistical significance of the experiments?
    \item[] Answer: \answerYes{}
    \item[] Justification: All multi-seed entries in Table~\ref{tab:headline} report mean$\,\pm\,$std with $n\!=\!3$ stated in the table caption. The error bars are sample standard deviations across random seeds with $n\!-\!1$ denominator, stated explicitly at the start of Appendix~\ref{app:per-seed}. The seeds resample the universe $U$ and, on SQuAD, the evaluation subset (Appendix~\ref{app:protocol-details}). Per-seed values for every multi-seed \method{} entry of Table~\ref{tab:headline} and Appendix~\ref{app:full-grid} are in Table~\ref{tab:per-seed-all}, so a reader can recompute those means and sample standard deviations. Single-seed entries are marked as single-seed or seed $0$, or shown without a spread. Every \method{} entry of Table~\ref{tab:headline}, on both tasks, is multi-seed.

\item {\bf Experiments compute resources}
    \item[] Question: For each experiment, does the paper provide sufficient information on the computer resources (type of compute workers, memory, time of execution) needed to reproduce the experiments?
    \item[] Answer: \answerYes{}
    \item[] Justification: Appendix~\ref{app:compute} gives the hardware, one NVIDIA H100 GPU per reported run, and the GPU-hours by kind of experiment: about $3{,}600$ to reproduce every number in the paper and about $8{,}000$ for the whole project, including failed, superseded and exploratory runs. Appendix~\ref{app:cost-profile} reports per-step time and peak memory.
    
\item {\bf Code of ethics}
    \item[] Question: Does the research conducted in the paper conform, in every respect, with the NeurIPS Code of Ethics \url{https://neurips.cc/public/EthicsGuidelines}?
    \item[] Answer: \answerYes{}
    \item[] Justification: The research conforms with the NeurIPS Code of Ethics. The work is a privacy-mechanism contribution evaluated on standard public benchmarks (SST-2, SQuAD) using publicly released base models (OPT-1.3B, OPT-6.7B, Qwen2.5-7B and Llama-3.1-8B). No human subjects, no crowdsourcing, no scraped data, and no new dataset of personally identifiable information are involved.

\item {\bf Broader impacts}
    \item[] Question: Does the paper discuss both potential positive societal impacts and negative societal impacts of the work performed?
    \item[] Answer: \answerYes{}
    \item[] Justification: The paper's positive societal impact is methodological: it expands
    the operating regime in which LLM fine-tuning on sensitive corpora is defensible, by 
    producing usable utility at $\MI(S^*;\tilde Y_{1:T})\!=\!0$, an MIA-resistance level prior
    DP zeroth-order methods cannot reach at finite per-step noise (Table~\ref{tab:dp-cliff},
    \S\ref{sec:conc}). The principal misuse is to read the guarantee as worst-case DP, or as a
    bound on what a release reveals about the universe $U$. \S\ref{sec:zpl} and \S\ref{sec:discussion} state that it is neither. The paper does not introduce a new generative model or scraped corpus.
    
\item {\bf Safeguards}
    \item[] Question: Does the paper describe safeguards that have been put in place for responsible release of data or models that have a high risk for misuse (e.g., pre-trained language models, image generators, or scraped datasets)?
    \item[] Answer: \answerNA{}
    \item[] Justification: The paper does not release any pre-trained model, image generator, or scraped dataset. The released artifact is a training-time mechanism (algorithm and reference implementation) that operates on publicly available base models (OPT-1.3B, OPT-6.7B, Qwen2.5-7B and Llama-3.1-8B) and public benchmarks (SST-2, SQuAD). No fine-tuned checkpoints are distributed.

\item {\bf Licenses for existing assets}
    \item[] Question: Are the creators or original owners of assets (e.g., code, data, models), used in the paper, properly credited and are the license and terms of use explicitly mentioned and properly respected?
    \item[] Answer: \answerYes{}
    \item[] Justification: The models, datasets and baseline code are credited by citation. OPT-1.3B and OPT-6.7B~\citep{zhang2022opt} are released by Meta under their custom OPT model license, Qwen2.5-7B~\citep{qwen2025qwen25} under Apache-2.0 and Llama-3.1-8B~\citep{grattafiori2024llama3} under the Llama 3.1 Community License, all loaded from Hugging Face. SST-2~\citep{socher2013sst}, distributed without a stated license, and SQuAD v1.1~\citep{rajpurkar2016squad}, under CC BY-SA 4.0, are used with the splits and sample counts published by~\citet{bao2025dpaggzo}. Our \texttt{src/} task scaffolding is adapted from the DP-AggZO reference codebase~\citep{bao2025dpaggzo}, which builds on MeZO~\citep{malladi2023mezo}, and our LoRA layers from loralib~\citep{hu2022lora}, all three under the MIT license. We use Hugging Face Transformers \texttt{4.45.2} (Apache-2.0) and PyTorch (BSD-3-Clause). Privacy-bound conversions referencing prior PAC Privacy work~\citep{xiao2023pac, sridhar2025pac, zhu2025pacresponses} are cited at every use.

\item {\bf New assets}
    \item[] Question: Are new assets introduced in the paper well documented and is the documentation provided alongside the assets?
    \item[] Answer: \answerYes{}
    \item[] Justification: The new asset is the \method{} reference implementation, released under the MIT license at \url{https://github.com/bilgehanertan/paczero-bac}. Its README documents the setup, the code layout and how to run each variant and reproduce each experiment. No new dataset or model is released.

\item {\bf Crowdsourcing and research with human subjects}
    \item[] Question: For crowdsourcing experiments and research with human subjects, does the paper include the full text of instructions given to participants and screenshots, if applicable, as well as details about compensation (if any)? 
    \item[] Answer: \answerNA{}
    \item[] Justification: The paper does not involve crowdsourcing or research with human subjects. All experiments are automated fine-tuning runs on the public SST-2 and SQuAD benchmarks.

\item {\bf Institutional review board (IRB) approvals or equivalent for research with human subjects}
    \item[] Question: Does the paper describe potential risks incurred by study participants, whether such risks were disclosed to the subjects, and whether Institutional Review Board (IRB) approvals (or an equivalent approval/review based on the requirements of your country or institution) were obtained?
    \item[] Answer: \answerNA{}
    \item[] Justification: The paper does not involve human subjects or any data collected from human participants, so no IRB review applies. The training data is the public SST-2 and SQuAD benchmarks, used at the splits specified in \S\ref{sec:setup}.

\item {\bf Declaration of LLM usage}
    \item[] Question: Does the paper describe the usage of LLMs if it is an important, original, or non-standard component of the core methods in this research? Note that if the LLM is used only for writing, editing, or formatting purposes and does \emph{not} impact the core methodology, scientific rigor, or originality of the research, declaration is not required.
    \item[] Answer: \answerNA{}
    \item[] Justification: The core methodological contributions are the sign-quantized PAC-private mechanism, the exact MI calibration of the Gaussian channel (\S\ref{sec:binary-mi}), the binary asymmetric channel (\S\ref{sec:abc}) and the zero-MI variants \PACAGCz{} and \PACABCz{} (\S\ref{sec:zpl}). None of these involves an LLM as an original or non-standard component. LLMs are used for implementing standard methods, editing (e.g., grammar, spelling, word choice), facilitating or running experiments, drafting sections of the paper.

\end{enumerate}

\end{document}